\documentclass{article}
\usepackage{microtype}
\usepackage{mathtools}
\usepackage{mathrsfs}
\usepackage{algorithm}
\usepackage{algorithmic}
\usepackage{amssymb}
\usepackage{graphicx}
\usepackage[numbers,sort&compress]{natbib}
\usepackage[colorlinks=true, allcolors=blue]{hyperref}
\usepackage{subcaption}
\usepackage{bm}
\usepackage{amsmath}
\numberwithin{equation}{section}
\usepackage{booktabs}
\usepackage{tabularx}
\usepackage{multirow}
\usepackage{enumerate}
\usepackage{amsthm}
\usepackage{authblk}

\usepackage{fancyhdr}
\newtheorem{theorem}{Theorem}[section]
\newtheorem{lemma}[theorem]{Lemma}

\newtheorem{corollary}[theorem]{Corollary}

\theoremstyle{definition}
\newtheorem{definition}[theorem]{Definition}

\newtheorem{remark}[theorem]{Remark}
\allowdisplaybreaks

\def\e{\mathbb{E}}

\def\rbb{\mathbb{R}}

\newcommand{\bz}{\mathbf{z}}

\newcommand{\bw}{\mathbf{w}}

\newcommand{\xcal}{\mathcal{X}}
\newcommand{\wcal}{\mathcal{W}}

\newcommand{\bx}{\mathbf{x}}

\newcommand{\zcal}{\mathcal{Z}}

\newcommand{\ycal}{\mathcal{Y}}

\newcommand{\ebb}{\mathbb{E}}

\newcommand{\pbb}{\mathbb{P}}

\newcommand{\argmin}{\operatorname{\arg\min}}

\title{{\bf Learning Theory of the SVRG: Generalization and Convergence Analysis}}

\author[1]{Yunwen Lei}
\author[2]{Zimeng Wang}
\author[1]{Xiaoming Yuan}
\affil[1]{Department of Mathematics, The University of Hong Kong, Hong Kong, China\\
\texttt{leiyw@hku.hk, xmyuan@hku.hk}}
\affil[2]{Department of Mathematics and Mathematical Statistics, Ume\r{a} University, Sweden\\
\texttt{zimeng.wang@umu.se}}

\date{May 27, 2026}

\begin{document}

\maketitle

\begin{abstract}
  Variance reduction (VR) methods employ stochastic gradients with decreasing variance, and they have been widely applied to solve large-scale optimization problems in machine learning because of their efficiency. Existing theoretical studies of VR methods are mainly focused on the convergence analysis, leaving the generalization behavior largely unexplored. In this paper, we bridge this gap by developing the first non-vacuous generalization analysis of the representative VR method: Stochastic Variance Reduced Gradient (SVRG), through the lens of algorithmic stability. In particular, we establish sharp stability bounds of the SVRG in both convex and strongly convex settings by exploiting its algorithmic structure. The obtained bounds are data-dependent, because the training errors are incorporated along the trajectory. Our analysis clarifies the interplay between optimization and generalization, leading to optimal excess population risk bounds in both settings. Our approach differs substantially from existing analyses of stochastic algorithms in the sense that we decompose the SVRG update as an SGD-like step plus a zero-mean correction term and then introduce novel Lyapunov functions to absorb the additional gradient terms induced by the reference points. Our analytical framework can be generalized to other VR methods, and we demonstrate the generalization by the well-known Stochastic Average Gradient Accelerated (SAGA) method.
\end{abstract}


\section{Introduction}
The rapid growth of model complexity and data scale in modern machine learning has posed significant challenges for traditional optimization algorithms such as gradient descent (GD).
To address this bottleneck, stochastic methods have gained widespread adoption, which balance the computational cost and precision for scalability and efficiency. Among these, the stochastic gradient descent (SGD) method \citep{robbins1951stochastic} and its variants \citep{cotter2011better, dekel2012optimal, loizou2020momentum,nguyen2021unified,fang2019sharp,rosasco2020convergence} have emerged as the dominant optimization framework. Unlike GD, which relies on full gradient computations, SGD leverages random sampling to reduce per-iteration computational costs. This property enables its superior performance in large-scale model training across diverse applications. Unfortunately, SGD suffers from inherent limitations: its convergence rate is sublinear even for strongly convex objectives, as diminishing step sizes are required to ensure asymptotic convergence. This often results in slow practical performance, motivating the development of accelerated variants.
Among these, variance reduction (VR) methods such as the Stochastic Variance Reduced Gradient (SVRG) \citep{johnson2013accelerating,xiao2014proximal}, Stochastic Average Gradient (SAG) \citep{roux2012stochastic}, and Stochastic Average Gradient Accelerated (SAGA) \citep{defazio2014saga} have been particularly successful in achieving faster convergence than SGD. Unlike SGD, which uses stochastic gradient estimates with non-vanishing variance, VR methods employ more sophisticated update rules that progressively reduce the variance of gradient approximations, and hence often exhibit faster convergence both in theory and practice. Specifically, these VR methods admit constant step sizes and can achieve linear convergence rates for strongly convex objectives.

Existing analyses of VR methods have primarily focused on optimization errors (i.e., the performance of models on the training dataset), while in machine learning, the generalization behavior of models on unseen testing samples is of greater interest. Unfortunately, the generalization analysis of VR methods is much lagging behind. There exist only very few investigations on the generalization behavior of SVRG \citep{charles2018stability, liu2024general, meng2017generalization, sun2024understanding}. Specifically, the studies \citep{charles2018stability, meng2017generalization} derive black-box generalization bounds for convergent learning algorithms, expressed in terms of optimization error and training sample size. While these bounds can be specialized to SVRG by invoking its known convergence rates, they fail to leverage the specific algorithmic structure of SVRG and imply loose guarantees. Furthermore, these results require additional assumptions such as the strong convexity and Polyak-\L ojasiewicz (PL) conditions. More recently, some more refined analyses are taken in \citep{liu2024general, sun2024understanding}, which investigate the uniform stability \citep{bousquet2002stability} and on-average stability \citep{shalev2010learnability} of SVRG. Although they consider the algorithmic structure, their analyses require a strong Lipschitzness assumption on the loss functions and very small step sizes for which the method converges at extremely slow rates. %

Motivated by the above discussions, we aim to bridge the gap between the well-established optimization theory and the underdeveloped generalization analysis of VR methods. To this end, we develop the generalization analysis of SVRG through the lens of algorithmic stability \citep{bousquet2002stability}, which is a fundamental framework for quantifying the sensitivity of learning algorithms to perturbations in the training data. We then demonstrate the generality of our approach by extending the analysis to SAGA.
Our main contributions are summarized as follows.
\begin{itemize}
    \item[1).] We develop generalization error bounds for SVRG in both convex and strongly convex settings by investigating its on-average model stability \citep{lei2020fine}, and extend the analysis to SAGA. Notably, we obtain these results without assuming Lipschitzness of the loss functions. To the best of our knowledge, this represents the first non-vacuous generalization analysis of any VR method that can exploit the algorithmic structure (i.e., without interpreting it as a black-box method). Moreover, our stability bounds involve training errors, which show the benefit of optimization to generalization.
    \item[2).] We derive convergence guarantees for SVRG under convex loss functions with milder step size conditions compared to those in \citet{reddi2016stochastic}, and similarly establish optimization error bounds for SAGA with weaker step size requirements than those in the original work \citep{defazio2014saga}. By combining these bounds with our generalization results, we further obtain excess population risk (EPR) bounds of SVRG and SAGA for both convex and strongly convex problems.
    \item[3).] The derived EPR bounds for SVRG and SAGA are optimal. Specifically, our bounds for both methods are of order $O(1/\sqrt{n})$ for convex problems and of order $\widetilde{O}(1/(n\mu))$ for $\mu$-strongly convex problems, where $n$ denotes the sample size and $\widetilde{O}(\cdot)$ hides logarithmic factors. These results match the existing minimax lower bounds for statistical guarantees \citep{agarwal2009information}.
\end{itemize}

Our generalization analysis of SVRG differs significantly from existing ones of SGD and its variants \citep{hardt2016train,lei2020fine}, owing to the more complicated gradient structure inherent in SVRG's two-loop update mechanism. To overcome this difficulty, we split the iterative scheme into two parts: one mimics the SGD step and the other serves as a gradient correction. To handle the redundant terms arising from the gradient correction term, we introduce novel Lyapunov functions to help construct recursion inequalities. Additionally, to remove the Lipschitzness assumption, we leverage the self-bounding property \citep{srebro2010smoothness} to bound the terms involving norms of gradients. These techniques extend naturally to SAGA, and we believe they can further inspire generalization analysis for other VR methods.

The rest of the paper is organized as follows. We discuss related work in Section~\ref{sec:rw} and formulate the problem in Section~\ref{sec:setup}. Sections~\ref{sec:svrg-stab} and~\ref{sec:svrg-epr} present our stability and excess population risk analysis for SVRG in the convex setting, respectively. Section~\ref{sec:saga} presents the analysis of SAGA as an application of our analytical framework. Section~\ref{sec:simulation} is devoted to numerical experiments, and Section~\ref{sec:conclusion} concludes the paper. The omitted proofs for SAGA are collected in Appendix~\ref{sec:saga-proofs}, and the strongly convex analyses of both SVRG and SAGA are deferred to Appendix~\ref{app:sc}.

\section{Related Work}\label{sec:rw}
In this section, we review related work on algorithmic stability and variance reduction methods.

\noindent\textbf{Algorithmic stability.} Algorithmic stability is a fundamental concept in statistical learning theory to study the generalization behavior of learning algorithms~\citep{bousquet2002stability}. Intuitively, a learning algorithm is said to be stable if the output of this algorithm changes slightly if we perturb the training dataset by a single example. Various stability concepts have been introduced in the literature, including the uniform stability~\citep{bousquet2002stability}, on-average stability~\citep{shalev2010learnability}, hypothesis stability~\citep{bousquet2002stability}, argument stability~\citep{liu2017algorithmic}, and on-average model stability~\citep{lei2020fine}. The seminal work~\citep{hardt2016train} pioneered the stability analysis of SGD for smooth and Lipschitz problems. Recent progress relaxes the Lipschitzness assumptions and develops optimistic stability bounds by incorporating the training errors into the stability analysis~\citep{kuzborskij2018data,lei2020fine,schliserman2022stability}. A notable property of the stability analysis is that it can imply algorithm-dependent generalization bounds to understand the balance between optimization and generalization, which motivates the growing interest in studying the stability of various stochastic optimization algorithms, including federated learning~\citep{chen2023minimax,sun2023understanding}, delayed SGD~\citep{deng2025towards}, stochastic proximal point methods~\citep{yuan2023sharper}, zeroth-order SGD~\citep{nikolakakis2022black}, clipped-SGD~\citep{zeng2026stochastic}, Langevin dynamics~\citep{mou2018generalization}, minibatch SGD~\citep{nikolakakis2025select}, differential privacy~\citep{bassily2020stability}, multi-objective learning~\citep{chen2024three} and neural network training~\citep{deora2024optimization, richards2021stability,taheri2024generalization}.

\noindent\textbf{Variance reduction methods.} The concept of variance reduction (VR) techniques was initially developed to mitigate the non-vanishing variance issue in SGD that prevents convergence \citep{gower2020variance, wang2013variance}.
The contemporary era of VR methods began with the stochastic average gradient (SAG) method \citep{roux2012stochastic}, which was then followed by several influential alternatives, including the SAGA \citep{defazio2014saga}, the stochastic variance-reduced gradient (SVRG) \citep{johnson2013accelerating}, and the stochastic dual coordinate ascent (SDCA) \citep{shalev2013stochastic}. Specifically, the SAG method is a stochastic variant of the earlier incremental aggregated gradient method \citep{blatt2007convergent}. For $\mu$-strongly convex and $\alpha$-smooth problems, SAG was the first stochastic method to enjoy a linear convergence rate with a total gradient complexity (number of component gradient evaluations) of $O((\alpha/\mu+n)\log(1/\epsilon))$ to achieve a precision of $\epsilon>0$, using a step size of $O(1/\alpha)$. The SAGA method is an unbiased variant of SAG that employs reference points to construct unbiased stochastic gradients. Similar to SAG, it exhibits a linear convergence rate and a gradient complexity of $O((\alpha/\mu+n)\log(1/\epsilon))$, but offers a simpler theoretical analysis. Another prominent VR method, SVRG, also leverages reference points but reduces the storage cost of SAG while maintaining the same order of gradient complexity. The stability analysis of SVRG was also established in the literature, which, however, either interpreted SVRG as a black-box method~\citep{charles2018stability, meng2017generalization} or implied vacuous risk bounds~\citep{liu2024general, sun2024understanding}. For a broader overview of VR methods, we refer the reader to \citep{allen2018katyusha, csiba2015stochastic, fang2018spider, mairal2015incremental, nguyen2017sarah, shalev2013stochastic}.

\section{Problem Setup}\label{sec:setup}
Let $\pbb$ be a probability measure defined on a sample space $\zcal=\xcal\times\ycal$, where $\xcal$ is an input space and $\ycal$ is an output space.
Let $S=\{\bz_1,\ldots,\bz_n\}$ be a sequence of training examples drawn independently from $\pbb$, from which we aim to learn a model $h:\xcal\mapsto\ycal$ for prediction. We consider parametric models where each model is parameterized by a parameter $\bw$ in a parameter space $\wcal=\rbb^d$.
Let $\ell(\bw;\bz)$ be the loss incurred by applying $\bw$ to do the prediction on $\bz:=(\bx,y)$.
Popular examples of loss functions include the logistic loss $\ell(\bw;\bz)=\log(1+\exp(-y\langle\bw,\bx\rangle))$ and the smoothed hinge loss for classification, as well as the least-squares loss $\ell(\bw;\bz)=\frac{1}{2}(\bw^\top\bx-y)^2$ and Huber loss for regression.
We are interested in minimizing the population risk
\[
L(\bw):=\ebb_{\bz\sim\pbb}[\ell(\bw;\bz)],
\]
where $\ebb_{\bz}[\cdot]$ denotes the expectation w.r.t. $\bz$.
Since the probability measure $\pbb$ is unknown, we consider its empirical approximation
\begin{equation}\label{eq:LS-def}
    L_S(\bw):=\frac{1}{n}\sum_{i=1}^{n}\ell(\bw;\bz_i),
\end{equation}
which is referred to as the empirical risk (training error).

\subsection{Stability and Generalization}
We often apply a learning algorithm to approximately minimize the training error. We denote by $A(S)$ the output model by applying an algorithm $A$ to the dataset $S$. In statistical learning theory (SLT), we are especially interested in the relative behavior of the output model $A(S)$ as compared to the best model $\bw^*:=\arg\min_{\bw\in\wcal} L(\bw)$, which is measured by the excess population risk $L(A(S))-L(\bw^*)$. A standard approach to tackle the excess population risk is to consider the following error decomposition~\citep{bousquet2008tradeoffs}
\begin{equation*}
\ebb_{S, A}[L(A(S))-L(\mathbf{w}^*)]=\ebb_{S, A}[L(A(S))-L_S(A(S))]+\ebb_{S, A}[L_S(A(S))-L_S(\mathbf{w}^*)],
\end{equation*}
where we use the identity $\ebb[L_S(\bw^*)]=L(\bw^*)$ since $\bw^*$ is independent of $S$. We refer to $L(A(S))-L_S(A(S))$ as the generalization gap, which is the difference between testing and training on the output model. We refer to  $L_S(A(S))-L_S(\mathbf{w}^*)$ as the optimization error, which measures the suboptimality of the output model as compared to $\bw^*$ in training.
Generalization gap is a term of central interest in SLT, while optimization error has been extensively studied in optimization theory~\citep{bottou2018optimization}.

In this paper, we leverage the algorithmic stability to study the generalization gap. Intuitively, we say an algorithm is stable if a change of a single example does not lead to a significant change in the behavior of output models. Various stability concepts have been introduced, including uniform stability~\citep{bousquet2002stability}, on-average stability~\citep{shalev2010learnability}, and on-average model stability~\citep{lei2020fine}. In this paper, we focus on the on-average model stability since it can imply fast rates without a Lipschitzness assumption.  %
\begin{definition}[On-average Model Stability \citep{lei2020fine}]\label{def:stability}
    Let $S=\{\bz_1, \ldots, \bz_n\}$, $S^{\prime}=\{\bz_1^{\prime}, \ldots, \bz_n^{\prime}\}$ be drawn independently from $\pbb$. For any $i \in[n]:=\{1, \ldots, n\}$, define
    \begin{equation*}
      S^{(i)}:=\{\bz_1, \ldots, \bz_{i-1}, \bz_i^{\prime}, \bz_{i+1}, \ldots, \bz_n\}
    \end{equation*}
    as the set formed from $S$ by replacing its $i$-th element with $z_i^{\prime}$. Let $\epsilon>0$. We say a randomized algorithm $A$ is  on-average model $\epsilon$-stable if $\mathbb{E}_{S, S^{\prime}, A}[\frac{1}{n} \sum_{i=1}^n\|A(S)-A(S^{(i)})\|^2] \leq \epsilon^2$. %
\end{definition}
Before the analysis, we recall the definitions of convexity, Lipschitzness, and smoothness.
\begin{definition}[Convexity, Lipschitzness, and smoothness]
Let $g:\wcal\mapsto\rbb$ be differentiable. Let $\sigma\geq0, \alpha>0$, and $G>0$.
\begin{enumerate}
  \item[1).] We say $g$ is $\sigma$-strongly convex if
  \[
  g(\bw) \geq g(\bw') + \langle\bw-\bw',\nabla g(\bw')\rangle+\frac{\sigma}{2}\|\bw-\bw'\|^2,\;\forall\bw,\bw'\in\wcal.
  \]
  We say $g$ is convex if the above inequality holds with $\sigma=0$.
  \item[2).] We say $g$ is $G$-Lipschitz continuous if
  $
  |g(\bw)-g(\bw')|\leq G\|\bw-\bw'\|,\;\forall\bw,\bw'\in\wcal.
  $
  \item[3).] We say $g$ is $\alpha$-smooth if
  $
  \|\nabla g(\bw)-\nabla g(\bw')\|\leq \alpha\|\bw-\bw'\|,\;\forall\bw,\bw'\in\wcal.
  $
\end{enumerate}
\end{definition}

We also recall the self-bounding property of smooth functions, which is a key tool to remove the Lipschitzness assumption in our analysis.
\begin{lemma}[\textit{Self-bounding} property~\citep{srebro2010smoothness}]\label{lem:self-bound}
    For $\bz\in \zcal$, if the function $\bw \rightarrow \ell(\bw; \bz)$ is nonnegative and $\alpha$-smooth, then $\|\nabla \ell(\bw ; \bz)\|^2 \leq 2 \alpha \ell(\bw ; \bz)$ for all $\bw\in\wcal$.
\end{lemma}
Our stability analysis also requires the following lemma on the coercivity of smooth and convex functions.
\begin{lemma}[\textit{Coercivity}~\citep{hardt2016train}]\label{lem:coercivity}
    For $\bz\in \zcal$, if the function $\bw \rightarrow \ell(\bw; \bz)$ is convex and $\alpha$-smooth, then
    \[
    \langle \bw-\bw', \nabla \ell(\bw; \bz)-\nabla \ell(\bw';\bz)\rangle \ge \frac{1}{\alpha}\|\nabla \ell(\bw; \bz)-\nabla \ell(\bw'; \bz)\|^2,~\forall \bw,\bw'\in\wcal.
    \]
\end{lemma}
The following lemma gives a quantitative connection between the generalization gap and the on-average model stability for smooth problems. %
\begin{lemma}[\citep{lei2020fine}]\label{lem:general-stablity}
    Let $S, S^{\prime}$, and $S^{(i)}$ be constructed as in Definition \ref{def:stability}, and $\gamma>0$. If for any $\bz\in\zcal$, the function $\mathbf{w} \mapsto \ell(\mathbf{w} ; \bz)$ is nonnegative and $\alpha$-smooth, then
        \begin{equation*}
            \ebb_{S, A}[L(A(S))-L_S(A(S))] \leq \frac{\alpha}{\gamma} \ebb_{S, A}[L_S(A(S))]+\frac{\alpha+\gamma}{2 n} \sum_{i=1}^n \ebb_{S, S^{\prime}, A}[\|A(S^{(i)})-A(S)\|^2].
        \end{equation*}
\end{lemma}

\subsection{The SVRG Method}
We study the SVRG method~\citep{johnson2013accelerating}, which is arguably the most classic VR method for minimizing the empirical risk $L_S(\bw)$ defined in Eq.~\eqref{eq:LS-def}. The detailed implementation of SVRG is summarized in Algorithm~\ref{algo: svrg}, which features a two-loop structure: an outer loop indexed by $t$ that updates the reference point, and an inner loop consisting of $m$ stochastic gradient steps performed using the same reference point. Compared with some other VR methods, such as SAG and SAGA, SVRG reduces the storage cost from $O(nd)$ to $O(d)$ by maintaining only a single reference gradient.

\begin{algorithm}
    \caption{SVRG for minimizing $L_S(\bw)$}
    \label{algo: svrg}
    \begin{algorithmic}
    \REQUIRE initial point $\bw_1 = \bx_m^1 \in \rbb^d$; step size $\eta>0$.
    \FOR{$t = 1, 2, \dots$}
    \STATE Set $\bx_0^{t+1}=\bx_m^t$ (option \uppercase\expandafter{\romannumeral 1}) or $\bx_0^{t+1}=\bw_t$ (option \uppercase\expandafter{\romannumeral 2}).
    \FOR{$k = 0, \dots, m-1$}
    \STATE Sample an index $i_k$ uniformly at random from $[n]$.
    \STATE Set $g_k^{t+1}=\nabla \ell(\bx_k^{t+1}; \bz_{i_k})-\nabla \ell(\bw_t; \bz_{i_k})+\nabla L_S(\bw_t)$.
    \STATE Update $\bx_{k+1}^{t+1}=\bx_{k}^{t+1}-\eta g_k^{t+1}$.
    \ENDFOR
    \STATE Update $\bw_{t+1} \in\{\bx_k^{t+1}\}_{k=0}^{m-1}$ uniformly at random.
    \ENDFOR
    \end{algorithmic}
\end{algorithm}

In the next two sections (Sections~\ref{sec:svrg-stab} and \ref{sec:svrg-epr}), we present a detailed analysis of SVRG, including stability, convergence, and excess population risk (EPR) guarantees. To better illustrate the main ideas, we focus on the behavior of SVRG under the assumption that $\ell(\cdot;\bz)$ is nonnegative, convex, and $\alpha$-smooth for all $\bz\in\zcal$, whereas the analysis for the strongly convex case is deferred to Appendix~\ref{app:sc}. For simplicity, we use option~\uppercase\expandafter{\romannumeral 1} for the inner-loop initialization in the convex setting.

\section{Stability Analysis of SVRG}\label{sec:svrg-stab}
In this section, we conduct stability analysis for SVRG, which together with Lemma~\ref{lem:general-stablity} yields generalization error bounds. In particular, given neighboring datasets $S$ and $S^{(i)}$ as defined in Definition~\ref{def:stability}, we aim at establishing on-average model stability bound $\|\bx_m^{t+1}-\widetilde{\bx}_m^{t+1}\|^2$ for $t\ge 1$, where $\{\bx_k^{t}\}_{t,k}$ and $\{\widetilde{\bx}_k^{t}\}_{t,k}$ are the sequences generated by SVRG applied to $S$ and $S^{(i)}$, respectively. This is much more challenging than the stability analysis of SGD~\citep{hardt2016train,lei2020fine}, and the difficulty mainly arises from: (i) the two-loop structure and (ii) the more complicated gradient constructions of SVRG.

\paragraph{Technical Strategies.} To overcome challenge (i), we adopt an inner-to-outer analysis: we first fix an outer-loop index $t\ge 1$ and analyze the corresponding inner-loop updates, then telescope the resulting inequalities over the inner loop to derive a recursive relation for the outer-loop iterates. For notational convenience, we omit the superscript and denote $\bx_{k} := \bx_{k}^{t+1}$. Then the core step of SVRG reads as
\begin{equation}\label{eq:svrg-update}
    \bx_{k+1}=\bx_k-\eta\big(\nabla \ell(\bx_k; \bz_{i_k})-\nabla \ell(\bw_t; \bz_{i_k})+\nabla L_S(\bw_t)\big),
\end{equation}
which differs from SGD by involving additional gradient terms evaluated at the reference point $\bw_t$. To mitigate the complexity brought by (ii), we define auxiliary operators $G_k$ and $d_k$ via
\begin{equation}\label{def:svrg-G-d}
    G_k:=\bx_k-\eta \nabla \ell(\bx_k; \bz_{i_k}), \quad d_k:=\nabla \ell(\bw_t; \bz_{i_k})- \nabla L_S(\bw_t).
\end{equation}
Then Eq.~\eqref{eq:svrg-update} can be written as
\begin{equation}\label{eq:svrg-core}
    \bx_{k+1} = G_k + \eta d_k.
\end{equation}
Hereafter, we denote by $\e[\cdot]$ the expectation taken over all randomness and $\e_{k}[\cdot]$ the conditional expectation given all history up to time $(t, k)$.
From the uniform randomness of index $i_k\in[n]$, we know that $\e_k[d_k]=0$.
Therefore, the reformulation Eq.~\eqref{eq:svrg-core} can be viewed as an SGD-like update $G_k$ plus a zero-mean correction term $\eta d_k$, which serves as the key starting point for analyzing the stability bound of SVRG within the inner loop.
After telescoping the inner loop, we then introduce a sequence of Lyapunov functions $\{U_t\}_{t\ge 1}$ defined by (we omit the dependency on $i$ in $U_t$ for simplicity):
\begin{equation}\label{eq:svrg-lyapunov}
U_t := \|\bx_m^t - \widetilde{\bx}_m^t\|^2 + \frac{2m\eta^2}{n} \sum_{j=1}^n \big\|\nabla \ell(\bw_t; \bz_j) - \nabla \ell(\bw_t^{(i)}; \bz_j^{(i)})\big\|^2,
\end{equation}
where $\{\bw_t^{(i)}\}_{t\ge 1}$ is the sequence of reference points generated by SVRG applied to $S^{(i)}=\big\{\bz_j^{(i)}\big\}_{j=1}^n$. The introduction of $U_t$ helps balance the additional terms arising from $d_t$, which in turn allows us to derive a recursive relation for $\{U_t\}_{t\ge 1}$ and hence establish the stability bound for SVRG.
Note that the proposed analysis framework based on the reformulation Eq.~\eqref{eq:svrg-core} and the use of Lyapunov functions is fairly general and can be extended to other VR methods with similar gradient structures. As an illustration, we apply a similar approach to analyze SAGA in Section~\ref{sec:saga}.

The following theorem presents the stability bounds for SVRG in the convex setting. The upper bounds involve a summation of the empirical errors at the iterates, and improve for models with small training errors, which is consistent with existing analysis for SGD~\citep{lei2020fine}. To the best of our knowledge, this gives the first non-vacuous stability analysis for SVRG and more generally, for any VR method in the literature.

\begin{theorem}[Stability bounds]\label{thm:svrg-stab-cvx}
Let $S,S^{(i)}$ be defined in Definition~\ref{def:stability}. Let $\{\bx_k^{t}\}_{t,k}$ and $\{\widetilde{\bx}_k^{t}\}_{t,k}$ be the sequences produced by Algorithm~\ref{algo: svrg} with option \uppercase\expandafter{\romannumeral 1} applied to $S$ and $S^{(i)}$, respectively.
Assume that for any $\bz\in\zcal$, the map $\bw\mapsto\ell(\bw;\bz)$ is nonnegative, convex, and $\alpha$-smooth.
If $\eta \leq \frac{1}{2 \alpha}$, then for any $t\ge 1$, we have
\begin{equation}\label{eq:svrg-stab-thm}
\e\big[\|\bx_m^{t+1} - \widetilde{\bx}_m^{t+1}\|^2\big] \leq \frac{16e\alpha m\eta^2}{n} L(\bw_1) + \frac{8e\alpha(4+mt/n)}{n} \eta^2 \sum_{l=1}^t \sum_{k=0}^{m-1} \e\big[L_S(\bx_k^{l+1})\big],
\end{equation}
where $e$ denotes Euler's number.
\end{theorem}

\begin{proof}[Proof of Theorem~\ref{thm:svrg-stab-cvx}]
Consider a fixed outer-loop index $t\ge 1$. Recall the definition of $G_k$ and $d_k$ in Eq.~\eqref{def:svrg-G-d}, and similarly denote $G_k^{(i)}$ and $d_k^{(i)}$ as the counterparts generated from $S^{(i)}$. Then we have Eq.~\eqref{eq:svrg-core} and similarly $\bx_{k+1}^{(i)} = G_k^{(i)} + \eta d_k^{(i)}$.
Expanding $\|\bx_{k+1}-\bx_{k+1}^{(i)}\|^2$ gives
\begin{align}
\|\bx_{k+1}-\bx_{k+1}^{(i)}\|^2=&\|G_k-G_k^{(i)}\|^2+\eta^2\|d_k-d_k^{(i)}\|^2+2 \eta\langle G_k-G_k^{(i)}, d_k-d_k^{(i)}\rangle\notag \\
= & \|G_k-G_k^{(i)}\|^2+\eta^2\|d_k-d_k^{(i)}\|^2+2 \eta\langle \bx_k-\bx_k^{(i)}, d_k-d_k^{(i)}\rangle \notag\\
&-2 \eta^2\langle\nabla \ell(\bx_k; \bz_{i_k})-\nabla \ell(\bx_k^{(i)}; \bz_{i_k}^{(i)}), d_k-d_k^{(i)}\rangle \notag\\
\leq&  \|G_k-G_k^{(i)}\|^2+2 \eta^2\|d_k-d_k^{(i)}\|^2+\eta^2\|\nabla \ell(\bx_k; \bz_{i_k})-\nabla \ell(\bx_k^{(i)}; \bz_{i_k}^{(i)})\|^2\notag\\
&+2 \eta\langle \bx_k-\bx_k^{(i)}, d_k-d_k^{(i)}\rangle,\label{eq:svrg-stab-expand}
\end{align}
where $\bz_j^{(i)}$ represents the $j$-th data sample in $S^{(i)}$ and the inequality follows by applying the Cauchy-Schwarz's inequality
\[
-2\eta^2\langle\nabla \ell(\bx_k; \bz_{i_k})-\nabla \ell(\bx_k^{(i)}; \bz_{i_k}^{(i)}), d_k-d_k^{(i)}\rangle\leq \eta^2\|d_k-d_k^{(i)}\|^2+\eta^2\|\nabla \ell(\bx_k; \bz_{i_k})-\nabla \ell(\bx_k^{(i)}; \bz_{i_k}^{(i)})\|^2.
\]
Noting that $\e_k[d_k]=\e_k[d_k^{(i)}]=0$ from the randomness of $i_k$, we deduce that
\[
\e_k[\langle \bx_k-\bx_k^{(i)}, d_k-d_k^{(i)}\rangle]=0.
\]
Therefore, taking expectations on both sides of Eq.~\eqref{eq:svrg-stab-expand} leads to
\begin{equation}\label{eq:svrg-stab1}
    \e[\|\bx_{k+1}-\bx_{k+1}^{(i)}\|^2]\leq \e[\|G_k-G_k^{(i)}\|^2]+2 \eta^2\e[\|d_k-d_k^{(i)}\|^2]+\eta^2\e[\|\nabla \ell(\bx_k; \bz_{i_k})-\nabla \ell(\bx_k^{(i)}; \bz_{i_k}^{(i)})\|^2].
\end{equation}
To bound $\e\big[\|d_k - d_k^{(i)}\|^2\big]$, note that
\begin{multline*}
  \|d_k-d_k^{(i)}\|^2 =\big\|\nabla \ell(\bw_t; \bz_{i_k})-\nabla \ell(\bw_t^{(i)}; \bz_{i_k}^{(i)})\big\|^2+\big\|\nabla L_S(\bw_t)-\nabla L_{S^{(i)}}(\bw_t^{(i)})\big\|^2 \\
  -2\big\langle\nabla \ell(\bw_t; \bz_{i_k})-\nabla \ell(\bw_t^{(i)}; \bz_{i_k}^{(i)}), \nabla L_S(\bw_t)-\nabla L_{S^{(i)}}(\bw_t^{(i)})\big\rangle.
\end{multline*}
Taking the conditional expectation $\e_k[\cdot]$ on both sides and noting that $\e_k[\nabla \ell(\bw_t; \bz_{i_k})]=\nabla L_S(\bw_t)$ and $\e_k[\nabla \ell(\bw_t^{(i)}; \bz_{i_k}^{(i)})]=\nabla L_{S^{(i)}}(\bw_t^{(i)})$, we have
\begin{align}
\e_k\big[\|d_k - d_k^{(i)}\|^2\big]=&\e_k\big[\big\|\nabla \ell(\bw_t; \bz_{i_k})-\nabla \ell(\bw_t^{(i)}; \bz_{i_k}^{(i)})\big\|^2\big] -\big\|\nabla L_S(\bw_t)-\nabla L_{S^{(i)}}(\bw_t^{(i)})\big\|^2\notag\\
\leq & \frac{1}{n} \sum_{j=1}^n \big\|\nabla \ell(\bw_t; \bz_j) - \nabla \ell(\bw_t^{(i)}; \bz_j^{(i)})\big\|^2 \label{eq:svrg-d-square}.
\end{align}
Next, we bound $\e[\|G_k-G_k^{(i)}\|^2]$. To this end, we obtain from the definitions of $G_k$ and $G_k^{(i)}$ that
\begin{align*}
    \|G_k-G_k^{(i)}\|^2 = & \|\bx_k-\bx_k^{(i)}\|^2+\eta^2\|\nabla \ell(\bx_k; \bz_{i_k})-\nabla \ell(\bx_k^{(i)};\bz_{i_k}^{(i)})\|^2\\
    &-2 \eta\langle \bx_k-\bx_k^{(i)}, \nabla \ell(\bx_k; \bz_{i_k})-\nabla \ell(\bx_k^{(i)};\bz_{i_k}^{(i)})\rangle.
\end{align*}
First, we consider the case $i_k\neq i$ (with probability $1-1/n$), which implies $\bz_{i_k}^{(i)}=\bz_{i_k}$. Then by coercivity of $\ell$ (Lemma \ref{lem:coercivity}), we have
\[
\langle \bx_k-\bx_k^{(i)}, \nabla \ell(\bx_k; \bz_{i_k})-\nabla \ell(\bx_k^{(i)};\bz_{i_k}^{(i)})\rangle \ge \frac{1}{\alpha}\|\nabla \ell(\bx_k; \bz_{i_k})-\nabla \ell(\bx_k^{(i)}; \bz_{i_k})\|^2.
\]
Otherwise if $i_k=i$ (with probability $1/n$), then $\bz_{i_k}=\bz_i$ and $\bz_{i_k}^{(i)}=\bz_i^{\prime}$. By Young's inequality, for any $p>0$ we have
\[
-2 \eta\langle \bx_k-\bx_k^{(i)}, \nabla \ell(\bx_k; \bz_{i_k})-\nabla \ell(\bx_k^{(i)}; \bz_{i_k}^{(i)})\rangle \leq p\|\bx_k-\bx_k^{(i)}\|^2+\frac{\eta^2}{p}\|\nabla \ell(\bx_k; \bz_i)-\nabla \ell(\bx_k^{(i)}; \bz_i^{\prime})\|^2.
\]
Combining the two cases leads to
\begin{align}
\e\big[\|G_k - G_k^{(i)}\|^2\big]\leq & \Big(1+\frac{p}{n}\Big) \e\big[\|\bx_k - \bx_k^{(i)}\|^2\big] - \frac{1}{n}\Big(\frac{2\eta}{\alpha} - \eta^2\Big) \sum_{j \neq i} \e\big[\big\|\nabla \ell(\bx_k; \bz_j) - \nabla \ell(\bx_k^{(i)}; \bz_j)\big\|^2\big] \notag\\
& + \frac{1+1/p}{n} \eta^2 \e\big[\big\|\nabla \ell(\bx_k; \bz_i) - \nabla \ell(\bx_k^{(i)}; \bz_i^{\prime})\big\|^2\big].\label{eq:svrg-G-square}
\end{align}
Applying Eq.~\eqref{eq:svrg-d-square} and Eq.~\eqref{eq:svrg-G-square} on Eq.~\eqref{eq:svrg-stab1} gives
\begin{multline}\label{eq:svrg-stab2}
\e\big[\|\bx_{k+1} - \bx_{k+1}^{(i)}\|^2\big] \leq \Big(1+\frac{p}{n}\Big) \e\big[\|\bx_k - \bx_k^{(i)}\|^2\big] - \frac{2}{n}\Big(\frac{\eta}{\alpha} - \eta^2\Big) \sum_{j \neq i} \e\big[\big\|\nabla \ell(\bx_k; \bz_j) - \nabla \ell(\bx_k^{(i)}; \bz_j)\big\|^2\big] \\
+ \frac{2+1/p}{n} \eta^2 \e\big[\big\|\nabla \ell(\bx_k; \bz_i) - \nabla \ell(\bx_k^{(i)}; \bz_i^{\prime})\big\|^2\big]+ \frac{2\eta^2}{n} \sum_{j=1}^n \e\big[\big\|\nabla \ell(\bw_t; \bz_j) - \nabla \ell(\bw_t^{(i)}; \bz_j^{(i)})\big\|^2\big].
\end{multline}

Next, we analyze for the outer loop.
Telescoping Eq.~\eqref{eq:svrg-stab2} for $k=0, \ldots, m-1$ and noting that $\bx_0^{t+1} = \bx_m^t$ from option \uppercase\expandafter{\romannumeral 1}, we get
\begin{align}
&\e\big[\|\bx_m^{t+1} - \widetilde{\bx}_m^{t+1}\|^2\big] \notag\\
\leq & \Big(1+\frac{p}{n}\Big)^m \e\big[\|\bx_m^t - \widetilde{\bx}_m^t\|^2\big] - \frac{2}{n}\Big(\frac{\eta}{\alpha} - \eta^2\Big) \sum_{k=0}^{m-1}\Big(1+\frac{p}{n}\Big)^{m-1-k} \sum_{j \neq i} \e\big[\big\|\nabla \ell(\bx_k^{t+1}; \bz_j) - \nabla \ell(\widetilde{\bx}_k^{t+1}; \bz_j)\big\|^2\big] \notag\\
& + \frac{2+1/p}{n} \eta^2 \sum_{k=0}^{m-1}\Big(1+\frac{p}{n}\Big)^{m-1-k} \e\big[\big\|\nabla \ell(\bx_k^{t+1}; \bz_i) - \nabla \ell(\widetilde{\bx}_k^{t+1}; \bz_i^{\prime})\big\|^2\big] \notag\\
& + \frac{2\eta^2}{n} \sum_{k=0}^{m-1}\Big(1+\frac{p}{n}\Big)^{m-1-k} \cdot \sum_{j=1}^n \e\big[\big\|\nabla \ell(\bw_t; \bz_j) - \nabla \ell(\bw_t^{(i)}; \bz_j^{(i)})\big\|^2\big].\label{eq:svrg-stab2-1}
\end{align}
Since $\bw_{t+1}$ (resp., $\bw_{t+1}^{(i)}$) is randomly selected from $\{\bx_k^{t+1}\}_{k=0}^{m-1}$ (resp., $\{\widetilde{\bx}_k^{t+1}\}_{k=0}^{m-1}$), we have
\begin{align}
& \frac{1}{n}\sum_{j=1}^{n}\e\big[\big\|\nabla \ell(\bw_{t+1}; \bz_j) - \nabla \ell(\bw_{t+1}^{(i)}; \bz_j^{(i)})\big\|^2\big] = \frac{1}{nm}\sum_{j=1}^{n} \sum_{k=0}^{m-1} \e\big[\big\|\nabla \ell(\bx_k^{t+1}; \bz_j) - \nabla \ell(\widetilde{\bx}_k^{t+1}; \bz_j^{(i)})\big\|^2\big] \notag\\
= & \frac{1}{nm} \sum_{k=0}^{m-1} \sum_{j \neq i} \e\big[\big\|\nabla \ell(\bx_k^{t+1}; \bz_j) - \nabla \ell(\widetilde{\bx}_k^{t+1}; \bz_j)\big\|^2\big]+ \frac{1}{nm} \sum_{k=0}^{m-1} \e\big[\big\|\nabla \ell(\bx_k^{t+1}; \bz_i) - \nabla \ell(\widetilde{\bx}_k^{t+1}; \bz_i^{\prime})\big\|^2\big]\label{eq:svrg-stab-rule}.
\end{align}
Recall Lyapunov functions $\{U_t\}_{t\geq 1}$ defined in Eq.~\eqref{eq:svrg-lyapunov}.
Then applying Eq.~\eqref{eq:svrg-stab2-1} + Eq.~\eqref{eq:svrg-stab-rule} $\times 2 m \eta^2$ yields
\begin{align}
\e\big[U_{t+1}\big] \leq & \Big(1+\frac{p}{n}\Big)^m \e\big[\|\bx_m^t - \widetilde{\bx}_m^t\|^2\big] \notag\\
& - \sum_{k=0}^{m-1}\Big[\frac{2}{n}\Big(\frac{\eta}{\alpha} - \eta^2\Big)\Big(1+\frac{p}{n}\Big)^{m-1-k} - \frac{2\eta^2}{n}\Big] \sum_{j \neq i} \e\big[\big\|\nabla \ell(\bx_k^{t+1}; \bz_j) - \nabla \ell(\widetilde{\bx}_k^{t+1}; \bz_j)\big\|^2\big] \notag\\
& + \sum_{k=0}^{m-1}\Big[\frac{2+1/p}{n} \eta^2 \Big(1+\frac{p}{n}\Big)^{m-1-k} + \frac{2\eta^2}{n}\Big] \e\big[\big\|\nabla \ell(\bx_k^{t+1}; \bz_i) - \nabla \ell(\widetilde{\bx}_k^{t+1}; \bz_i^{\prime})\big\|^2\big]\notag\\
& + \frac{2\eta^2}{n} \sum_{k=0}^{m-1}\Big(1+\frac{p}{n}\Big)^{m-1-k} \cdot \sum_{j=1}^n \e\big[\big\|\nabla \ell(\bw_t; \bz_j) - \nabla \ell(\bw_t^{(i)}; \bz_j^{(i)})\big\|^2\big]\notag\\
\leq &\Big(1+\frac{p}{n}\Big)^m \e\Big[\|\bx_m^t - \widetilde{\bx}_m^t\|^2 + \frac{2m\eta^2}{n} \sum_{j=1}^n \big\|\nabla \ell(\bw_t; \bz_j) - \nabla \ell(\bw_t^{(i)}; \bz_j^{(i)})\big\|^2\Big] \notag\\
& - \sum_{k=0}^{m-1}\Big[\frac{2}{n}\Big(\frac{\eta}{\alpha} - \eta^2\Big)\Big(1+\frac{p}{n}\Big)^{m-1-k} - \frac{2\eta^2}{n}\Big]\sum_{j \neq i} \e\big[\big\|\nabla \ell(\bx_k^{t+1}; \bz_j) - \nabla \ell(\widetilde{\bx}_k^{t+1}; \bz_j)\big\|^2\big] \notag\\
& + \sum_{k=0}^{m-1}\Big[\frac{2+1/p}{n} \eta^2 \Big(1+\frac{p}{n}\Big)^{m-1-k} + \frac{2\eta^2}{n}\Big] \e\big[\big\|\nabla \ell(\bx_k^{t+1}; \bz_i) - \nabla \ell(\widetilde{\bx}_k^{t+1}; \bz_i^{\prime})\big\|^2\big],\label{eq:svrg-stab3}
\end{align}
where the last inequality follows from 
\[
\sum_{k=0}^{m-1}\big(1+\frac{p}{n}\big)^{m-1-k} = \sum_{k=0}^{m-1}\big(1+\frac{p}{n}\big)^k \leq m\big(1+\frac{p}{n}\big)^m.
\]
Since $\eta \leq \frac{1}{2\alpha}$, we have
\begin{equation}\label{eq:svrg-stab4}
\frac{2}{n}\Big(\frac{\eta}{\alpha} - \eta^2\Big)\Big(1+\frac{p}{n}\Big)^{m-1-k} \ge \frac{2}{n}\Big(\frac{\eta}{\alpha} - \eta^2\Big) \ge \frac{2\eta^2}{n},~\forall\,0\leq k \leq m-1.
\end{equation}
Additionally, we have for all $0\leq k \leq m-1$ that
\begin{equation}\label{eq:svrg-stab5}
\frac{2+1/p}{n} \eta^2 \Big(1+\frac{p}{n}\Big)^{m-1-k} + \frac{2\eta^2}{n} \leq \frac{2+1/p}{n} \eta^2 \Big(1+\frac{p}{n}\Big)^{m} + \frac{2\eta^2}{n} \leq \frac{4+1/p}{n}\Big(1+\frac{p}{n}\Big)^{m} \eta^2.
\end{equation}
Moreover, by Schwarz's inequality and the self-bounding property of $\ell$ (Lemma \ref{lem:self-bound}), we obtain
\begin{align}
    \e[\|\nabla \ell(\bx_k^{t+1}; \bz_{i})-\nabla \ell(\widetilde{\bx}_k^{t+1}; \bz_i^{\prime})\|^2] &\leq \e\big[2\|\nabla \ell(\bx_k^{t+1}; \bz_i)\|^2+2 \| \nabla\ell(\widetilde{\bx}_k^{t+1}; \bz_i^{\prime}) \|^2\big]\notag\\
    &\leq 4 \alpha\e[\ell(\bx_k^{t+1}; \bz_i)+\ell(\widetilde{\bx}_k^{t+1}; \bz_i^{\prime})]\notag\\
    &=8\alpha\ebb[\ell(\bx_k^{t+1};\bz_{i})] = 8\alpha\ebb[L_S(\bx_k^{t+1})]\label{eq:svrg-8alpha},
\end{align}
where the last equality holds due to the symmetry of the algorithm
\[
\ebb[\ell(\bx_k^{t+1};\bz_{i})] = \frac{1}{n}\sum_{i=1}^n \ebb[\ell(\bx_k^{t+1};\bz_{i})] = \ebb[L_S(\bx_k^{t+1})].
\]
Applying Eq.~\eqref{eq:svrg-stab4}, Eq.~\eqref{eq:svrg-stab5}, and Eq.~\eqref{eq:svrg-8alpha} to Eq.~\eqref{eq:svrg-stab3}, we obtain
\begin{align}
\e\big[U_{t+1}\big] & \leq \Big(1+\frac{p}{n}\Big)^m \e\big[U_t\big] + \frac{8\alpha(4+1/p)}{n}\Big(1+\frac{p}{n}\Big)^{m} \eta^2 \sum_{k=0}^{m-1} \e\big[L_S(\bx_k^{t+1})\big] \notag\\
& \leq \Big(1+\frac{p}{n}\Big)^{mt} \e[U_1] + \frac{8\alpha(4+1/p)}{n}\Big(1+\frac{p}{n}\Big)^{m} \eta^2 \sum_{l=1}^t \Big(1+\frac{p}{n}\Big)^{m(t-l)} \sum_{k=0}^{m-1} \e\big[L_S(\bx_k^{l+1})\big],\label{eq:svrg-stab6}
\end{align}
where the last inequality is obtained by applying the first inequality recursively.

Setting $p = \frac{n}{mt}$, then it follows from $l\ge 1$ that
\begin{equation}\label{eq:svrg-stab7}
\Big(1+\frac{p}{n}\Big)^{m}\Big(1+\frac{p}{n}\Big)^{m(t-l)} \leq \Big(1+\frac{p}{n}\Big)^{mt} = \Big(1+\frac{1}{mt}\Big)^{mt} \leq e.
\end{equation}
Since $\bx_m^1 = \widetilde{\bx}_m^1 = \bw_1=\bw_1^{(i)}$, we have
\begin{align}\label{eq:svrg-U1}
\e[U_1] & = \frac{2m\eta^2}{n} \sum_{j=1}^n \e\big[\big\|\nabla \ell(\bw_1; \bz_j) - \nabla \ell(\bw_1; \bz_j^{(i)})\big\|^2\big] \notag\\
& = \frac{2m\eta^2}{n} \e\big[\big\|\nabla \ell(\bw_1; \bz_i) - \nabla \ell(\bw_1; \bz_i^{\prime})\big\|^2\big] \leq \frac{16\alpha m\eta^2}{n} \e\big[L_S(\bw_1)\big].
\end{align}
Plugging Eq.~\eqref{eq:svrg-stab7} and Eq.~\eqref{eq:svrg-U1} into Eq.~\eqref{eq:svrg-stab6} and noting that $\e[L_S(\bw_1)] = L(\bw_1)$ yields
\[
\e\big[\|\bx_m^{t+1} - \widetilde{\bx}_m^{t+1}\|^2\big] \leq \e\big[U_{t+1}\big] \leq \frac{16e\alpha m\eta^2}{n} L(\bw_1) + \frac{8e\alpha(4+mt/n)}{n} \eta^2 \sum_{l=1}^t \sum_{k=0}^{m-1} \e\big[L_S(\bx_k^{l+1})\big],
\]
which completes the proof.
\end{proof}

\begin{remark}[Novelty\label{rem:novelty}]
    Note that our stability analysis for SVRG extends beyond existing analyses of SGD or its variants~\citep{hardt2016train,lei2020fine}.
    To estimate the stability bound $\|\bx_{k+1}-\bx_{k+1}^{(i)}\|^2$ within the inner loop, we utilize the reformulation Eq.~\eqref{eq:svrg-core} to derive
    \[
    \|\bx_{k+1}-\bx_{k+1}^{(i)}\|^2=\|G_k-G_k^{(i)}\|^2+\eta^2\|d_k-d_k^{(i)}\|^2+2 \eta\langle G_k-G_k^{(i)}, d_k-d_k^{(i)}\rangle.
    \]
    In the above decomposition, the first term on the RHS corresponds to the stability bound of the standard SGD, which can be controlled using similar arguments as established in~\citep{hardt2016train,lei2020fine}. However, the other two terms are introduced due to the correction term $d_k$ in the SVRG gradient construction, which require careful and refined estimation to prevent deteriorating the overall stability bound.
    For this purpose, we first decompose $2\eta\langle G_k-G_k^{(i)}, d_k-d_k^{(i)}\rangle$ into two components via
    \[
    2\eta\langle G_k-G_k^{(i)}, d_k-d_k^{(i)}\rangle=2\eta\langle \bx_k-\bx_k^{(i)}, d_k-d_k^{(i)}\rangle -2 \eta^2\langle\nabla \ell(\bx_k; \bz_{i_k})-\nabla \ell(\bx_k^{(i)}; \bz_{i_k}^{(i)}), d_k-d_k^{(i)}\rangle,
    \]
    where the first term vanishes after taking expectation since $\e_k[d_k]=0$, and the other term can be further controlled by Cauchy-Schwartz inequality:
    \[
    -2 \eta^2\langle\nabla \ell(\bx_k; \bz_{i_k})-\nabla \ell(\bx_k^{(i)}; \bz_{i_k}^{(i)}), d_k-d_k^{(i)}\rangle\leq\eta^2\|d_k-d_k^{(i)}\|^2+\eta^2\|\nabla \ell(\bx_k; \bz_{i_k})-\nabla \ell(\bx_k^{(i)}; \bz_{i_k}^{(i)})\|^2.
    \]
    Using uniform randomness of index $i_k$, we further have
    \[
    \e_k\big[\|d_k - d_k^{(i)}\|^2\big]\leq \frac{1}{n} \sum_{j=1}^n \big\|\nabla \ell(\bw_t; \bz_j) - \nabla \ell(\bw_t^{(i)}; \bz_j^{(i)})\big\|^2,
    \]
    which is fixed for all $k=0, \ldots, m-1$ within the inner loop. After telescoping the stability bound for the inner loop, we introduce the Lyapunov function $U_t$ in Eq.~\eqref{eq:svrg-lyapunov} to balance the redundant terms, and establish the following recursive relation for the outer loop:
    \begin{align*}
        &\e\big[\|\bx_m^{t+1}-\widetilde{\bx}_m^{t+1}\|^2\big]\leq\e[U_{t+1}] \\
        \leq & \Big(1+\frac{1}{mt}\Big)^m \e[U_t]+\frac{8\alpha(4+mt/n)}{n}\Big(1+\frac{1}{mt}\Big)^{m} \eta^2 \sum_{k=0}^{m-1}\e\big[L_S(\bx_k^{t+1})\big]\\
        \leq &\Big(1+\frac{1}{mt}\Big)^{mt} \e[U_1] + \frac{8\alpha(4+mt/n)}{n}\Big(1+\frac{1}{mt}\Big)^{m} \eta^2 \sum_{l=1}^t\Big(1+\frac{1}{mt}\Big)^{m(t-l)} \sum_{k=0}^{m-1} \e\big[L_S(\bx_k^{l+1})\big]\\
        \leq & e \ebb[U_1] + \frac{8e\alpha(4+mt/n)}{n} \eta^2 \sum_{l=1}^t \sum_{k=0}^{m-1} \e\big[L_S(\bx_k^{l+1})\big].
    \end{align*}
    Different from the stability bound of SGD~\citep{lei2020fine}, we have an additional term $e \ebb[U_1]$, which corresponds to $16e\alpha m\eta^2 L(\bw_1) / n$ in our stability bound Eq.~\eqref{eq:svrg-stab-thm}. Fortunately, as we show in Corollary~\ref{co:svrg-epr-convex}, this additional term does not deteriorate the excess population risk bound for properly chosen parameters.
\end{remark}

\subsection{Discussion with Existing Work}
In this subsection, we present a comparison of our analyses to the existing stability analyses of SVRG \citep{charles2018stability, liu2024general, meng2017generalization} and SCAFFOLD \citep{sun2024understanding}, which extends SVRG to the context of federated learning.

The work \citep{charles2018stability, meng2017generalization} studied SVRG by interpreting it as a black-box optimization algorithm, and related it to convergence rates under some additional assumptions, such as the strong convexity and the Polyak-Lojasiewicz (PL) condition. Specifically, under a $\mu$-PL condition and $G$-Lipschitzness assumption, the work \citep{charles2018stability} showed that SVRG is $O\big(G(\epsilon_{c}/\mu)^{1/2}+G^2/(\mu n)\big)$-uniformly stable, where $\epsilon_c>0$ denotes the optimization error of SVRG. Consequently, their approach fails to exploit the specific algorithmic structure of SVRG, leading to loose guarantees. In contrast, our analysis fully leverages the iterative schemes of SVRG to get tighter bounds.
For convex, $\alpha$-smooth and $G$-Lipschitz problems, the work~\citep{liu2024general} showed that zeroth-order SVRG with $\eta\leq C/(tm)$ is $O(G^2t^{2C\alpha}/n)$-uniformly stable, where $t>0$ is the outer-loop index and $m>0$ denotes the inner-loop iteration number. The recent work~\citep{sun2024understanding} considered a VR method called SCAFFOLD for federated learning.
Under an additional assumption that the variance of $\nabla \ell(\bw;\bz)$ is bounded by $\sigma^2>0$, SCAFFOLD \citep{sun2024understanding} is shown to be $O\big(G(G + \sigma)t^C \log t/n\big)$ on-average stable with a step size $\eta\leq C/(2\alpha t)$. However, these step sizes are too small, for which the methods converge only at extremely slow rates. Therefore, the stability analyses~\citep{liu2024general, sun2024understanding} cannot balance between optimization and generalization.
Moreover, these results rely on a strong Lipschitzness assumption on the loss functions~\citep{charles2018stability, liu2024general, meng2017generalization, sun2024understanding}, which is unnecessary for standard convergence analyses of SVRG. Notably, we remove this Lipschitzness assumption by leveraging the self-bounding property (Lemma \ref{lem:self-bound}). Besides, utilizing the on-average model stability, the established bound incorporates a weighted sum of empirical risks rather than the Lipschitz constant, clarifying how optimization enhances stability and generalization by producing models with small empirical risks.

\section{Excess Population Risk Analysis of SVRG}\label{sec:svrg-epr}
In this section, we study the excess population risk (EPR) bounds of SVRG. To this aim, we first analyze the optimization error of SVRG, and then combine it with the generalization gap obtained from our stability analysis to derive EPR bounds. We show that with properly chosen step size and total iterations, SVRG can achieve EPR bounds of order $O(1/\sqrt{n})$, which is minimax optimal for convex problems~\citep{agarwal2009information}.
\subsection{Convergence Rates}
We start by analyzing the decay of the optimization error for SVRG. The result established in the following theorem relaxes the step size requirement from $\eta<1/(4\alpha)$ in the existing result of~\citep[Theorem 6]{reddi2016stochastic} to $\eta<1/(2\alpha)$.
\begin{theorem}[Optimization error]
    \label{thm:svrg-opt-cvx}
    Assume for any $\bz\in\zcal$, the map $\bw\mapsto\ell(\bw;\bz)$ is nonnegative, convex, and $\alpha$-smooth.
    Let $\{\bx_k^{t}\}_{t, k}$ be the sequence produced by Algorithm~\ref{algo: svrg} with option \uppercase\expandafter{\romannumeral 1} applied to $S$. Denote $\bw_S:=\argmin_{\bw\in\wcal} L_S(\bw)$ and $\bar{\bw}_t := \frac{1}{mt} \sum_{l=1}^t \sum_{k=0}^{m-1} \bx_k^{l+1}$. If $\eta<\frac{1}{2\alpha}$, then for any $t\ge 1$ we have
    \[
    \e_A\big[L_S(\bar{\bw}_t) - L_S(\bw_S)\big] \leq \frac{1}{2Mm\eta t}\big[\|\bw_1 - \bw_S\|^2 + 4\alpha m\eta^2 (L_S(\bw_1) - L_S(\bw_S))\big],
    \]
    where $M>0$ is given by
    \begin{equation}\label{def:M}
        M:=M(\eta)=\begin{cases}
        1,\quad\textit{if}~0<\eta\leq \frac{1}{4\alpha},\\
        2(1-2\alpha\eta),\quad\textit{if}~\frac{1}{4\alpha}<\eta<\frac{1}{2\alpha}.
    \end{cases}
    \end{equation}
\end{theorem}
\begin{proof}[Proof of Theorem~\ref{thm:svrg-opt-cvx}]
Consider a fixed outer-loop index $t\ge 1$ and denote $\bx_{k} := \bx_{k}^{t+1}$ and $g_k:=g_k^{t+1}$, then the core step of SVRG can be written as $\bx_{k+1} = \bx_k - \eta g_k$. Expanding $\|\bx_{k+1} - \bw_S\|^2$ gives
\begin{equation}\label{eq:svrg-opt1}
\|\bx_{k+1} - \bw_S\|^2 = \|\bx_k - \bw_S\|^2 + \eta^2\|g_k\|^2 - 2\eta\big\langle \bx_k - \bw_S, g_k\big\rangle.
\end{equation}
Since $\e_k[g_k]=\nabla L_S(\bx_k)$, we obtain that
\begin{align}
\e_k[\langle \bx_k-\bw_S, g_k\rangle]=&\langle \bx_k-\bw_S, \nabla L_S(\bx_k)\rangle = \frac{1}{n}\sum_{j=1}^n \langle \bx_k-\bw_S, \nabla \ell(\bx_k;\bz_j)\rangle \notag\\
\ge & \frac{1}{n}\sum_{j=1}^n \Big(\ell(\bx_k;\bz_j)-\ell(\bw_S;\bz_j)+\frac{1}{2\alpha}\|\nabla \ell(\bx_k;\bz_j)-\nabla \ell(\bw_S;\bz_j)\|^2\Big)\notag\\
=&L_S(\bx_k)-L_S(\bw_S)+\frac{1}{2\alpha n}\sum_{j=1}^n \|\nabla \ell(\bx_k;\bz_j)-\nabla \ell(\bw_S;\bz_j)\|^2,\label{eq:svrg-opt-1-1}
\end{align}
where the inequality follows from coercivity of each $\ell(\cdot;\bz_j)$ due to the convexity and $\alpha$-smoothness~\citep{johnson2013accelerating}:
\begin{equation}\label{coercivity-ll}
  \ell(\bw;\bz_j) \geq \ell(\bw';\bz_j)+\langle\bw-\bw',\nabla\ell(\bw';\bz_j)\rangle+\frac{1}{2\alpha}\|\nabla \ell(\bw;\bz_j)-\nabla \ell(\bw';\bz_j)\|^2,\quad\forall \bw,\bw'\in\wcal.
\end{equation}
To bound $\|g_k\|^2$, from $\nabla L_S(\bw_S)=0$ and $\e_k[g_k]=\nabla L_S(\bx_k)$, we know that
\begin{align}
    \e_k[\|g_k\|^2] =&\e_k\big[\big\|\nabla \ell(\bx_k; \bz_{i_k})-\nabla \ell(\bw_S;\bz_{i_k})-\nabla \ell(\bw_t; \bz_{i_k})+\nabla \ell(\bw_S;\bz_{i_k})+\nabla L_S(\bw_t)\big\|^2\big]\notag\\
     \leq & 2\e_k[\|\nabla \ell(\bx_k; \bz_{i_k})-\nabla \ell(\bw_S;\bz_{i_k})\|^2]+2\e_k\big[\big\|\nabla \ell(\bw_t; \bz_{i_k})-\nabla \ell(\bw_S;\bz_{i_k})-\nabla L_S(\bw_t)\big\|^2\big]\notag\\
    \leq & 2\e_k[\|\nabla \ell(\bx_k; \bz_{i_k})-\nabla \ell(\bw_S;\bz_{i_k})\|^2]+2\e_k[\|\nabla \ell(\bw_t; \bz_{i_k})-\nabla \ell(\bw_S;\bz_{i_k})\|^2]\notag\\
    = & \frac{2}{n} \sum_{j=1}^n\|\nabla \ell(\bx_k; \bz_j)-\nabla \ell(\bw_S; \bz_j)\|^2+\frac{2}{n} \sum_{j=1}^n\|\nabla \ell(\bw_t; \bz_j)-\nabla \ell(\bw_S; \bz_j)\|^2\label{eq:svrg-opt-1-2},
\end{align}
where we apply Young's inequality to derive the first inequality, and the second inequality holds by noting that $\ebb[\|\zeta-\ebb[\zeta]\|^2]=\ebb[\|\zeta\|^2]-\|\ebb[\zeta]\|^2$ for any random vector $\zeta$.

Taking expectation on both sides of Eq.~\eqref{eq:svrg-opt1} and applying Eq.~\eqref{eq:svrg-opt-1-1} and Eq.~\eqref{eq:svrg-opt-1-2} yields
\begin{multline}\label{eq:svrg-opt2}
\e_A\big[\|\bx_{k+1} - \bw_S\|^2\big] \leq \; \e_A\big[\|\bx_k - \bw_S\|^2\big] + \frac{2\eta^2}{n} \sum_{j=1}^n \e_A\big[\big\|\nabla \ell(\bw_t; \bz_j) - \nabla \ell(\bw_S; \bz_j)\big\|^2\big] \\
+ \frac{2\eta^2 - \eta/\alpha}{n} \sum_{j=1}^n \e_A\big[\big\|\nabla \ell(\bx_k; \bz_j) - \nabla \ell(\bw_S; \bz_j)\big\|^2\big] - 2\eta \e_A\big[L_S(\bx_k) - L_S(\bw_S)\big].
\end{multline}

Next, we analyze the outer loop. Summing Eq.~\eqref{eq:svrg-opt2} over $k = 0, \ldots, m-1$ and noting that $\bx_0^{t+1} = \bx_m^t$ from option \uppercase\expandafter{\romannumeral 1}, we have
\begin{multline}\label{eq:svrg-opt3}
\e_A\big[\|\bx_m^{t+1} - \bw_S\|^2\big] \leq \; \e_A\big[\|\bx_m^t - \bw_S\|^2\big] + \frac{2m\eta^2}{n} \sum_{j=1}^n \e_A\big[\big\|\nabla \ell(\bw_t; \bz_j) - \nabla \ell(\bw_S; \bz_j)\big\|^2\big] \\
+ \frac{2\eta^2 - \eta/\alpha}{n} \sum_{k=0}^{m-1} \sum_{j=1}^n \e_A\big[\big\|\nabla \ell(\bx_k^{t+1}; \bz_j) - \nabla \ell(\bw_S; \bz_j)\big\|^2\big] - 2\eta \sum_{k=0}^{m-1} \e_A\big[L_S(\bx_k^{t+1}) - L_S(\bw_S)\big].
\end{multline}
Since $\bw_{t+1}$ is chosen uniformly at random from $\{\bx_k^{t+1}\}_{k=0}^{m-1}$, it follows that
\begin{equation}\label{eq:svrg-opt4}
\frac{1}{n}\sum_{j=1}^{n}\e_A\big[\big\|\nabla \ell(\bw_{t+1}; \bz_j) - \nabla \ell(\bw_S; \bz_j)\big\|^2\big] = \frac{1}{n m} \sum_{j=1}^{n}\sum_{k=0}^{m-1} \e_A\big[\big\|\nabla \ell(\bx_k^{t+1}; \bz_j) - \nabla \ell(\bw_S; \bz_j)\big\|^2\big].
\end{equation}
Define the Lyapunov function
\[
P_t := \|\bx_m^t - \bw_S\|^2 + \frac{2m\eta^2}{n} \sum_{j=1}^n \big\|\nabla \ell(\bw_t; \bz_j) - \nabla \ell(\bw_S; \bz_j)\big\|^2.
\]
Applying Eq.~\eqref{eq:svrg-opt3} + Eq.~\eqref{eq:svrg-opt4} $\times 2m\eta^2$ gives
\begin{multline}\label{eq:svrg-opt5}
\e_A\big[P_{t+1}\big] \leq \e_A\big[P_t\big]+ \frac{4\eta^2 - \eta/\alpha}{n} \sum_{k=0}^{m-1} \sum_{j=1}^n \e_A\big[\big\|\nabla \ell(\bx_k^{t+1}; \bz_j) - \nabla \ell(\bw_S; \bz_j)\big\|^2\big] \\
- 2\eta \sum_{k=0}^{m-1} \e_A\big[L_S(\bx_k^{t+1}) - L_S(\bw_S)\big].
\end{multline}
We now consider two cases within the range of $0<\eta<\frac{1}{2\alpha}$. Firstly, if $\eta \leq \frac{1}{4\alpha}$, then $4\eta^2 - \eta/\alpha \leq 0$. Hence it follows from Eq.~\eqref{eq:svrg-opt5} that
\[
2\eta \sum_{k=0}^{m-1} \e_A\big[L_S(\bx_k^{t+1}) - L_S(\bw_S)\big] \leq \e_A\big[P_t - P_{t+1}\big].
\]
Otherwise if $\frac{1}{4\alpha} < \eta < \frac{1}{2\alpha}$, then $4\eta^2 - \eta/\alpha > 0$. By Eq.~\eqref{coercivity-ll}, together with $\nabla L_S(\bw_S)=0$, we obtain
\begin{align}
\frac{1}{n} \sum_{j=1}^n \big\|\nabla \ell(\bx; \bz_j) - \nabla \ell(\bw_S; \bz_j)\big\|^2 \leq & \frac{1}{n} \sum_{j=1}^n 2 \alpha\big(\ell(\bx; \bz_j)-\ell(\bw_S; \bz_j)-\langle\nabla \ell(\bw_S; \bz_j), \bx-\bw_S\rangle\big)\notag\\
=& 2\alpha(L_S(\bx) - L_S(\bw_S)),~\forall \bx\in\wcal.\label{eq:svrg-opt5-1}
\end{align}
Applying Eq.~\eqref{eq:svrg-opt5-1} with $\bx = \bx_k^{t+1}$ on Eq.~\eqref{eq:svrg-opt5} implies
\[
4\eta(1 - 2\alpha\eta) \sum_{k=0}^{m-1} \e_A\big[L_S(\bx_k^{t+1}) - L_S(\bw_S)\big] \leq \e_A\big[P_t - P_{t+1}\big].
\]
Combining the above two cases, we know from the definition of $M$ in Eq.~\eqref{def:M} that
\begin{equation}\label{eq:svrg-opt6}
2M\eta \sum_{k=0}^{m-1} \e_A\big[L_S(\bx_k^{t+1}) - L_S(\bw_S)\big] \leq \e_A\big[P_t - P_{t+1}\big].
\end{equation}
Let $t=l$ in Eq.~\eqref{eq:svrg-opt6} and sum over $l=1, \ldots, t$, we obtain
\begin{equation}\label{eq:svrg-opt7}
\sum_{l=1}^t \sum_{k=0}^{m-1} \e_A\big[L_S(\bx_k^{l+1}) - L_S(\bw_S)\big] \leq \frac{1}{2M\eta} P_1.
\end{equation}
Additionally, we have from $\bw_1=\bx_m^1$ and Eq.~\eqref{eq:svrg-opt5-1} with $\bx = \bw_1$ that
\[
P_1 \leq \|\bw_1 - \bw_S\|^2 + 4\alpha m\eta^2(L_S(\bw_1) - L_S(\bw_S)).
\]
Therefore, it follows from Eq.~\eqref{eq:svrg-opt7} and the convexity of $L_S$ that
\begin{align*}
\e_A\big[L_S(\bar{\bw}_t) - L_S(\bw_S)\big] \leq & \frac{1}{mt}\sum_{l=1}^t \sum_{k=0}^{m-1} \e_A\big[L_S(\bx_k^{l+1}) - L_S(\bw_S)\big]\\
\leq & \frac{1}{2Mm\eta t}\big(\|\bw_1 - \bw_S\|^2 + 4\alpha m\eta^2 (L_S(\bw_1) - L_S(\bw_S))\big),
\end{align*}
which completes the proof.
\end{proof}

\begin{remark}\label{rk:svrg-cvx-opt}
The construction of the Lyapunov function $P_t$ in the proof of Theorem~\ref{thm:svrg-opt-cvx} is inspired by the stability analysis of Theorem~\ref{thm:svrg-stab-cvx}. Specifically, we use Eq.~\eqref{eq:svrg-opt4} to absorb the additional gradient terms arising from the reference point $\bw_t$, analogously to how Eq.~\eqref{eq:svrg-stab-rule} enables the Lyapunov function $U_t$ in the stability proof. We also note that~\citep{reddi2016stochastic} uses different Lyapunov functions to establish an optimization error bound for SVRG under the more restrictive step size constraint $\eta<1/(4\alpha)$. In contrast, Theorem~\ref{thm:svrg-opt-cvx} extends this to $\eta<1/(2\alpha)$, offering greater flexibility in step size selection. This improvement stems from the refined lower bound Eq.~\eqref{eq:svrg-opt-1-1}, which jointly exploits convexity and smoothness of $L_S$. In comparison,~\citep{reddi2016stochastic} relies on convexity alone, which yields a weaker bound without the summation term in Eq.~\eqref{eq:svrg-opt-1-1} and necessitates the more restrictive step size condition.
\end{remark}

\subsection{Excess Population Risk Bounds}
Next, we apply Lemma~\ref{lem:general-stablity} with the stability bounds established in Theorem~\ref{thm:svrg-stab-cvx} to derive generalization error bounds. Combining them with the optimization error bounds in Theorem~\ref{thm:svrg-opt-cvx}, we obtain the following EPR bounds. For ease of presentation, we denote $A\lesssim B$ if there exists a universal constant $c>0$ such that $A\leq c B$, and $A\asymp B$ if both $A\lesssim B$ and $B\lesssim A$ hold.
\begin{theorem}[Excess population risk]\label{thm:svrg-epr-cvx}
    Assume that for any $\bz\in\zcal$, the map $\bw\mapsto\ell(\bw;\bz)$ is nonnegative, convex, and $\alpha$-smooth.
    Let $\{\bx_k^t\}_{t,k}$ be the sequence produced by Algorithm~\ref{algo: svrg} with option \uppercase\expandafter{\romannumeral 1} applied to $S$ with $\eta<\frac{1}{2\alpha}$ and denote $\bar{\bw}_t:=\frac{1}{mt}\sum_{l=1}^t\sum_{k=0}^{m-1}\bx_k^{l+1}$.
    Then for any $t\ge 1$ and $\gamma>0$, we have
\begin{align*}
\e\big[L(\bar{\bw}_t)\big] - L(\bw^*)\lesssim & \; \frac{1}{\gamma}\Big(L(\bw^*) + \frac{1}{m\eta t}\big(\e_S\big[\|\bw_1 - \bw_S\|^2\big] + m\eta^2 L(\bw_1)\big)\Big) \notag\\
& + \frac{(1+\gamma)(1+mt/n)}{n} \eta^2\Big(mt L(\bw^*) + \frac{1}{\eta}\big(\e_S\big[\|\bw_1 - \bw_S\|^2\big] + m\eta^2 L(\bw_1)\big)\Big) \notag\\
& + \frac{(1+\gamma)m\eta^2}{n} L(\bw_1) + \frac{1}{m\eta t}\big(\e_S\big[\|\bw_1 - \bw_S\|^2\big] + m\eta^2 L(\bw_1)\big).
\end{align*}
\end{theorem}
\begin{proof}[Proof of Theorem~\ref{thm:svrg-epr-cvx}]
Note that $\frac{\eta}{\alpha} - \eta^2 \geq 0$ since $\eta < \frac{1}{2\alpha}$. For any $s\in[m]$, applying Eq.~\eqref{eq:svrg-stab2} recursively and dropping the negative term yields
\begin{align*}
\e\big[\|\bx_s^{t+1} - \widetilde{\bx}_s^{t+1}\|^2\big] \leq & \Big(1+\frac{p}{n}\Big)^s \e\big[\|\bx_m^t - \widetilde{\bx}_m^t\|^2\big]\\
& + \frac{2+1/p}{n} \eta^2\sum_{k=0}^{s-1}\Big(1+\frac{p}{n}\Big)^{s-1-k} \e\big[\big\|\nabla \ell(\bx_k^{t+1}; \bz_i) - \nabla \ell(\widetilde{\bx}_k^{t+1}; \bz_i^{\prime})\big\|^2\big]\\
& + \frac{2\eta^2}{n}\sum_{k=0}^{s-1}\Big(1+\frac{p}{n}\Big)^{s-1-k} \sum_{j=1}^n \e\big[\big\|\nabla \ell(\bw_t; \bz_j) - \nabla \ell(\bw_t^{(i)}; \bz_j^{(i)})\big\|^2\big],
\end{align*}
which is an increasing function in $s$. Therefore, the RHS of Eq.~\eqref{eq:svrg-stab3} after dropping the negative term also upper bounds $\e\big[\|\bx_s^{t+1} - \widetilde{\bx}_s^{t+1}\|^2\big]$ for all $s\in[m]$. Then it follows from Theorem~\ref{thm:svrg-stab-cvx} and convexity of $\|\cdot\|^2$ that
\begin{equation}\label{eq:svrg-epr-stab}
\e\big[\|\bar{\bw}_t - \bar{\bw}_t^{(i)}\|^2\big] \leq \frac{16e\alpha m\eta^2}{n} L(\bw_1) + \frac{8e\alpha(4+mt/n)\eta^2}{n} \sum_{l=1}^t \sum_{k=0}^{m-1} \e\big[L_S(\bx_k^{l+1})\big].
\end{equation}
Applying Lemma~\ref{lem:general-stablity} and plugging in Eq.~\eqref{eq:svrg-epr-stab} gives
\begin{multline}\label{eq:svrg-gen1}
\e\big[L(\bar{\bw}_t) - L_S(\bar{\bw}_t)\big] \leq \frac{\alpha}{\gamma} \e[L_S(\bar{\bw}_t)]+\frac{\alpha+\gamma}{2 n} \sum_{i=1}^n \e[\|\bar{\bw}_t-\bar{\bw}_t^{(i)}\|^2] \\
\leq \frac{\alpha}{\gamma} \e\big[L_S(\bar{\bw}_t)\big] + \frac{4(\alpha+\gamma)e\alpha(4+mt/n)}{n} \eta^2 \sum_{l=1}^t \sum_{k=0}^{m-1} \e\big[L_S(\bx_k^{l+1})\big]+ \frac{8(\alpha+\gamma)e\alpha m\eta^2}{n} L(\bw_1).
\end{multline}
On the other hand, we know from the proof of Theorem~\ref{thm:svrg-opt-cvx} and $0\leq \e[L_S(\bw_S)] \leq L(\bw^*)$ that
\begin{equation}\label{eq:svrg-gen2}
\e\big[L_S(\bar{\bw}_t)\big] \leq L(\bw^*) + \frac{1}{2Mm\eta t}\big(\e_S\big[\|\bw_1 - \bw_S\|^2\big] + 4\alpha m\eta^2 L(\bw_1)\big)
\end{equation}
and
\begin{equation}\label{eq:svrg-gen3}
\sum_{l=1}^t \sum_{k=0}^{m-1} \e\big[L_S(\bx_k^{l+1})\big] \leq mt L(\bw^*) + \frac{1}{2M\eta}\big(\e_S\big[\|\bw_1 - \bw_S\|^2\big] + 4\alpha m\eta^2 L(\bw_1)\big).
\end{equation}
Plugging Eq.~\eqref{eq:svrg-gen2} and Eq.~\eqref{eq:svrg-gen3} into Eq.~\eqref{eq:svrg-gen1} and omitting constants leads to the generalization error bound
\begin{align}
\e\big[L(\bar{\bw}_t) - L_S(\bar{\bw}_t)\big] \lesssim & \frac{1}{\gamma}\Big(L(\bw^*) + \frac{1}{m\eta t}\big(\e_S\big[\|\bw_1 - \bw_S\|^2\big] + m\eta^2 L(\bw_1)\big)\Big) \notag\\
& + \frac{(1+\gamma)(1+mt/n)}{n} \eta^2 \Big(mt L(\bw^*) + \frac{1}{\eta}\big(\e_S\big[\|\bw_1 - \bw_S\|^2\big] + m\eta^2 L(\bw_1)\big)\Big) \notag\\
& + \frac{(1+\gamma)m\eta^2}{n} L(\bw_1)\label{eq:svrg-gen4}.
\end{align}
Combining Eq.~\eqref{eq:svrg-gen2} and Eq.~\eqref{eq:svrg-gen4}, we obtain the desired excess population risk bound
\begin{align*}
&\e\big[L(\bar{\bw}_t)\big] - L(\bw^*) = \e\big[L(\bar{\bw}_t) - L_S(\bar{\bw}_t)\big] + \e\big[L_S(\bar{\bw}_t) - L(\bw^*)\big] \notag\\
\lesssim & \; \frac{1}{\gamma}\Big(L(\bw^*) + \frac{1}{m\eta t}\big(\e_S\big[\|\bw_1 - \bw_S\|^2\big] + m\eta^2 L(\bw_1)\big)\Big) \notag\\
& + \frac{(1+\gamma)(1+mt/n)}{n} \eta^2\Big(mt L(\bw^*) + \frac{1}{\eta}\big(\e_S\big[\|\bw_1 - \bw_S\|^2\big] + m\eta^2 L(\bw_1)\big)\Big) \notag\\
& + \frac{(1+\gamma)m\eta^2}{n} L(\bw_1) + \frac{1}{m\eta t}\big(\e_S\big[\|\bw_1 - \bw_S\|^2\big] + m\eta^2 L(\bw_1)\big),
\end{align*}
which completes the proof.
\end{proof}

We then specify the choices of $m,t,\gamma$, and $\eta$ to derive the following corollary, which shows that SVRG achieves optimal EPR bounds of order $O\big(\sqrt{1/n}\big)$ in the convex setting. Note the condition $mt \asymp n$ ensures that the total number of inner-loop iterations scales proportionally with $n$.
\begin{corollary}\label{co:svrg-epr-convex}
Let the assumptions in Theorem \ref{thm:svrg-epr-cvx} hold and assume that $\e_S[\|\bw_1-\bw_S\|^2]<\infty$. Then we can take $m t \asymp n$, $\eta \asymp \frac{1}{\sqrt{n L(\bw_1)}}$, and $\gamma = \sqrt{n L(\bw_1)}$ to derive
\[
\ebb[L(\bar{\bw}_t)-L(\bw^*)]\lesssim \sqrt{L(\bw_1)/n}.
\]
\end{corollary}
\begin{proof}[Proof of Corollary~\ref{co:svrg-epr-convex}]
Firstly, it follows from Theorem \ref{thm:svrg-epr-cvx} and boundedness of $\e_S[\|\bw_1-\bw_S\|^2]$ that
\begin{multline}\label{eq:svrg-cor-epr1}
\e\big[L(\bar{\bw}_t)\big] - L(\bw^*)\lesssim \frac{L(\bw^*)}{\gamma}+ \frac{1}{\gamma m\eta t}\big(1 + m\eta^2 L(\bw_1)\big) + \frac{(1+\gamma)(1+mt/n)\eta^2 mt L(\bw^*)}{n}\\
+ \frac{(1+\gamma)(1+mt/n)\eta}{n}\big(1 + m\eta^2 L(\bw_1)\big) + \frac{(1+\gamma)m\eta^2}{n} L(\bw_1) + \frac{1}{m\eta t}\big(1 + m\eta^2 L(\bw_1)\big).
\end{multline}
From the condition on $\gamma$ and $L(\bw^*) \leq L(\bw_1)$, we have that
\begin{equation}\label{eq:svrg-cor-epr2}
\frac{L(\bw^*)}{\gamma} \leq \frac{L(\bw_1)}{\sqrt{n L(\bw_1)}}=\sqrt{\frac{L(\bw_1)}{n}}.
\end{equation}
Furthermore, it follows from $mt\asymp n$ that
\begin{equation}\label{eq:svrg-cor-epr3}
    \frac{(1+\gamma)(1+mt/n)\eta^2 mt L(\bw^*)}{n} \lesssim \frac{\sqrt{nL(\bw_1)} \cdot 1 \cdot \frac{1}{nL(\bw_1)} \cdot n\cdot L(\bw_1)}{n} = \sqrt{\frac{L(\bw_1)}{n}}.
\end{equation}
To proceed, note that $mt\asymp n$ implies $m\lesssim n$, which further yields
\[
1 + m\eta^2 L(\bw_1) \asymp 1 + m \cdot\frac{1}{n L(\bw_1)}\cdot L(\bw_1) \lesssim 1.
\]
Then we can bound the remaining terms in \eqref{eq:svrg-cor-epr1} as follows:
\begin{align}
&\frac{1}{\gamma m\eta t}\big(1 + m\eta^2 L(\bw_1)\big) \lesssim \frac{1}{\gamma m\eta t} \asymp \frac{1}{\sqrt{nL(\bw_1)} \cdot n/\sqrt{nL(\bw_1)}} = \frac{1}{n}\lesssim \sqrt{\frac{L(\bw_1)}{n}},\label{eq:svrg-cor-epr4}\\
&\frac{(1+\gamma)(1+mt/n)\eta}{n}\big(1 + m\eta^2 L(\bw_1)\big) \lesssim \frac{\sqrt{nL(\bw_1)} \cdot 1 / \sqrt{nL(\bw_1)}}{n} = \frac{1}{n} \lesssim \sqrt{\frac{L(\bw_1)}{n}},\label{eq:svrg-cor-epr5}\\
&\frac{(1+\gamma)m\eta^2}{n} L(\bw_1) \lesssim \frac{\sqrt{nL(\bw_1)} \cdot n \cdot \frac{1}{nL(\bw_1)}}{n}\cdot L(\bw_1) = \sqrt{\frac{L(\bw_1)}{n}},\label{eq:svrg-cor-epr6}\\
&\frac{1}{m\eta t}(1 + m\eta^2 L(\bw_1)) \lesssim \frac{1}{m\eta t} \asymp \frac{1}{n/\sqrt{nL(\bw_1)}} = \sqrt{\frac{L(\bw_1)}{n}}.\label{eq:svrg-cor-epr7}
\end{align}
Plugging Eq.~\eqref{eq:svrg-cor-epr2}-Eq.~\eqref{eq:svrg-cor-epr7} into Eq.~\eqref{eq:svrg-cor-epr1} gives
\[
\ebb[L(\bar{\bw}_t)-L(\bw^*)]\lesssim \sqrt{L(\bw_1)/n},
\]
which completes the proof.
\end{proof}

\begin{remark}
    Corollary~\ref{co:svrg-epr-convex} establishes the conditions for selecting the parameters $m, t, \eta$, and $\gamma$ to obtain EPR bounds of order $O(\sqrt{L(\bw_1)/n})$ in the convex case.
    Notably, the derived bound of order $O(1/\sqrt{n})$ remains minimax optimal in the general case. Therefore, our stability and optimization analysis implies optimal EPR bounds.
    Furthermore, it shows that the generalization behavior depends on the initial point $\bw_1$. If we initialize SVRG with a point of small population risk, then our analysis can capture this property to show that SVRG attains a good generalization behavior. This is in spirit similar to the existing data-dependent stability analysis of SGD~\citep{kuzborskij2018data}.
\end{remark}

\section{Extension to SAGA}\label{sec:saga}
In Section~\ref{sec:svrg-stab}, we established stability bounds for SVRG. The key technical ingredients are the reformulation Eq.~\eqref{eq:svrg-core}, which decomposes the SVRG update into an SGD step plus a zero-mean correction term, and the Lyapunov function Eq.~\eqref{eq:svrg-lyapunov}, which absorbs the additional gradient terms arising from the reference point. This analytical approach is not specific to SVRG; rather, it can be extended to a broad class of VR methods.

In this section, we take SAGA \citep{defazio2014saga} as a concrete example and demonstrate how the same approach yields a stability analysis for SAGA. Moreover, by imitating the Lyapunov-based convergence analysis of Theorem~\ref{thm:opt-cvx}, we obtain an optimization bound for SAGA that improves upon the original analysis in \citep{defazio2014saga}. Combining the stability and convergence results, we then derive optimal EPR bounds for SAGA. For brevity, we assume throughout this section that $\ell(\cdot;\bz)$ is nonnegative, convex, and $\alpha$-smooth for all $\bz\in\zcal$, and state only the main results. All proofs are deferred to Appendix~\ref{sec:saga-proofs}. The strongly convex case of SAGA is discussed separately in Appendix~\ref{app:sc}.

The implementation of SAGA is summarized in Algorithm~\ref{algo: saga}. Note that in practice, one can maintain a table of $n$ gradients instead of the reference points $\{\phi_{t,j}\}_{j=1}^n$. Compared with SVRG, SAGA eliminates the outer loop and instead employs a simple single-loop update scheme, at the cost of maintaining a full gradient table of size $O(nd)$ instead of $O(d)$.
\begin{algorithm}
    \caption{SAGA for minimizing $L_S(\bw)$}
    \label{algo: saga}
    \begin{algorithmic}
    \REQUIRE initial point $\bw_1 \in \rbb^d$; step size $\eta>0$.
    \STATE Set $\phi_{1, j}:=\bw_1$ for all $j\in[n]$.
    \FOR{$t = 1, 2, \dots$}
    \STATE Sample an index $i_t$ uniformly at random from $[n]$.
    \STATE Set $\phi_{t+1,i_t}:=\bw_t$ and $\phi_{t+1,j}:=\phi_{t,j}$ for all $j\neq i_t$.
    \STATE Set $g_t=\nabla \ell(\bw_t; \bz_{i_t})-\nabla \ell(\phi_{t, i_t}; \bz_{i_t})+\frac{1}{n} \sum_{j=1}^n \nabla \ell(\phi_{t, j}; \bz_j)$.
    \STATE Update $\bw_{t+1}=\bw_t-\eta g_t$.
    \ENDFOR
    \end{algorithmic}
\end{algorithm}

We first analyze the stability of SAGA. Following the decomposition Eq.~\eqref{def:svrg-G-d}--\eqref{eq:svrg-core} introduced in Section~\ref{sec:svrg-stab}, we similarly define
\begin{equation}\label{eq:def-G-d}
    G_t:=\bw_t-\eta \nabla \ell(\bw_t; \bz_{i_t}), \quad d_t:=\nabla \ell(\phi_{t, i_t}; \bz_{i_t})-\frac{1}{n} \sum_{j=1}^n \nabla \ell(\phi_{t, j}; \bz_j).
\end{equation}
Then it follows from Algorithm~\ref{algo: saga} that $\bw_{t+1}=G_t+\eta d_t$ with $\e_t[d_t]=0$, since $i_t\in[n]$ is randomly selected. To balance the additional gradient terms arising from the reference points $\{\phi_{t,j}\}_{j=1}^n$, we define the following Lyapunov function for SAGA
\begin{equation}\label{eq:saga-lyapunov}
\Phi_t:=\|\bw_t-\bw_t^{(i)}\|^2+2 \eta^2 \sum_{j=1}^n \| \nabla \ell(\phi_{t, j}; \bz_j)-\nabla \ell(\phi_{t,j}^{(i)}; \bz_j^{(i)})\|^2.
\end{equation}

With the above two technical tools, we establish the stability bound for SAGA in the upcoming theorem, which mirrors the structure of Theorem~\ref{thm:svrg-stab-cvx}. To the best of our knowledge, this gives the first stability analysis for SAGA in the literature.

\begin{theorem}[Stability bounds]\label{thm:stab-cvx}
Let $S,S^{(i)}$ be defined in Definition~\ref{def:stability}. Let $\{\bw_t\}_{t\ge 1}$ and $\{\bw_t^{(i)}\}_{t\ge 1}$ be the sequences produced by Algorithm~\ref{algo: saga} applied to $S$ and $S^{(i)}$, respectively.
Assume that for any $\bz\in\zcal$, the map $\bw\mapsto\ell(\bw;\bz)$ is nonnegative, convex, and $\alpha$-smooth.
If $\eta \leq \frac{1}{2 \alpha}$, then for any $t\ge 1$, we have
\[
\e[\|\bw_{t+1}-\bw_{t+1}^{(i)}\|^2] \leq \frac{8e\alpha (4+t / n)}{n} \eta^2 \sum_{k=1}^t \e[L_S(\bw_k)]+16e\alpha\eta^2L(\bw_1),
\]
where $e$ denotes Euler's number.
\end{theorem}

We now analyze the convergence of SAGA using an approach similar to that used in the proof of Theorem~\ref{thm:svrg-opt-cvx}. Our approach substantially simplifies the original analysis in \citep{defazio2014saga}. Moreover, while the original work only provides convergence guarantees for the specific step size choice $\eta=1/(3\alpha)$, our analysis allows for a wider range of step sizes $\eta<1/(2\alpha)$, which provides more flexibility in practice.
\begin{theorem}[Optimization error]\label{thm:opt-cvx}
    Assume for any $\bz\in\zcal$, the map $\bw\mapsto\ell(\bw;\bz)$ is nonnegative, convex, and $\alpha$-smooth.
    Let $\{\bw_t\}_{t\ge 1}$ be the sequence produced by Algorithm~\ref{algo: saga} applied to $S$. Denote $\bw_S:=\argmin_{\bw\in\wcal} L_S(\bw)$ and $\bar{\bw}_t:=\frac{1}{t}\sum_{k=1}^t \bw_k$. If $\eta<\frac{1}{2\alpha}$, then for any $t\ge 1$ we have
    \[
    \e_A[L_S(\bar{\bw}_t)]-L_S(\bw_S) \leq \frac{1}{2M\eta t}\big[\|\bw_1-\bw_S\|^2+4 n \alpha \eta^2(L_S(\bw_1)-L_S(\bw_S))\big],
    \]
    where $M>0$ is given by Eq.~\eqref{def:M}.
\end{theorem}

Now we can plug the stability bounds established in Theorem~\ref{thm:stab-cvx} into Lemma~\ref{lem:general-stablity} to derive generalization bounds, which, together with the optimization error bounds in Theorem~\ref{thm:opt-cvx}, imply the following excess population risk (EPR) bounds.
\begin{theorem}[Excess population risk]\label{thm:epr-cvx}
    Assume that for any $\bz\in\zcal$, the map $\bw\mapsto\ell(\bw;\bz)$ is nonnegative, convex, and $\alpha$-smooth.
    Let $\{\bw_t\}_{t\ge 1}$ be the sequence produced by Algorithm~\ref{algo: saga} applied to $S$ with $\eta<\frac{1}{2\alpha}$ and denote $\bar{\bw}_t:=\frac{1}{t}\sum_{k=1}^t \bw_k$.
    Then for any $t\ge 1$ and $\gamma>0$, we have
\begin{align*}
\e[L(\bar{\bw}_t)]-L(\bw^*) \lesssim & \frac{1}{\gamma}\Big(L(\bw^*)+\frac{1}{\eta t}\big(\e_S[\|\bw_1-\bw_S\|^2]+n \eta^2L(\bw_1)\big)\Big) \notag\\
& + \frac{(1+\gamma)(1+t / n)}{n} \eta^2\Big(t L(\bw^*)+\frac{1}{\eta}\big(\e_S[\|\bw_1-\bw_S\|^2]+n \eta^2L(\bw_1)\big)\Big) \notag \\
& + (1+\gamma) \eta^2 L(\bw_1)+\frac{1}{\eta t}\big(\e_S[\|\bw_1-\bw_S\|^2]+n \eta^2L(\bw_1)\big).
\end{align*}
\end{theorem}

To simplify the results in Theorem~\ref{thm:epr-cvx}, we specify the choices of $t,\gamma$, and $\eta$ to derive optimal EPR bounds of order $O\big(\sqrt{1/n}\big)$ for SAGA.
\begin{corollary}\label{co:epr-convex}
Let the assumptions in Theorem \ref{thm:epr-cvx} hold and assume that $\e_S[\|\bw_1-\bw_S\|^2]<\infty$. Then we can take $t\asymp n$, $\eta \asymp \frac{1}{\sqrt{n L(\bw_1)}}$, and $\gamma = \sqrt{n L(\bw_1)}$ to derive
\[
\ebb[L(\bar{\bw}_t)-L(\bw^*)]\lesssim \sqrt{L(\bw_1)/n}.
\]
\end{corollary}

\section{Experimental Results}\label{sec:simulation}
In this section, we conduct numerical experiments on SVRG (Algorithm~\ref{algo: svrg}) and SAGA (Algorithm~\ref{algo: saga}) to support our theoretical findings. Specifically, we consider the binary logistic regression problem, where the loss function is given by
\[
\ell(\bw; \bz):=\log(1+\exp(-y \langle\bw, \bx\rangle)).
\]
Here, $\bz:=(\bx, y)\in \rbb^d\times \{-1, 1\}$ represents the data sample, where $\bx\in\rbb^d$ corresponds to the feature vector and $y\in\{-1,1\}$ represents the binary label.
We use four real datasets from the LIBSVM library \citep{chang2008libsvm}, whose information is summarized in Table~\ref{tab:dataset}. For multi-class datasets, labels are mapped to binary by assigning values in the lower half to $-1$ and those in the upper half to $1$. Feature vectors are normalized to unit $\ell_2$-norm. For each repetition, $80\%$ of each dataset is randomly selected as the training set $S$, and a single randomly chosen example is replaced to form the neighboring dataset $S'$.
To test stability, we initialize both methods at the same point and run them in parallel on $S$ and $S'$, producing coupled iterate sequences: $\{\bx_k^{t}\}_{t,k}$ and $\{\widetilde{\bx}_k^{t}\}_{t,k}$ for SVRG, and $\{\bw_t\}_{t\ge 1}$ and $\{\bw_t'\}_{t\ge 1}$ for SAGA.
We then record the empirical distance $d_k^t := \|\bx_k^{t} - \widetilde{\bx}_k^{t}\|_2$ (SVRG) and $d_t := \|\bw_t - \bw_t'\|_2$ (SAGA) at each iteration, and report the mean and standard deviation over 100 independent repetitions as a function of the number of training epochs (data passes).

\begin{table}[h]
\centering
\begin{tabular}{|c|c|c|c|c|}
\hline
Dataset & $mnist$ & $a9a$ & $w6a$ & $mushrooms$ \\ \hline
$n$ & 60000 & 32561 & 17188 & 8124 \\ \hline
$d$ & 780 & 123 & 300 & 112 \\ \hline
\end{tabular}
\caption{Summary of datasets used in the experiments. Here, $n$ denotes the sample size and $d$ the number of features.}
\label{tab:dataset}
\end{table}

\paragraph{Comparison of SVRG, SAGA, and SGD with different step sizes.}
We first compare the three methods on the $mnist$ dataset with four constant step sizes $\eta\in\{0.01, 0.05, 0.25, 1\}$, running for $8$ epochs.
The inner-loop length of SVRG is set to $m = n$, where $n$ is the number of training samples.
Figure~\ref{fig:mnist_lr} shows the evolution of the empirical distance during training.
We observe that for both SVRG and SAGA, the stability distance grows with the number of epochs and the step size, which is consistent with our stability bounds established in Theorems~\ref{thm:svrg-stab-cvx} and~\ref{thm:stab-cvx}.
Additionally, SVRG and SAGA exhibit similar behavior to SGD, which suggests that their stability properties are comparable to SGD in this setting.

\begin{figure}[t]
    \centering
    \begin{subfigure}[b]{0.325\linewidth}
        \centering
        \includegraphics[width=\linewidth]{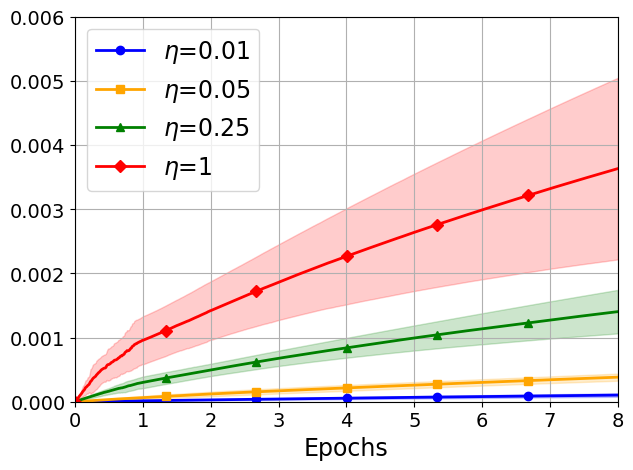}
        \caption{SVRG}
        \label{fig:mnist_svrg_lr}
    \end{subfigure}
    \begin{subfigure}[b]{0.325\linewidth}
        \centering
        \includegraphics[width=\linewidth]{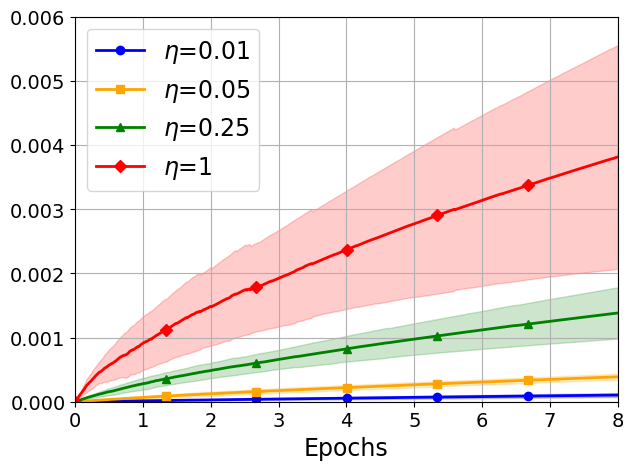}
        \caption{SAGA}
        \label{fig:mnist_saga_lr}
    \end{subfigure}
    \begin{subfigure}[b]{0.325\linewidth}
        \centering
        \includegraphics[width=\linewidth]{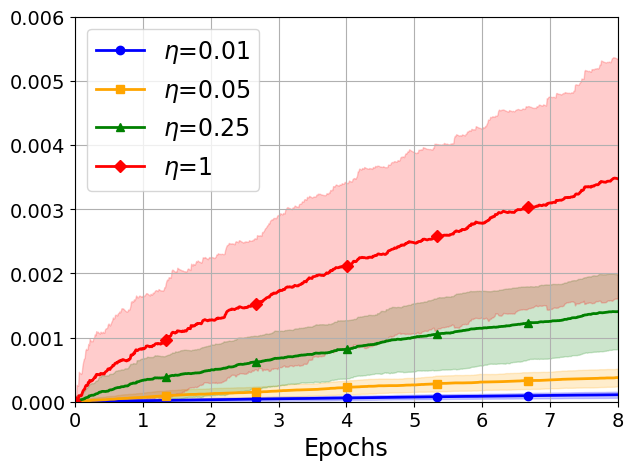}
        \caption{SGD}
        \label{fig:mnist_sgd_lr}
    \end{subfigure}
    \caption{Evolution of empirical distance for SVRG, SAGA, and SGD with different step sizes $\eta\in\{0.01,0.05,0.25,1\}$ on $mnist$ dataset. Shaded regions indicate $\pm 1$ standard deviation over 100 runs.}
    \label{fig:mnist_lr}
\end{figure}

\paragraph{Effect of inner-loop length $m$ in SVRG.}
Next, we investigate the effect of the inner-loop length $m$ on the stability of SVRG. To this end, we fix $\eta = 0.1$ and vary the inner-loop length over $m \in \{0.01n, 0.1n, n, 10n\}$ on the $a9a$ and $w6a$ datasets for $10$ epochs. The remaining settings are the same as those in the previous experiment.
Figure~\ref{fig:svrg_m} reports the mean empirical distance over 100 runs. We observe that the stability of SVRG is relatively robust to the choice of $m$, with slight degradation in performance for large $m$ (e.g., $m=10n$).

\begin{figure}[t]
    \centering
    \begin{subfigure}[b]{0.49\linewidth}
        \centering
        \includegraphics[width=\linewidth]{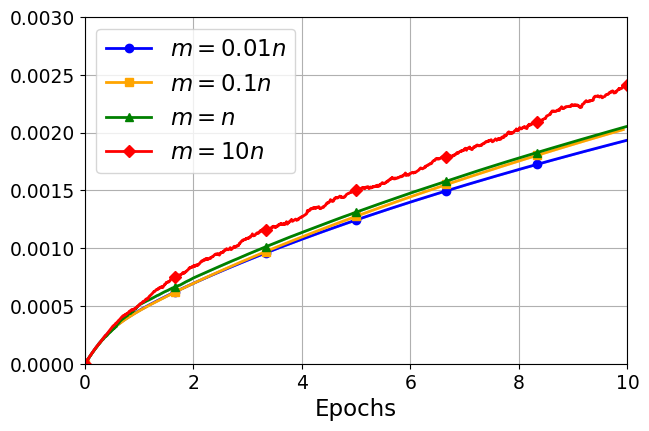}
        \caption{$a9a$}
        \label{fig:a9a_svrg_m}
    \end{subfigure}
    \hfill
    \begin{subfigure}[b]{0.482\linewidth}
        \centering
        \includegraphics[width=\linewidth]{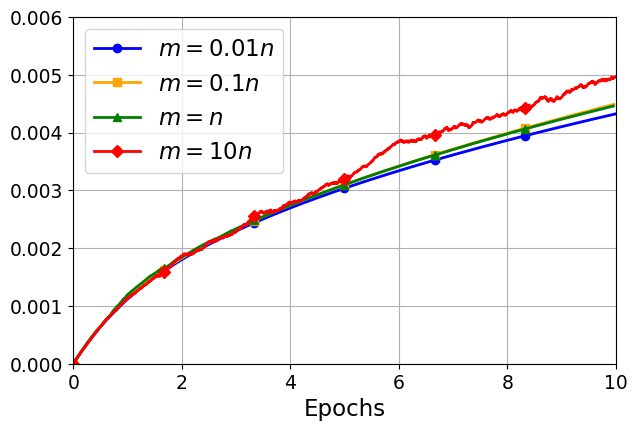}
        \caption{$w6a$}
        \label{fig:w6a_svrg_m}
    \end{subfigure}
    \caption{Evolution of empirical distance for SVRG with fixed step size $\eta=0.1$ and different inner-loop sizes $m\in\{0.01n, 0.1n, n, 10n\}$.}
    \label{fig:svrg_m}
\end{figure}

\paragraph{Empirical distance vs.\ theoretical bound.}
To provide stronger validation of Theorems~\ref{thm:svrg-stab-cvx} and~\ref{thm:stab-cvx}, we conduct additional experiments that directly compare the empirical distance with the theoretical stability bound. To this aim, we run SVRG and SAGA on the $mushrooms$ dataset for $100$ epochs with step size $\eta=0.1$, and inner-loop length $m=n$ for SVRG. All other experimental settings are kept consistent with those in the previous experiment.
Figure~\ref{fig:mushrooms_tb} reports the comparison between the empirical distances and the theoretical bounds.
From the results, we observe that the theoretical bound consistently upper bounds empirical distance throughout the training process. Moreover, the difference decreases as the iteration progresses, indicating that the bound becomes tighter over time.

\begin{figure}[t]
    \centering
    \begin{subfigure}[b]{0.49\linewidth}
        \centering
        \includegraphics[width=\linewidth]{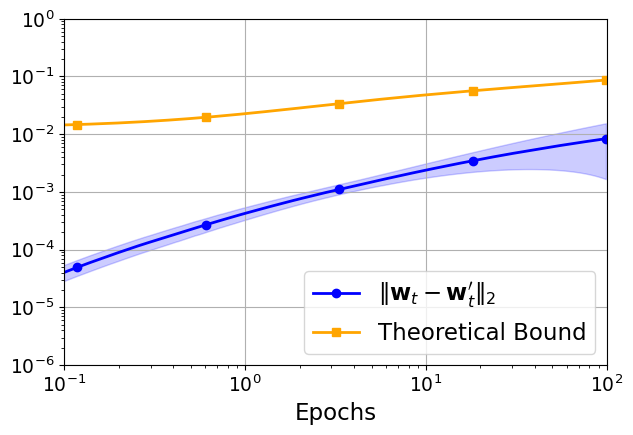}
        \caption{SVRG}
        \label{fig:mushrooms_svrg_tb}
    \end{subfigure}
    \hfill
    \begin{subfigure}[b]{0.485\linewidth}
        \centering
        \includegraphics[width=\linewidth]{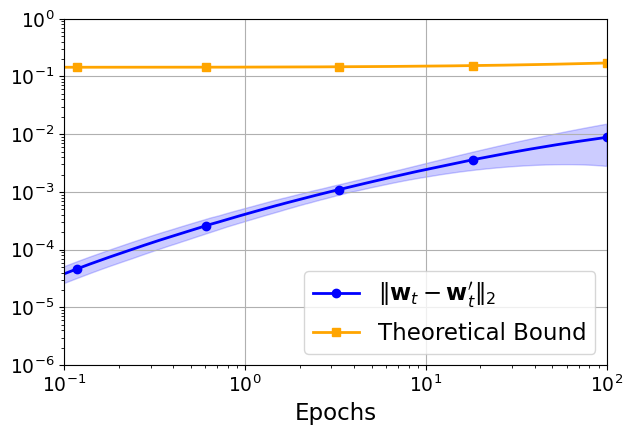}
        \caption{SAGA}
        \label{fig:mushrooms_saga_tb}
    \end{subfigure}
    \caption{Comparison between the empirical distance and the theoretical stability bound established in Theorems~\ref{thm:svrg-stab-cvx} (for SVRG) and~\ref{thm:stab-cvx} (for SAGA), when training logistic regression models on the $mushrooms$ dataset. Shaded regions indicate $\pm 1$ standard deviation over 100 runs.}
    \label{fig:mushrooms_tb}
\end{figure}

\section{Conclusion}\label{sec:conclusion}
In this paper, we develop the learning theory of SVRG, a representative variance reduction (VR) method, by considering both its generalization and convergence behavior. We leverage algorithmic stability to study the generalization and apply techniques from optimization theory to analyze the convergence. By carefully balancing these two components, we further derive optimal excess population risk (EPR) bounds. To overcome the difficulties arising from the complicated algorithmic structure, we introduce several novel techniques, including the reformulation as an SGD update plus a correction term and the construction of tailored Lyapunov functions to derive recursive inequalities in both the stability and convergence analyses. The proposed approach is not specific to SVRG, and we demonstrate its generality by extending the framework to SAGA.

There remain several interesting problems worthy of further investigation. First, we consider convex and strongly convex problems in this paper. A natural extension would be to investigate SVRG and SAGA for nonconvex problems. For instance, applying our framework to analyze the training of neural networks with these methods, potentially by leveraging the weak convexity of the associated loss functions, appears to be a promising research direction. Second, while we focus on SVRG and SAGA here, it would be valuable to extend our analytical techniques to other VR methods. Finally, the EPR bounds developed in this work hold in expectation. Developing high-probability guarantees would provide a deeper understanding of the robustness of these VR methods.

\section*{Acknowledgements}
The work of Y. Lei was supported by the Hong Kong Research Grants Council under the GRF projects 17302624 and 17305425. The work of Z. Wang was supported by the Wallenberg AI, Autonomous Systems and Software Program (WASP) funded by the Knut and Alice Wallenberg Foundation. The work of X. Yuan was supported by the Hong Kong Research Grants Council under the GRF project 17309824.

\appendix
\bigskip\bigskip
{\noindent\Large \bf Appendix}

\section{Proofs for Section~\ref{sec:saga}}\label{sec:saga-proofs}
This appendix contains the proofs of all results stated for SAGA in Section~\ref{sec:saga}.
\subsection{Proof of Stability Bounds\label{sec:proof-stab-cvx}}
\begin{proof}[Proof of Theorem~\ref{thm:stab-cvx}]
Recall the definition of $G_t$ and $d_t$ in Eq.~\eqref{eq:def-G-d}, and similarly denote
\[
    G_t^{(i)}:=\bw_t^{(i)}-\eta \nabla \ell(\bw_t^{(i)}; \bz_{i_t}^{(i)}), \quad d_t^{(i)}:=\nabla \ell(\phi_{t, i_t}^{(i)}; \bz_{i_t}^{(i)})-\frac{1}{n} \sum_{j=1}^n \nabla \ell(\phi_{t, j}^{(i)}; \bz_j^{(i)}),
\]
where $\bz_j^{(i)}$ represents the $j$-th data sample in $S^{(i)}$.
Then from the update formula of SAGA, we have
\begin{equation*}
 \bw_{t+1}=G_t+\eta d_t,\quad\bw_{t+1}^{(i)}=G_t^{(i)}+\eta d_t^{(i)}.
\end{equation*}
Therefore, we obtain by following similar arguments as Eq.~\eqref{eq:svrg-stab-expand}-\eqref{eq:svrg-stab1} that
\begin{equation}\label{eq:cvx-stab2}
    \e[\|\bw_{t+1}-\bw_{t+1}^{(i)}\|^2]\leq \e[\|G_t-G_t^{(i)}\|^2]+2 \eta^2\e[\|d_t-d_t^{(i)}\|^2]+\eta^2\e[\|\nabla \ell(\bw_t; \bz_{i_t})-\nabla \ell(\bw_t^{(i)}; \bz_{i_t}^{(i)})\|^2].
\end{equation}
To bound the square term $\e[\|d_t-d_t^{(i)}\|^2]$, we know from the definitions of $d_t$ and $d_t^{(i)}$ that
\begin{multline*}
  \|d_t-d_t^{(i)}\|^2 =\|\nabla \ell(\phi_{t, i_t}; \bz_{i_t})-\nabla \ell(\phi_{t, i_t}^{(i)}; \bz_{i_t}^{(i)})\|^2+\Big\|\frac{1}{n} \sum_{j=1}^n\big(\nabla \ell(\phi_{t, j}; \bz_j)-\nabla \ell(\phi_{t, j}^{(i)}; \bz_j^{(i)})\big)\Big\|^2 \\
  -2\Big\langle\nabla \ell(\phi_{t, i_t}; \bz_{i_t})-\nabla \ell(\phi_{t, i_t}^{(i)}; \bz_{i_t}^{(i)}), \frac{1}{n} \sum_{j=1}^n\big(\nabla \ell(\phi_{t, j}; \bz_j)-\nabla \ell(\phi_{t, j}^{(i)}; \bz_j^{(i)})\big)\Big\rangle.
\end{multline*}
To bound the last term in the above equality, note that
\begin{align*}
  &\e_t\Big[\Big\langle\nabla \ell(\phi_{t, i_t}; \bz_{i_t})-\nabla \ell(\phi_{t, i_t}^{(i)}; \bz_{i_t}^{(i)}), \frac{1}{n} \sum_{j=1}^n\big(\nabla \ell(\phi_{t, j}; \bz_j)-\nabla \ell(\phi_{t, j}^{(i)}; \bz_j^{(i)})\big)\Big\rangle\Big]\\
  =&\Big\|\frac{1}{n} \sum_{j=1}^n(\nabla \ell(\phi_{t, j}; \bz_j)-\nabla \ell(\phi_{t, j}^{(i)}; \bz_j^{(i)}))\Big\|^2.
\end{align*}
It then follows that
\begin{equation}\label{eq:cvx-stab-dt}
\e_t[\|d_t-d_t^{(i)}\|^2] \!\leq\! \e_t[\|\nabla \ell(\phi_{t, i_t}; \bz_{i_t})\!-\!\nabla \ell(\phi_{t, i_t}^{(i)}; \bz_{i_t}^{(i)})\|^2] \!=\!\frac{1}{n} \sum_{j=1}^n\|\nabla \ell(\phi_{t, j}; \bz_j)\!-\!\nabla \ell(\phi_{t, j}^{(i)}; \bz_j^{(i)})\|^2.
\end{equation}
Moreover, following similar arguments for the derivation of Eq.~\eqref{eq:svrg-G-square}, we have for any $p>0$ that
\begin{multline}\label{eq:cvx-stab-Gt}
    \e[\|G_t-G_t^{(i)}\|^2] \leq (1+p / n) \e[\|\bw_t-\bw_t^{(i)}\|^2]\\-\frac{1}{n}(\frac{2 \eta}{\alpha}-\eta^2) \sum_{j \neq i} \e[\|\nabla \ell(\bw_t; \bz_j)-\nabla \ell(\bw_t^{(i)}; \bz_j)\|^2] +\frac{1+1 / p}{n} \eta^2 \e[\|\nabla \ell(\bw_t; \bz_i)-\nabla \ell(\bw_t^{(i)}; \bz_i^{\prime})\|^2].
\end{multline}
Applying Eq.~\eqref{eq:cvx-stab-dt} and Eq.~\eqref{eq:cvx-stab-Gt} in Eq.~\eqref{eq:cvx-stab2} further gives
\begin{multline}\label{eq:cvx-stab3}
    \e[\|\bw_{t+1}\!-\!\bw_{t+1}^{(i)}\|^2] \!\leq\!\big(1\!+\!\frac{p}{n}\big) \e[\|\bw_t\!-\!\bw_t^{(i)}\|^2]\!-\!\frac{2}{n}(\frac{\eta}{\alpha}\!-\!\eta^2) \!\sum_{j \neq i}\! \e[\|\nabla \ell(\bw_t; \bz_j)\!-\!\nabla \ell(\bw_t^{(i)}; \bz_j)\|^2] \\+\frac{2+1 / p}{n} \eta^2 \e[\|\nabla \ell(\bw_t; \bz_i)\!-\!\nabla \ell(\bw_t^{(i)}; \bz_i^{\prime})\|^2]+\frac{2 \eta^2}{n} \sum_{j=1}^n \e[\|\nabla \ell(\phi_{t, j}; \bz_j)-\nabla \ell(\phi_{t, j}^{(i)}; \bz_j^{(i)})\|^2].
\end{multline}
From the update rule, we have for all $j\in[n]$ that
\begin{equation}\label{eq:phi_update}
   \phi_{t+1, j}= \begin{cases}\bw_t, \quad\text { with prob } \frac{1}{n}, \\ \phi_{t, j}, \quad \text { with prob } 1-\frac{1}{n},\end{cases}
\end{equation}
which implies that
\begin{multline}
    \e\Big[\sum_{j=1}^n\|\nabla \ell(\phi_{t+1, j}; \bz_j)-\nabla \ell(\phi_{t+1,j}^{(i)}; \bz_j^{(i)})\|^2\Big]\\
    = \frac{1}{n} \e\Big[\sum_{j=1}^n\|\nabla \ell(\bw_t; \bz_j)-\nabla \ell(\bw_t^{(i)}; \bz_j^{(i)})\|^2\Big]+\frac{n-1}{n} \e\Big[\sum_{j=1}^n\|\nabla \ell(\phi_{t, j}; \bz_j)-\nabla \ell(\phi_{t, j}^{(i)}; \bz_j^{(i)})\|^2\Big]\label{eq:cvx-stab-phi}.
\end{multline}
Applying Eq.~\eqref{eq:cvx-stab3} + Eq.~\eqref{eq:cvx-stab-phi} $\times 2 \eta^2$, we obtain
\begin{align}
&\e\Big[\|\bw_{t+1}-\bw_{t+1}^{(i)}\|^2+2 \eta^2 \sum_{j=1}^n\|\nabla \ell(\phi_{t+1, j}; \bz_j)-\nabla \ell(\phi_{t+1,j}^{(i)}; \bz_j^{(i)})\|^2\Big]\notag \\
\leq & (1+p / n) \e[\|\bw_t-\bw_t^{(i)}\|^2]-\frac{2}{n}(\frac{\eta}{\alpha}-2 \eta^2) \sum_{j \neq i} \e[\|\nabla \ell(\bw_t; \bz_j)-\nabla \ell(\bw_t^{(i)}; \bz_j)\|^2]\notag \\
& +\frac{4\!+\!1 / p}{n} \eta^2 \e[\|\nabla \ell(\bw_t; \bz_i)\!-\!\nabla \ell(\bw_t^{(i)}; \bz_i^{\prime})\|^2]+2 \eta^2 \sum_{j=1}^n \e[\|\nabla \ell(\phi_{t,j}; \bz_j)-\nabla \ell(\phi_{t,j}^{(i)}; \bz_j^{(i)})\|^2]\label{eq:cvx-stab4}.
\end{align}
Since $\eta \leq \frac{1}{2 \alpha}$, we know $\eta / \alpha-2 \eta^2 \ge 0$. Moreover, by similar arguments as Eq.~\eqref{eq:svrg-8alpha}, we obtain
\begin{equation}\label{eq:cvx-stab-i}
\e[\|\nabla \ell(\bw_t; \bz_{i})-\nabla \ell(\bw_t^{(i)}; \bz_i^{\prime})\|^2] \leq 8\alpha\ebb[L_S(\bw_{t})].
\end{equation}
Then plugging Eq.~\eqref{eq:cvx-stab-i} into Eq.~\eqref{eq:cvx-stab4} leads to
\begin{multline}
\e\Big[\|\bw_{t+1}-\bw_{t+1}^{(i)}\|^2+2 \eta^2 \sum_{j=1}^n\|\nabla \ell(\phi_{t+1, j}; \bz_j)-\nabla \ell(\phi_{t+1,j}^{(i)}; \bz_j^{(i)})\|^2\Big]
\leq  \\(1+p / n) \e[\|\bw_t-\bw_t^{(i)}\|^2]+2 \eta^2 \sum_{j=1}^n \e[\|\nabla \ell(\phi_{t,j}; \bz_j)-\nabla \ell(\phi_{t,j}^{(i)}; \bz_j^{(i)})\|^2]+\frac{8 \alpha(4+1 / p)}{n} \eta^2 \e[L_S(\bw_t)]\label{eq:cvx-stab5}.
\end{multline}
Recall the Lyapunov function $\Phi_t$ defined by Eq.~\eqref{eq:saga-lyapunov}.
Then it follows from Eq.~\eqref{eq:cvx-stab5} that
\begin{equation}\label{eq:cvx-stab5-1}
    \e[\Phi_{t+1}] \leq(1+p / n) \e[\Phi_t]+\frac{8 \alpha(4+1 / p)}{n} \eta^2 \e[L_S(\bw_t)].
\end{equation}
Hence, applying Eq.~\eqref{eq:cvx-stab5-1} recursively gives
\begin{equation}\label{eq:cvx-stab6}
\e[\Phi_{t+1}] \leq \sum_{k=1}^t(1+p / n)^{t-k} \frac{8 \alpha(4+1 / p)}{n} \eta^2 \e[L_S(\bw_k)]+(1+p / n)^t \e[\Phi_1].
\end{equation}
Let $p=n/t$. Then we know that $(1+p / n)^{t-k} \leq (1+p / n)^t = (1+1 / t)^t \leq e, \forall k \ge 1$. Moreover, from $\bw_1=\bw_1^{(i)}$ and $\phi_{1, j}=\bw_1,~\phi_{1, j}^{(i)}=\bw_1^{(i)}$ for all $j\in[n]$, we know that
\[
\Phi_1=2 \eta^2 \sum_{j=1}^n \| \nabla \ell(\bw_1; \bz_j)-\nabla \ell(\bw_1; \bz_j^{(i)})\|^2=2 \eta^2 \| \nabla \ell(\bw_1; \bz_i)-\nabla \ell(\bw_1; \bz'_i)\|^2,
\]
which together with Eq.~\eqref{eq:cvx-stab-i} implies that
\[
\e[\Phi_1] \leq 16\alpha\eta^2\e[L_S(\bw_1)]=16\alpha\eta^2L(\bw_1).
\]
Therefore, we obtain from Eq.~\eqref{eq:cvx-stab6} that
\begin{equation*}
  \e[\|\bw_{t+1}-\bw_{t+1}^{(i)}\|^2] \leq \e[\Phi_{t+1}] \leq \frac{8e\alpha (4+t / n)}{n} \eta^2 \sum_{k=1}^t \e[L_S(\bw_k)]+16e\alpha\eta^2L(\bw_1),
\end{equation*}
which completes the proof.
\end{proof}

\subsection{Proof of Convergence Rates\label{sec:proof-conv-cvx}}
We now turn to convergence analysis. To this aim, we introduce a different Lyapunov function.
\begin{proof}[Proof of Theorem~\ref{thm:opt-cvx}]
    From the update formula $\bw_{t+1}=\bw_t-\eta g_t$, we can expand $\|\bw_{t+1}-\bw_S\|^2$ via
\begin{align}
\|\bw_{t+1}-\bw_S\|^2 & =\|\bw_{t+1}-\bw_t\|^2+\|\bw_t-\bw_S\|^2+2\langle \bw_{t+1}-\bw_t, \bw_t-\bw_S\rangle \notag\\
& =\eta^2\|g_t\|^2+\|\bw_t-\bw_S\|^2-2 \eta\langle \bw_t-\bw_S, g_t\rangle\label{eq:cvx-opt1},
\end{align}
where $g_t=\nabla \ell(\bw_t; \bz_{i_t})-\nabla \ell(\phi_{t, i_t}; \bz_{i_t})+\frac{1}{n} \sum_{j=1}^n \nabla \ell(\phi_{t, j}; \bz_j)$.
Using similar arguments as Eq.~\eqref{eq:svrg-opt-1-1} (with $\bx_k$ replaced by $\bw_t$), we obtain
\begin{equation}\label{eq:cvx-opt-1}
\e_t[\langle \bw_t-\bw_S, g_t\rangle] \ge L_S(\bw_t)-L_S(\bw_S)+\frac{1}{2\alpha n}\sum_{j=1}^n \|\nabla \ell(\bw_t;\bz_j)-\nabla \ell(\bw_S;\bz_j)\|^2.
\end{equation}
To bound $\|g_t\|^2$, from $\nabla L_S(\bw_S)=0$ and $\e_t[g_t]=\nabla L_S(\bw_t)$, we know that
\begin{align}
    \e_t[\|g_t\|^2] =&\e_t\Big[\!\Big\|\nabla \ell(\bw_t; \bz_{i_t})\!-\!\nabla \ell(\bw_S;\bz_{i_t})\!-\!\nabla \ell(\phi_{t, i_t}; \bz_{i_t})\!+\!\nabla \ell(\bw_S;\bz_{i_t})\!+\!\frac{1}{n} \sum_{j=1}^n \nabla \ell(\phi_{t, j}; \bz_j)\Big\|^2\!\Big]\notag\\
     \leq & 2\e_t[\|\nabla \ell(\bw_t; \bz_{i_t})-\nabla \ell(\bw_S;\bz_{i_t})\|^2]\notag\\
     &+2\e_t\Big[\Big\|\nabla \ell(\phi_{t, i_t}; \bz_{i_t})\!-\!\nabla \ell(\bw_S;\bz_{i_t})\!-\!\Big[\frac{1}{n}\! \sum_{j=1}^n\! \nabla \ell(\phi_{t, j}; \bz_j)\!-\!\nabla L_S(\bw_S)\Big]\Big\|^2\Big]\notag\\
    \leq & 2\e_t[\|\nabla \ell(\bw_t; \bz_{i_t})-\nabla \ell(\bw_S;\bz_{i_t})\|^2]+2\e_t[\|\nabla \ell(\phi_{t, i_t}; \bz_{i_t})-\nabla \ell(\bw_S;\bz_{i_t})\|^2]\notag\\
    = & \frac{2}{n} \sum_{j=1}^n\|\nabla \ell(\bw_t; \bz_j)-\nabla \ell(\bw_S; \bz_j)\|^2+\frac{2}{n} \sum_{j=1}^n\|\nabla \ell(\phi_{t, j}; \bz_j)-\nabla \ell(\bw_S; \bz_j)\|^2\label{eq:cvx-opt-1-1},
\end{align}

Taking expectations over both sides of Eq.~\eqref{eq:cvx-opt1}, and applying Eq.~\eqref{eq:cvx-opt-1} and Eq.~\eqref{eq:cvx-opt-1-1} gives %
\begin{align}
    \e_A[\|\bw_{t+1}-\bw_S\|^2] \leq&\e_A[\|\bw_t-\bw_S\|^2]+\e_A\Big[\frac{2\eta^2}{n} \sum_{j=1}^n\|\nabla \ell(\phi_{t, j}; \bz_j)-\nabla \ell(\bw_S; \bz_j)\|^2\Big]\notag\\
    &+\e_A\Big[\frac{2\eta^2-\eta/\alpha}{n} \sum_{j=1}^n\|\nabla \ell(\bw_t; \bz_j)-\nabla \ell(\bw_S; \bz_j)\|^2\Big]-2 \eta\e_A[L_S(\bw_t)-L_S(\bw_S)].\label{eq:cvx-opt2}
\end{align}
On the other hand, we know from the update rule Eq.~\eqref{eq:phi_update} that
\begin{align}
& \e_t\Big[\sum_{j=1}^n\|\nabla \ell(\phi_{t+1, j}; \bz_j)-\nabla \ell(\bw_S; \bz_j)\|^2\Big]\notag \\
= & \frac{1}{n} \sum_{j=1}^n\|\nabla \ell(\bw_t; \bz_j)-\nabla \ell(\bw_S; \bz_j)\|^2+\frac{n-1}{n} \sum_{j=1}^n\|\nabla \ell(\phi_{t, j}; \bz_j)-\nabla \ell(\bw_S; \bz_j)\|^2.\label{eq:cvx-opt-3}
\end{align}
Multiplying both sides of Eq.~\eqref{eq:cvx-opt-3} by $2\eta^2$ and adding it to Eq.~\eqref{eq:cvx-opt2}, we obtain
\begin{align}
& \e_A\Big[\|\bw_{t+1}-\bw_S\|^2+2 \eta^2 \sum_{j=1}^n\|\nabla \ell(\phi_{t+1, j}; \bz_j)-\nabla \ell(\bw_S; \bz_j)\|^2\Big]\notag \\
\leq & \e_A\Big[\|\bw_t-\bw_S\|^2+2 \eta^2 \sum_{j=1}^n\|\nabla \ell(\phi_{t, j}; \bz_j)-\nabla \ell(\bw_S; \bz_j)\|^2\Big]\notag \\
& +\e_A\Big[\frac{4\eta^2-\eta/\alpha}{n} \sum_{j=1}^n\|\nabla \ell(\bw_t; \bz_j)-\nabla \ell(\bw_S; \bz_j)\|^2-2 \eta\big(L_S(\bw_t)-L_S(\bw_S)\big)\Big].\label{eq:cvx-opt3}
\end{align}

Define the Lyapunov function $V_t$ via
\[
V_t:=\|\bw_t-\bw_S\|^2+2 \eta^2 \sum_{j=1}^n\|\nabla \ell(\phi_{t, j}; \bz_j)-\nabla \ell(\bw_S; \bz_j)\|^2.
\]
Next, we discuss two cases for the range of $0<\eta<\frac{1}{2\alpha}$. Firstly, consider $\eta\leq\frac{1}{4\alpha}$, then it follows that $4\eta^2-\eta/\alpha\leq 0$. Therefore, we derive from Eq.~\eqref{eq:cvx-opt3} that
\begin{equation}\label{eq:cvx-opt-case1}
    2 \eta \e_A[L_S(\bw_t)-L_S(\bw_S)] \leq \e_A[V_t]-\e_A[V_{t+1}].
\end{equation}
Otherwise if $\frac{1}{4\alpha}<\eta<\frac{1}{2\alpha}$, letting $\bx := \bw_t$ into Eq.~\eqref{eq:svrg-opt5-1} yields
\begin{equation}\label{eq:cvx-opt-4}
\frac{1}{n} \sum_{j=1}^n\|\nabla \ell(\bw_t; \bz_j)-\nabla \ell(\bw_S; \bz_j)\|^2 \leq 2 \alpha(L_S(\bw_t)-L_S(\bw_S)).
\end{equation}
Plugging Eq.~\eqref{eq:cvx-opt-4} into Eq.~\eqref{eq:cvx-opt3} leads to
\begin{equation}\label{eq:cvx-opt-case2}
4 \eta(1-2\alpha \eta) \e_A[L_S(\bw_t)-L_S(\bw_S)] \leq \e_A[V_t]-\e_A[V_{t+1}].
\end{equation}
Combining the above two cases Eq.~\eqref{eq:cvx-opt-case1} and Eq.~\eqref{eq:cvx-opt-case2}, together with the definition of $M$ in Eq.~\eqref{def:M}, we derive
\begin{equation}\label{eq:cvx-opt5}
    2M\eta \e_A[L_S(\bw_t)-L_S(\bw_S)] \leq \e_A[V_t]-\e_A[V_{t+1}].
\end{equation}
Summing over Eq.~\eqref{eq:cvx-opt5} from $1$ to $t$ gives
\begin{equation}\label{eq:cvx-opt6}
2 M \eta\sum_{k=1}^t \e_A[L_S(\bw_k)-L_S(\bw_S)] \leq V_1-\e_A[V_{t+1}] \leq V_1.
\end{equation}
To estimate $V_1$, since $\phi_{1,j}=\bw_1$ for all $j\in[n]$, we know from Eq.~\eqref{eq:cvx-opt-4} that
\begin{align*}
V_1 & =\|\bw_1-\bw_S\|^2+2 \eta^2 \sum_{j=1}^n\|\nabla \ell(\bw_1; \bz_j)-\nabla \ell(\bw_S; \bz_j)\|^2 \\
& \leq\|\bw_1-\bw_S\|^2+4 n \alpha \eta^2(L_S(\bw_1)-L_S(\bw_S)).
\end{align*}
Therefore, it follows from Eq.~\eqref{eq:cvx-opt6} and the convexity of $L_S$ that
\begin{align}
\e_A[L_S(\bar{\bw}_t)]-L_S(\bw_S) \leq &\frac{1}{t}\sum_{k=1}^t \e_A[L_S(\bw_k)-L_S(\bw_S)]\notag\\
\leq & \frac{1}{2M \eta t}\big[\|\bw_1-\bw_S\|^2+4 n \alpha \eta^2(L_S(\bw_1)-L_S(\bw_S))\big] \label{eq:cvx-opt7},
\end{align}
which completes the proof.
\end{proof}

\subsection{Proofs of EPR Bounds\label{sec:proof-exc-cvx}}
We combine the previous results on stability and convergence to derive EPR bounds in Theorem~\ref{thm:epr-cvx}.
\begin{proof}[Proof of Theorem~\ref{thm:epr-cvx}]
It follows from Theorem \ref{thm:stab-cvx} that
\begin{equation}\label{eq:epr-cvx1}
    \e[\|\bw_{t+1}-\bw_{t+1}^{(i)}\|^2] \leq \frac{8e\alpha (4+t / n)}{n} \eta^2 \sum_{k=1}^t \e[L_S(\bw_k)]+16e\alpha\eta^2L(\bw_1).
\end{equation}
Note that the RHS of Eq.~\eqref{eq:epr-cvx1} is an increasing function in terms of $t$. Then, from the convexity of $\|\cdot\|^2$ we know
\begin{equation}\label{eq:epr-cvx2}
    \e[\|\bar{\bw}_t-\bar{\bw}_t^{(i)}\|^2]\leq\frac{1}{t}\sum_{k=1}^t\e[\|\bw_k-\bw_k^{(i)}\|^2]\leq \frac{8e\alpha (4+t / n)}{n} \eta^2 \sum_{k=1}^t \e[L_S(\bw_k)]+16e\alpha\eta^2L(\bw_1).
\end{equation}
Applying Lemma \ref{lem:general-stablity} and plugging in Eq.~\eqref{eq:epr-cvx2}, we know that
\begin{align}
\e[L(\bar{\bw}_t)-&L_S(\bar{\bw}_t)] \leq  \frac{\alpha}{\gamma} \e[L_S(\bar{\bw}_t)]+\frac{\alpha+\gamma}{2 n} \sum_{i=1}^n \e[\|\bar{\bw}_t-\bar{\bw}_t^{(i)}\|^2] \notag\\
\leq & \frac{\alpha}{\gamma} \e[L_S(\bar{\bw}_t)]+\frac{4(\alpha+\gamma) e \alpha(4+t / n)}{n} \eta^2 \sum_{k=1}^t \e[L_S(\bw_k)]+8(\alpha+\gamma) e \alpha \eta^2 L(\bw_1)\label{eq:epr-cvx3}.
\end{align}
Next we bound the term involving $\e[L_S(\bar{\bw}_t)]$ and $\sum_{k=1}^t \e[L_S(\bw_k)]$. Since $\bw_S$ is a minimizer of $L_S$, we obtain
\begin{equation}\label{eq:epr-cvx2-0}
    L(\bw^*) = \e[L_S(\bw^*)] \ge \e[L_S(\bw_S)].
\end{equation}
Note that $\e[L_S(\bw_1)]=L(\bw_1)$ and $L_S(\bw_S)\ge 0$ since the loss function $\ell$ is nonnegative. Then,  combining Eq.~\eqref{eq:cvx-opt7} and Eq.~\eqref{eq:epr-cvx2-0} leads to
\begin{equation}\label{eq:epr-cvx4}
\e[L_S(\bar{\bw}_t)] \leq L(\bw^*)+\frac{1}{2M\eta t}\big(\e_S[\|\bw_1-\bw_S\|^2]+4 n \alpha \eta^2 L(\bw_1)\big)
\end{equation}
and
\begin{equation}\label{eq:epr-cvx5}
\sum_{k=1}^t \e[L_S(\bw_k)] \leq t L(\bw^*)+\frac{1}{2M\eta}\big(\e_S[\|\bw_1-\bw_S\|^2]+4 n \alpha \eta^2 L(\bw_1)\big).
\end{equation}
Plugging Eq.~\eqref{eq:epr-cvx4} and Eq.~\eqref{eq:epr-cvx5} into Eq.~\eqref{eq:epr-cvx3} and omitting the factor $M\asymp1$ gives
\begin{align}
&\e[L(\bar{\bw}_t)-L_S(\bar{\bw}_t)] \lesssim \frac{1}{\gamma}\Big(L(\bw^*)+\frac{1}{\eta t}\big(\e_S[\|\bw_1-\bw_S\|^2]+n \eta^2L(\bw_1)\big)\Big)\notag \\
+ & \frac{(1\!+\!\gamma)(1\!+\!t / n)}{n} \eta^2\Big(t L(\bw^*)\!+\!\frac{1}{\eta}\big(\e_S[\|\bw_1\!-\!\bw_S\|^2]\!+\!n \eta^2L(\bw_1)\big)\Big)\!+\! (1\!+\!\gamma) \eta^2 L(\bw_1)\label{eq:epr-cvx-gen}.
\end{align}
Combining the generalization error in Eq.~\eqref{eq:epr-cvx-gen} and optimization error in Eq.~\eqref{eq:epr-cvx4} leads to
\begin{align*}
\e[L(\bar{\bw}_t)]-L(\bw^*) = &\e[L(\bar{\bw}_t)-L_S(\bar{\bw}_t)] + \e[L_S(\bar{\bw}_t)-L(\bw^*)]\notag\\
\lesssim & \frac{1}{\gamma}\Big(L(\bw^*)+\frac{1}{\eta t}\big(\e_S[\|\bw_1-\bw_S\|^2]+n \eta^2L(\bw_1)\big)\Big) \notag\\
& + \!\frac{(1\!+\!\gamma)(1\!+\!t / n)}{n} \eta^2\Big(t L(\bw^*)\!+\!\frac{1}{\eta}\big(\e_S[\|\bw_1\!-\!\bw_S\|^2]\!+\!n \eta^2L(\bw_1)\big)\Big) \notag \\
& + (1+\gamma) \eta^2 L(\bw_1)+\frac{1}{\eta t}\big(\e_S[\|\bw_1-\bw_S\|^2]+n \eta^2L(\bw_1)\big),%
\end{align*}
which completes the proof.
\end{proof}
We now specialize the parameters in Theorem~\ref{thm:epr-cvx} to prove Corollary~\ref{co:epr-convex}.
\begin{proof}[Proof of Corollary~\ref{co:epr-convex}]
Firstly, it follows from Theorem \ref{thm:epr-cvx} and boundedness of $\e_S[\|\bw_1-\bw_S\|^2]$ that
\begin{multline}\label{eq:epr-cvx6}
    \e[L(\bar{\bw}_t)]-L(\bw^*) \lesssim \frac{L(\bw^*)}{\gamma}+\frac{1}{\gamma\eta t}\big(1+n \eta^2 L(\bw_1)\big) + \frac{(1+\gamma)(1+t / n)\eta^2 t L(\bw^*)}{n}\\
+ \frac{(1+\gamma)(1+t / n)\eta}{n}\big(1+n \eta^2 L(\bw_1)\big) + (1+\gamma) \eta^2 L(\bw_1)+\frac{1}{\eta t}\big(1+n \eta^2 L(\bw_1)\big).
\end{multline}
From the condition on $\gamma$ and $\eta$, and noting that $L(\bw^*) \leq L(\bw_1)$, we have
\begin{equation}\label{eq:epr-cvx7}
    \frac{L(\bw^*)}{\gamma} \leq \frac{L(\bw_1)}{\gamma}\lesssim \sqrt{\frac{L(\bw_1)}{n}},\quad (1+\gamma) \eta^2 L(\bw_1) \lesssim \sqrt{n L(\bw_1)} \cdot\frac{1}{n L(\bw_1)}\cdot L(\bw_1)\lesssim \sqrt{\frac{L(\bw_1)}{n}}.
\end{equation}
Furthermore, we know from the condition $t\asymp n$ that
\begin{equation}\label{eq:epr-cvx8}
    \frac{(1+\gamma)(1+t / n)\eta^2 t L(\bw^*)}{n}\lesssim\frac{\sqrt{nL(\bw_1)}\cdot 1\cdot \frac{1}{nL(\bw_1)} \cdot n L(\bw_1)}{n} =\sqrt{\frac{L(\bw_1)}{n}}.
\end{equation}

To bound the remaining terms, note that
\[
1+n \eta^2 L(\bw_1)\lesssim 1+n\cdot\frac{1}{nL(\bw_1)}\cdot L(\bw_1)\lesssim 1.
\]
Therefore, we can derive
\begin{align}
&\frac{1}{\gamma\eta t}\big(1+n \eta^2 L(\bw_1)\big)\lesssim\frac{1}{\gamma\eta t} \lesssim \frac{1}{n}\lesssim\sqrt{\frac{L(\bw_1)}{n}},\label{eq:epr-cvx9}\\
&\frac{(1+\gamma)(1+t / n)\eta}{n}\big(1+n \eta^2 L(\bw_1)\big)\lesssim\frac{(1+\gamma)(1+t / n)\eta}{n}\lesssim\frac{1}{n}\lesssim\sqrt{\frac{L(\bw_1)}{n}},\label{eq:epr-cvx10}\\
&\frac{1}{\eta t}\big(1+n \eta^2 L(\bw_1)\big)\lesssim\frac{1}{\eta t}\lesssim\frac{1}{n/\sqrt{nL(\bw_1)}}=\sqrt{\frac{L(\bw_1)}{n}}.\label{eq:epr-cvx11}
\end{align}
Plugging Eq.~\eqref{eq:epr-cvx7}-Eq.~\eqref{eq:epr-cvx11} into Eq.~\eqref{eq:epr-cvx6} gives
\[
\ebb[L(\bar{\bw}_t)-L(\bw^*)]\lesssim \sqrt{L(\bw_1)/n},
\]
which completes the proof.
\end{proof}

\section{Analysis under Strong Convexity}\label{app:sc}
This section presents the strongly convex analysis for both SVRG and SAGA, including theorem statements and their proofs. In this setting, we assume the loss function $\ell(\cdot;\bz)$ is nonnegative, $\mu$-strongly convex, and $\alpha$-smooth for all $\bz\in\zcal$, which enables tighter stability and EPR bounds than the convex case.

\subsection{Analysis of SVRG}\label{app:svrg-sc}
In the strongly convex setting, we implement Algorithm~\ref{algo: svrg} with option~\uppercase\expandafter{\romannumeral 2}, which initializes each inner loop with the reference point to facilitate linear convergence of SVRG~\citep{johnson2013accelerating}.

We first present the stability bound of SVRG in the following theorem, which improves upon the convex case (Theorem~\ref{thm:svrg-stab-cvx}) by involving a potential contraction factor of $(c-1)^{-(t-l)}$ for the empirical error $\e[L_S(\bw_{l+1})]$. This happens because $\mu$-strong convexity implies an improved monotonicity property:
\[
\langle \bw-\bw', \nabla \ell(\bw; \bz)-\nabla \ell(\bw';\bz)\rangle \ge \mu\|\bw-\bw'\|^2,~\forall \bw,\bw'\in\wcal.
\]
This property helps establish a tighter estimate for $\e[\|G_k-G_k^{(i)}\|^2]$, which in turn reduces the overall stability bound. The detailed proof can be found in Appendix~\ref{sec:svrg-proof-stab-sc}.

\begin{theorem}[Stability bounds]\label{thm:svrg-stab-sc}
Let $S,S^{(i)}$ be defined in Definition~\ref{def:stability}. Let $\{\bw_t\}_{t\ge 1}$ and $\{\bw_t^{(i)}\}_{t\ge 1}$ be the sequences produced by Algorithm~\ref{algo: svrg} with option \uppercase\expandafter{\romannumeral 2} applied to $S$ and $S^{(i)}$, respectively.
Assume that for any $\bz\in\zcal$, the map $\bw\mapsto\ell(\bw;\bz)$ is nonnegative, $\mu$-strongly convex, and $\alpha$-smooth.
Denote $c:=m\eta\mu$.
If $c>2$ and $\eta\leq \frac{n-2}{2\alpha(1+c)(n-1)}$, then for any $t\ge 1$ we have
\begin{equation*}
    \e\big[\|\bw_{t+1} - \bw_{t+1}^{(i)}\|^2\big] \leq \frac{16\alpha m\eta^2}{(c-1)^t\,n}L(\bw_1)+ \frac{8\alpha m(4+mt/n)}{n}\eta^2\sum_{l=1}^t\frac{1}{(c-1)^{t-l}}\,\e\big[L_S(\bw_{l+1})\big].
\end{equation*}
\end{theorem}

\begin{remark}\label{rmk:svrg-stab-sc}
    Theorem~\ref{thm:svrg-stab-sc} shows that SVRG achieves optimal stability bounds of order $O(1/(\mu n))$ in the strongly convex case. To see this, consider $m=\lceil 3/(\eta\mu) \rceil$, $t\lesssim n/m$, and $\eta=\frac{1}{\mu n+18\alpha}$. Then it follows from \citep[Theorem 1]{johnson2013accelerating} that (see Eq.~\eqref{eq:svrg-sc-epr-opt} for detailed derivation):
    \[
    \e\big[L_S(\bw_{l+1})\big]\leq L(\bw^*)+\frac{L(\bw_1)}{2^l}.
    \]
    Plugging the above inequality into Theorem~\ref{thm:svrg-stab-sc} and using $c-1\ge 2$ and $L(\bw^*)\leq L(\bw_1)\lesssim 1$ gives
    \begin{align*}
        \e\big[\|\bw_{t+1}-\bw_{t+1}^{(i)}\|^2\big] \lesssim & \frac{m\eta^2}{2^t n}+ \frac{m(1+mt/n)}{n}\eta^2\sum_{l=1}^t\frac{1+1/2^l}{2^{t-l}}\\
        \lesssim & \frac{m\eta^2}{2^t n} + \frac{m \eta^2}{n}\Big(2+\frac{t}{2^t}\Big)
        \lesssim \frac{m \eta^2}{n} \asymp \frac{\eta}{\mu n} \leq \frac{1}{(\mu n)^2},
    \end{align*}
    where the second inequality follows from $\sum_{l=1}^t \frac{1}{2^{t-l}}=\sum_{k=0}^{t-1} \frac{1}{2^{k}}= 2(1-\frac{1}{2^t}) \leq 2$ and $m t\lesssim n$, and the third inequality is due to $2^t\ge 1$ and $t\leq 2^t$ for all $t\ge 1$. This matches the stability bounds of SGD in the strongly convex case~\citep{hardt2016train}, and we remove the Lipschitzness assumption there.
\end{remark}

Combining the stability bound established in Theorem~\ref{thm:svrg-stab-sc} with the known linear convergence rate of SVRG~\citep[Theorem 1]{johnson2013accelerating}, we obtain the following EPR bounds for strongly convex problems. The proofs of Theorem~\ref{thm:svrg-epr-sc} and Corollary~\ref{co:svrg-epr-sc} are provided in Appendix~\ref{sec:svrg-proof-exc-sc}.
\begin{theorem}[Excess population risk]\label{thm:svrg-epr-sc}
Assume that for any $\bz\in\zcal$, the map $\bw\mapsto\ell(\bw;\bz)$ is nonnegative, $\mu$-strongly convex, and $\alpha$-smooth.
    Let $\{\bw_t\}_{t\ge 1}$ be the sequence produced by Algorithm~\ref{algo: svrg} with option \uppercase\expandafter{\romannumeral 2} applied to $S$.
    Denote $c:=m\eta\mu$.
    If $c\ge 3$ and $\eta\leq \min\big\{\frac{1}{18\alpha}, \frac{n-2}{2\alpha(1+c)(n-1)}\big\}$, then for any $t\ge 1$ and $\gamma>0$, we have
\begin{multline*}
    \e\big[L(\bw_{t}) - L(\bw^*)\big] \lesssim \frac{1}{\gamma}\Big(L(\bw^*) + \frac{L(\bw_1)}{2^{t}}\Big)\\ + \frac{(1+\gamma) m(1+mt/n)}{n}\eta^2\Big(L(\bw^*) + \frac{tL(\bw_1)}{2^{t}}\Big)
    + \frac{(1+\gamma) m\eta^2 L(\bw_1)}{2^t n} + \frac{L(\bw_1)}{2^{t}}.
\end{multline*}
\end{theorem}

By specifying parameters $m, t, \eta$, and $\gamma$ in Theorem~\ref{thm:svrg-epr-sc}, we derive a near optimal EPR bound of order $O\big(\sqrt{\log_2(\mu n)}/(\mu n)\big)$ in the following corollary.
\begin{corollary}\label{co:svrg-epr-sc}
Let the assumptions in Theorem \ref{thm:svrg-epr-sc} hold and assume $\mu>1/n$. We can take $m = \lceil 3/(\eta\mu)\rceil$, $t \asymp \log_2 (\mu n)$, $\eta = \frac{1}{\mu n+ 18\alpha}$, and $\gamma = \frac{\mu n}{\sqrt{\log_2 (\mu n)}}$ to derive
\[
    \e\big[L(\bw_t)\big] - L(\bw^*) \lesssim \frac{\sqrt{\log_2 (\mu n)}}{\mu n}.
\]
\end{corollary}

\subsubsection{Proof of Stability Bounds\label{sec:svrg-proof-stab-sc}}
We now prove refined stability bounds for SVRG under strong convexity.
\begin{proof}[Proof of Theorem~\ref{thm:svrg-stab-sc}]
    Consider a fixed outer-loop index $t\ge 1$. We start by plugging Eq.~\eqref{eq:svrg-d-square} into Eq.~\eqref{eq:svrg-stab1} to obtain
\begin{multline}\label{eq:svrg-sc-stab1}
    \e[\|\bx_{k+1}-\bx_{k+1}^{(i)}\|^2]\leq \e[\|G_k-G_k^{(i)}\|^2]\\
    +\frac{2 \eta^2}{n} \sum_{j=1}^n \e\big[\big\|\nabla \ell(\bw_t; \bz_j) - \nabla \ell(\bw_t^{(i)}; \bz_j^{(i)})\big\|^2\big]+\eta^2\e[\|\nabla \ell(\bx_k; \bz_{i_k})-\nabla \ell(\bx_k^{(i)}; \bz_{i_k}^{(i)})\|^2].
\end{multline}
To bound the term $\e\big[\|G_k-G_k^{(i)}\|^2\big]$, first consider the case $i_k\neq i$ (with probability $1-1/n$), which implies $\bz_{i_k}^{(i)}=\bz_{i_k}$. From the definitions of $G_k$ and $G_k^{(i)}$, we have
\begin{align}
  \|G_k-G_k^{(i)}\|^2 = & \|\bx_k-\bx_k^{(i)}\|^2+\eta^2\|\nabla \ell(\bx_k; \bz_{i_k})-\nabla \ell(\bx_k^{(i)};\bz_{i_k})\|^2\notag\\
  &-2 \eta\langle \bx_k-\bx_k^{(i)}, \nabla \ell(\bx_k; \bz_{i_k})-\nabla \ell(\bx_k^{(i)};\bz_{i_k})\rangle.\label{eq:sc-G-G}
\end{align}
By the coercivity of $\ell$, we have
\begin{equation}\label{eq:sc-co}
    \langle \bx_k-\bx_k^{(i)}, \nabla \ell(\bx_k; \bz_{i_k})-\nabla \ell(\bx_k^{(i)};\bz_{i_k})\rangle \ge \frac{1}{\alpha}\|\nabla \ell(\bx_k; \bz_{i_k})-\nabla \ell(\bx_k^{(i)}; \bz_{i_k})\|^2.
\end{equation}
Moreover, we know from $\mu$-strong convexity of $\ell$ that
\begin{equation}\label{eq:sc-mu}
    \langle \bx_k-\bx_k^{(i)}, \nabla \ell(\bx_k; \bz_{i_k})-\nabla \ell(\bx_k^{(i)};\bz_{i_k})\rangle \ge \mu\|\bx_k-\bx_k^{(i)}\|^2.
\end{equation}
Combining Eq.~\eqref{eq:sc-co} and Eq.~\eqref{eq:sc-mu} gives
\begin{align*}
    &2 \eta\langle \bx_k-\bx_k^{(i)}, \nabla \ell(\bx_k; \bz_{i_k})-\nabla \ell(\bx_k^{(i)};\bz_{i_k})\rangle\\
    = &\Big(\frac{n-2}{n-1}+\frac{n}{n-1}\Big)\eta \langle \bx_k-\bx_k^{(i)}, \nabla \ell(\bx_k; \bz_{i_k})-\nabla \ell(\bx_k^{(i)};\bz_{i_k})\rangle\\
    \ge & \frac{(n-2)\eta}{(n-1)\alpha}\|\nabla \ell(\bx_k; \bz_{i_k})-\nabla \ell(\bx_k^{(i)}; \bz_{i_k})\|^2+\frac{n}{n-1}\eta\mu\|\bx_k-\bx_k^{(i)}\|^2.
\end{align*}
Hence it follows from Eq.~\eqref{eq:sc-G-G} that
\begin{equation}\label{eq:sc-G-G-1}
    \|G_k-G_k^{(i)}\|^2\leq \Big(1-\frac{n}{n-1}\eta\mu\Big)\|\bx_k-\bx_k^{(i)}\|^2-\Big(\frac{(n-2)\eta}{(n-1)\alpha}-\eta^2\Big)\|\nabla \ell(\bx_k; \bz_{i_k})-\nabla \ell(\bx_k^{(i)};\bz_{i_k})\|^2.
\end{equation}
Otherwise if $i_k=i$ (with probability $1/n$), then $\bz_{i_k}=\bz_i$ and $\bz_{i_k}^{(i)}=\bz_i^{\prime}$. By Young's inequality, for any $p>0$, we have that
\begin{align}
    \|G_k-G_k^{(i)}\|^2&=\|\bx_k-\bx_k^{(i)}-\eta\big(\nabla \ell(\bx_k; \bz_i)-\nabla \ell(\bx_k^{(i)}; \bz_i^{\prime})\big)\|^2\notag\\
    &\leq (1+p)\|\bx_k-\bx_k^{(i)}\|^2+(1+1/p)\eta^2\|\nabla \ell(\bx_k; \bz_i)-\nabla \ell(\bx_k^{(i)}; \bz_i^{\prime})\|^2.\label{eq:sc-G-G-2}
\end{align}
Combining the two cases Eq.~\eqref{eq:sc-G-G-1} and Eq.~\eqref{eq:sc-G-G-2} leads to
\begin{multline}\label{eq:svrg-sc-G}
\e\big[\|G_k - G_k^{(i)}\|^2\big]\leq \Big(1+\frac{p}{n}-\eta\mu\Big) \e\big[\|\bx_k - \bx_k^{(i)}\|^2\big]\\
- \frac{1}{n}\Big(\frac{(n-2)\eta}{(n-1)\alpha} - \eta^2\Big) \sum_{j \neq i} \e\big[\big\|\nabla \ell(\bx_k; \bz_j) - \nabla \ell(\bx_k^{(i)}; \bz_j)\big\|^2\big] + \frac{1+1/p}{n} \eta^2 \e\big[\big\|\nabla \ell(\bx_k; \bz_i) - \nabla \ell(\bx_k^{(i)}; \bz_i^{\prime})\big\|^2\big].
\end{multline}
Substituting Eq.~\eqref{eq:svrg-sc-G} into Eq.~\eqref{eq:svrg-sc-stab1} gives
\begin{align}\label{eq:svrg-sc-stab2}
    &\e[\|\bx_{k+1}-\bx_{k+1}^{(i)}\|^2]\notag\\
    \leq & \Big(1+\frac{p}{n}-\eta\mu\Big) \e\big[\|\bx_k - \bx_k^{(i)}\|^2\big]- \frac{1}{n}\Big(\frac{(n-2)\eta}{(n-1)\alpha} - 2\eta^2\Big) \sum_{j \neq i} \e\big[\big\|\nabla \ell(\bx_k; \bz_j) - \nabla \ell(\bx_k^{(i)}; \bz_j)\big\|^2\big]\notag\\
    &+ \frac{2+1/p}{n} \eta^2 \e\big[\big\|\nabla \ell(\bx_k; \bz_i) - \nabla \ell(\bx_k^{(i)}; \bz_i^{\prime})\big\|^2\big]+\frac{2 \eta^2}{n} \sum_{j=1}^n \e\big[\big\|\nabla \ell(\bw_t; \bz_j) - \nabla \ell(\bw_t^{(i)}; \bz_j^{(i)})\big\|^2\big].
\end{align}

Next we analyze the outer loop. Recall notations $\bx_k:=\bx_k^{t+1}$ and $\bx_k^{(i)}:=\widetilde{\bx}_k^{t+1}$. Since $\bw_{t+1}$ (resp., $\bw_{t+1}^{(i)}$) is uniformly chosen from $\{\bx_k^{t+1}\}_{k=0}^{m-1}$ (resp., $\{\widetilde\bx_k^{t+1}\}_{k=0}^{m-1}$), and $\bx_0^{t+1}=\bw_t$ (resp., $\widetilde\bx_0^{t+1} = \bw_t^{(i)}$) from option \uppercase\expandafter{\romannumeral 2}, we have
\begin{align}
    \e\big[\|\bw_{t+1} - \bw_{t+1}^{(i)}\|^2\big] = & \frac{1}{m}\sum_{k=0}^{m-1}\e\big[\|\bx_k^{t+1} - \widetilde\bx_k^{t+1}\|^2\big]\leq \frac{1}{m}\sum_{k=0}^{m}\e\big[\|\bx_k^{t+1} - \widetilde\bx_k^{t+1}\|^2\big]\notag\\
    = & \frac{1}{m}\e\big[\|\bw_t - \bw_t^{(i)}\|^2\big]+\frac{1}{m}\sum_{k=0}^{m-1}\e\big[\|\bx_{k+1}^{t+1} - \widetilde\bx_{k+1}^{t+1}\|^2\big].\label{eq:svrg-sc-avg}
\end{align}
Applying Eq.~\eqref{eq:svrg-sc-stab2} once for each $k$ in the summation term on the RHS of Eq.~\eqref{eq:svrg-sc-avg} gives
\begin{align}
    \e\big[\|\bw_{t+1} - \bw_{t+1}^{(i)}\|^2\big] \leq & \frac{1}{m}\e\big[\|\bw_t - \bw_t^{(i)}\|^2\big] + \frac{1+p/n-\eta\mu}{m}\sum_{k=0}^{m-1}\e\big[\|\bx_k^{t+1} - \widetilde{\bx}_k^{t+1}\|^2\big] \notag\\
    & - \frac{1}{mn}\Big(\frac{(n-2)\eta}{(n-1)\alpha} - 2\eta^2\Big)\sum_{k=0}^{m-1}\sum_{j\neq i}\e\big[\big\|\nabla\ell(\bx_k^{t+1};\bz_j) - \nabla\ell(\widetilde\bx_k^{t+1};\bz_j)\big\|^2\big] \notag\\
    & + \frac{2+1/p}{mn}\eta^2\sum_{k=0}^{m-1}\e\big[\big\|\nabla\ell(\bx_k^{t+1};\bz_i) - \nabla\ell(\widetilde\bx_k^{t+1};\bz_i')\big\|^2\big] \notag\\
    & + \frac{2\eta^2}{n}\sum_{j=1}^n\e\big[\big\|\nabla\ell(\bw_t;\bz_j) - \nabla\ell(\bw_t^{(i)};\bz_j^{(i)})\big\|^2\big].\label{eq:svrg-sc-avg2}
\end{align}
For notational simplicity, we denote
\[
\rho := \frac{1}{m(\eta\mu-p/n)}.
\]
Setting $p = \frac{n}{mt}$, we know from $c=m\eta\mu > 2$ and $t\ge 1$ that
\begin{equation}\label{eq:svrg-sc-rho}
    0<\frac{1}{c}<\rho = \frac{1}{c-1/t} \leq \frac{1}{c-1} < 1.
\end{equation}
For the second term on the RHS of Eq.~\eqref{eq:svrg-sc-avg2}, we have
\begin{equation}\label{eq:svrg-sc-avg2-term}
    \frac{1+p/n-\eta\mu}{m}\sum_{k=0}^{m-1}\e\big[\|\bx_k^{t+1} - \widetilde{\bx}_k^{t+1}\|^2\big] = \Big(1-\frac{1}{\rho m}\Big)\e\big[\|\bw_{t+1} - \bw_{t+1}^{(i)}\|^2\big].
\end{equation}
Plugging Eq.~\eqref{eq:svrg-sc-avg2-term} into Eq.~\eqref{eq:svrg-sc-avg2} and rearranging gives
\begin{align}
    \e\big[\|\bw_{t+1} - \bw_{t+1}^{(i)}\|^2\big]
    \leq & \rho\,\e\big[\|\bw_t - \bw_t^{(i)}\|^2\big] \notag\\
    & - \frac{\rho}{n}\Big(\frac{(n-2)\eta}{(n-1)\alpha} - 2\eta^2\Big)\sum_{k=0}^{m-1}\sum_{j\neq i}\e\big[\big\|\nabla\ell(\bx_k^{t+1};\bz_j) - \nabla\ell(\widetilde\bx_k^{t+1};\bz_j)\big\|^2\big] \notag\\
    & + \frac{(2+mt/n)\rho}{n}\eta^2\sum_{k=0}^{m-1}\e\big[\big\|\nabla\ell(\bx_k^{t+1};\bz_i) - \nabla\ell(\widetilde\bx_k^{t+1};\bz_i')\big\|^2\big] \notag\\
    & + \frac{2\rho m\eta^2}{n}\sum_{j=1}^n\e\big[\big\|\nabla\ell(\bw_t;\bz_j) - \nabla\ell(\bw_t^{(i)};\bz_j^{(i)})\big\|^2\big].\label{eq:svrg-sc-outer}
\end{align}

To proceed, define the Lyapunov function $Q_t$ via
\[
    Q_t := \|\bw_t - \bw_t^{(i)}\|^2 + \frac{2m\eta^2}{n}\sum_{j=1}^n \big\|\nabla\ell(\bw_t;\bz_j) - \nabla\ell(\bw_t^{(i)};\bz_j^{(i)})\big\|^2.
\]
Applying Eq.~\eqref{eq:svrg-sc-outer} + Eq.~\eqref{eq:svrg-stab-rule}$\times 2m\eta^2$, we obtain
\begin{align}\label{eq:svrg-sc-P-recur}
    \e\big[Q_{t+1}\big] \leq & \rho\,\e\big[Q_t\big] - \left[\frac{\rho}{n}\Big(\frac{(n-2)\eta}{(n-1)\alpha} - 2\eta^2\Big) - \frac{2\eta^2}{n}\right]\sum_{k=0}^{m-1}\sum_{j\neq i}\e\big[\big\|\nabla\ell(\bx_k^{t+1};\bz_j) - \nabla\ell(\widetilde\bx_k^{t+1};\bz_j)\big\|^2\big] \notag\\
    & + \left(\frac{(2+mt/n)\rho}{n}\eta^2 + \frac{2\eta^2}{n}\right)\sum_{k=0}^{m-1}\e\big[\big\|\nabla\ell(\bx_k^{t+1};\bz_i) - \nabla\ell(\widetilde\bx_k^{t+1};\bz_i')\big\|^2\big].
\end{align}
Since $\eta\leq \frac{n-2}{2\alpha(1+c)(n-1)}$, which together with Eq.~\eqref{eq:svrg-sc-rho} implies that
\[
    \frac{\rho}{n}\Big(\frac{(n-2)\eta}{(n-1)\alpha} - 2\eta^2\Big) > \frac{1}{c n}\Big(\frac{(n-2)\eta}{(n-1)\alpha} - 2\eta^2\Big) \ge\frac{1}{cn}\big(2(1+c)\eta^2-2\eta^2\big)= \frac{2\eta^2}{n}.
\]
Then, dropping the nonnegative term in Eq.~\eqref{eq:svrg-sc-P-recur} leads to
\begin{align}
    \e\big[Q_{t+1}\big] &\leq \rho\,\e\big[Q_t\big] + \left(\frac{(2+mt/n)\rho}{n}\eta^2 + \frac{2\eta^2}{n}\right)\sum_{k=0}^{m-1}\e\big[\big\|\nabla\ell(\bx_k^{t+1};\bz_i) - \nabla\ell(\widetilde\bx_k^{t+1};\bz_i')\big\|^2\big]\notag\\
    &\leq \rho\,\e\big[Q_t\big] + \frac{8\alpha(4+mt/n)}{n}\eta^2\sum_{k=0}^{m-1}\e\big[L_S(\bx_k^{t+1})\big] \notag\\
    &= \rho\,\e\big[Q_t\big] + \frac{8\alpha m(4+mt/n)}{n}\eta^2\,\e\big[L_S(\bw_{t+1})\big],\label{eq:svrg-sc-P-simple}
\end{align}
where we used $\rho<1$ and Eq.~\eqref{eq:svrg-8alpha} to derive the second inequality and the fact that $\bw_{t+1}$ is uniformly chosen from $\{\bx_k^{t+1}\}_{k=0}^{m-1}$ for the equality.
Applying Eq.~\eqref{eq:svrg-sc-P-simple} recursively yields
\begin{equation}\label{eq:svrg-sc-P-recursive}
    \e\big[Q_{t+1}\big] \leq \rho^t\,\e\big[Q_1\big] + \frac{8\alpha m(4+mt/n)}{n}\eta^2\sum_{l=1}^t\rho^{t-l}\,\e\big[L_S(\bw_{l+1})\big].
\end{equation}
From $\bw_1 = \bw_1^{(i)}$, we know that
\begin{align}\label{eq:svrg-sc-P1}
    \e\big[Q_1\big] = & \frac{2m\eta^2}{n}\sum_{j=1}^n\e\big[\big\|\nabla\ell(\bw_1;\bz_j) - \nabla\ell(\bw_1;\bz_j^{(i)})\big\|^2\big] \notag\\
    = & \frac{2m\eta^2}{n}\e\big[\|\nabla\ell(\bw_1;\bz_i) - \nabla\ell(\bw_1;\bz_i')\|^2\big] \leq \frac{16\alpha m\eta^2}{n}L(\bw_1).
\end{align}
Plugging Eq.~\eqref{eq:svrg-sc-P1} into Eq.~\eqref{eq:svrg-sc-P-recursive}, together with Eq.~\eqref{eq:svrg-sc-rho}, we conclude
\begin{align*}
    &\e\big[\|\bw_{t+1} - \bw_{t+1}^{(i)}\|^2\big] \leq \e\big[Q_{t+1}\big] \\
    \leq & \frac{16\alpha m\eta^2}{n}\Big(\frac{1}{c-1}\Big)^t L(\bw_1) + \frac{8\alpha m(4+mt/n)}{n}\eta^2\sum_{l=1}^t\Big(\frac{1}{c-1}\Big)^{t-l}\,\e\big[L_S(\bw_{l+1})\big],
\end{align*}
which completes the proof.
\end{proof}

\subsubsection{Proof of EPR Bounds\label{sec:svrg-proof-exc-sc}}
We first recall the following convergence result from \citep[Theorem 1]{johnson2013accelerating}. It shows that SVRG converges linearly for strongly convex problems.
\begin{theorem}[Optimization error~\citep{johnson2013accelerating}]
    \label{thm:svrg-opt-sc}
    Assume for any $\bz\in\zcal$, the map $\bw\mapsto\ell(\bw;\bz)$ is nonnegative, $\mu$-strongly convex, and $\alpha$-smooth.
    Let $\{\bw_t\}_{t\ge 1}$ be produced by Algorithm~\ref{algo: svrg} with option \uppercase\expandafter{\romannumeral 2} applied to $S$. Denote $\bw_S:=\argmin_{\bw\in\wcal} L_S(\bw)$ and $c:=m\eta\mu$. If $\eta<\frac{1}{2\alpha}$ and $c$ is large enough such that
    \[
    \rho:=\frac{1}{c(1-2 \alpha\eta)}+\frac{2 \alpha \eta}{1-2 \alpha \eta}<1,
    \]
    then for any $t\ge 1$ we have
    \[
    \e_A\big[L_S(\bw_t) - L_S(\bw_S)\big] \leq \rho^{t-1} (L_S(\bw_1) - L_S(\bw_S)).
    \]
\end{theorem}
Combining the stability bound from Theorem~\ref{thm:svrg-stab-sc} and the convergence result from Theorem~\ref{thm:svrg-opt-sc}, we are ready to prove the EPR bound for SVRG under strong convexity.
\begin{proof}[Proof of Theorem~\ref{thm:svrg-epr-sc}]
Since $c\ge 3$, we have $\frac{1}{c-1}\leq \frac{1}{2}$. Then it follows from Theorem~\ref{thm:svrg-stab-sc} that
\begin{align}
    \e\big[\|\bw_{t} - \bw_{t}^{(i)}\|^2\big] \leq & \frac{16\alpha m\eta^2}{2^{t-1} n} L(\bw_1) + \frac{8\alpha m(4+mt/n)}{n}\eta^2\sum_{l=1}^{t-1} \frac{1}{2^{t-1-l}}\,\e\big[L_S(\bw_{l+1})\big]\notag\\
    \leq & \frac{16\alpha m\eta^2}{2^{t-1} n} L(\bw_1) + \frac{8\alpha m(4+mt/n)}{n}\eta^2\sum_{l=1}^{t} \frac{1}{2^{t-l}}\,\e\big[L_S(\bw_{l})\big].\label{eq:svrg-sc-epr1}
\end{align}
Applying Lemma~\ref{lem:general-stablity} together with Eq.~\eqref{eq:svrg-sc-epr1}, we obtain
\begin{multline}\label{eq:svrg-sc-epr2}
    \e\big[L(\bw_{t}) - L_S(\bw_{t})\big]
    \leq \frac{\alpha}{\gamma}\e\big[L_S(\bw_{t})\big]
    + \frac{\alpha+\gamma}{2n}\sum_{i=1}^n\e\big[\|\bw_{t}-\bw_{t}^{(i)}\|^2\big] \\
    \leq \frac{\alpha}{\gamma}\e\big[L_S(\bw_{t})\big]
    + \frac{4(\alpha+\gamma)\alpha m(4+mt/n)}{n}\eta^2\sum_{l=1}^t \frac{1}{2^{t-l}}\,\e\big[L_S(\bw_{l})\big]
    + \frac{16(\alpha+\gamma)\alpha m\eta^2}{2^t n} L(\bw_1).
\end{multline}
We next bound the terms involving $\e[L_S(\bw_l)]$. Since $c\ge 3$ and $\eta\leq \frac{1}{18\alpha}$, we know that
\[
\frac{1}{c(1-2 \alpha\eta)}+\frac{2 \alpha \eta}{1-2 \alpha \eta}\leq \frac{1}{3\cdot8/9} + \frac{1/9}{8/9}=\frac{1}{2}.
\]
Then it follows from Theorem~\ref{thm:svrg-opt-sc} and $0\leq \e[L_S(\bw_S)]\leq L(\bw^*)$ that
\begin{equation}\label{eq:svrg-sc-epr-opt}
    \e\big[L_S(\bw_l)\big] - L(\bw^*) \leq \e\big[L_S(\bw_l) - L_S(\bw_S)\big] \leq \frac{L(\bw_1)}{2^{l-1}},~\forall l\ge 1,
\end{equation}
which together with $\sum_{l=1}^t \frac{1}{2^{t-l}}=\sum_{k=0}^{t-1} \frac{1}{2^{k}}= 2(1-\frac{1}{2^t}) \leq 2$ further implies
\begin{equation}\label{eq:svrg-sc-epr-sum}
    \sum_{l=1}^{t} \frac{1}{2^{t-l}}\,\e\big[L_S(\bw_{l})\big]
    \leq L(\bw^*)\sum_{l=1}^t \frac{1}{2^{t-l}} + L(\bw_1)\sum_{l=1}^t \frac{1}{2^{t-l}}\cdot \frac{1}{2^{l-1}}
    \leq 2 L(\bw^*) + \frac{tL(\bw_1)}{2^{t-1}}.
\end{equation}
Plugging Eq.~\eqref{eq:svrg-sc-epr-opt} with $l=t$ and Eq.~\eqref{eq:svrg-sc-epr-sum} into Eq.~\eqref{eq:svrg-sc-epr2} and omitting constant factors yields
\begin{align}
    &\e\big[L(\bw_{t}) - L_S(\bw_{t})\big] \notag\\
    \lesssim & \frac{1}{\gamma}\Big(L(\bw^*) + \frac{L(\bw_1)}{2^{t}}\Big)
    + \frac{(1+\gamma) m(1+mt/n)}{n}\eta^2\Big(L(\bw^*) + \frac{tL(\bw_1)}{2^{t}}\Big)
    + \frac{(1+\gamma) m\eta^2 L(\bw_1)}{2^t n}.\label{eq:svrg-sc-epr-gen}
\end{align}
Combining the generalization error Eq.~\eqref{eq:svrg-sc-epr-gen} and the optimization error Eq.~\eqref{eq:svrg-sc-epr-opt} with $l=t$ gives
\begin{multline*}
    \e\big[L(\bw_{t}) - L(\bw^*)\big] \lesssim \frac{1}{\gamma}\Big(L(\bw^*) + \frac{L(\bw_1)}{2^{t}}\Big)\\ + \frac{(1+\gamma) m(1+mt/n)}{n}\eta^2\Big(L(\bw^*) + \frac{tL(\bw_1)}{2^{t}}\Big)
    + \frac{(1+\gamma) m\eta^2 L(\bw_1)}{2^t n} + \frac{L(\bw_1)}{2^{t}},
\end{multline*}
which completes the proof.
\end{proof}

\begin{proof}[Proof of Corollary~\ref{co:svrg-epr-sc}]
We first verify the conditions required in Theorem \ref{thm:svrg-epr-sc}. It follows from $m \ge 3/(\eta\mu)$ that $c=m\eta\mu\ge3$. Additionally, from the choice of $\eta$ and $\alpha\ge \mu$, we have
\[
m\leq 3/(\eta\mu)+1 = 3(n+18\alpha/\mu)+1<3.5(n+18\alpha/\mu)=3.5/(\eta\mu),
\]
which further gives $c=m\eta\mu<3.5$. This implies $\frac{n-2}{2\alpha(1+c)(n-1)}> \frac{1}{18\alpha}$ considering $n\ge 3$, and hence we have
\[
\eta = \frac{1}{\mu n + 18\alpha} < \frac{1}{18\alpha}= \min\Big\{\frac{1}{18\alpha}, \frac{n-2}{2\alpha(1+c)(n-1)}\Big\}.
\]
Then we can apply Theorem \ref{thm:svrg-epr-sc} and use $L(\bw^*)\leq L(\bw_1)\lesssim 1$ to obtain
\begin{equation}\label{eq:svrg-epr-sc-co1}
    \e\big[L(\bw_{t}) - L(\bw^*)\big] \lesssim \frac{1}{\gamma}\Big(1 + \frac{1}{2^{t}}\Big) + \frac{(1+\gamma) m(1+mt/n)}{n}\eta^2\Big(1 + \frac{t}{2^{t}}\Big)
    + \frac{(1+\gamma) m\eta^2}{2^t n} + \frac{1}{2^{t}}.
\end{equation}
To proceed, note from $\mu n>1$ and the choices of $m, \eta$, and $t$ that
\begin{equation}\label{eq:svrg-epr-sc-mtn}
    \frac{mt}{n} \asymp \frac{t}{\eta\mu n} \asymp \frac{\mu n+18\alpha}{\mu n} \log_2(\mu n) < (1+18\alpha) \log_2(\mu n) \asymp \log_2(\mu n).
\end{equation}
Applying Eq.~\eqref{eq:svrg-epr-sc-mtn} and the basic inequality $t\leq 2^t$, we derive
\begin{equation}\label{eq:svrg-epr-sc-term2}
    \frac{(1+\gamma) m(1+mt/n)}{n}\eta^2\Big(1+\frac{t}{2^t}\Big) \lesssim \frac{\frac{\mu n}{\sqrt{\log_2(\mu n)}}\cdot \frac{1}{\eta\mu}\cdot\log_2(\mu n)}{n}\eta^2 = \eta\sqrt{\log_2(\mu n)}\leq \frac{\sqrt{\log_2(\mu n)}}{\mu n}.
\end{equation}
Since $1\leq t \asymp \log_2(\mu n)$, we have $\frac{1}{\mu n} \asymp \frac{1}{2^t} \lesssim 1$. Therefore, we can bound the remaining terms on the RHS of Eq.~\eqref{eq:svrg-epr-sc-co1} as follows:
\begin{align}
    &\frac{1}{\gamma}\Big(1+\frac{1}{2^t}\Big) \lesssim \frac{1}{\gamma} = \frac{\sqrt{\log_2(\mu n)}}{\mu n},\label{eq:svrg-epr-sc-term1}\\
    &\frac{(1+\gamma) m\eta^2}{2^t n} \lesssim \frac{\frac{\mu n}{\sqrt{\log_2(\mu n)}}\cdot \frac{1}{\eta\mu}\cdot\eta^2}{n} = \frac{\eta}{\sqrt{\log_2(\mu n)}} \lesssim \frac{\sqrt{\log_2(\mu n)}}{\mu n},\label{eq:svrg-epr-sc-term3}\\
    &\frac{1}{2^t} \asymp \frac{1}{\mu n} \lesssim \frac{\sqrt{\log_2(\mu n)}}{\mu n}.\label{eq:svrg-epr-sc-term4}
\end{align}
Plugging Eqs.~\eqref{eq:svrg-epr-sc-term2}--\eqref{eq:svrg-epr-sc-term4} into Eq.~\eqref{eq:svrg-epr-sc-co1} gives
\[
\e\big[L(\bw_t)\big] - L(\bw^*) \lesssim \frac{\sqrt{\log_2(\mu n)}}{\mu n},
\]
which completes the proof.
\end{proof}

\subsection{Analysis of SAGA}\label{app:saga-sc}
We proceed to study the behavior of SAGA for strongly convex problems. We first establish a refined stability bound for SAGA in the following theorem. Compared with the result in Theorem \ref{thm:stab-cvx}, we have a potential contraction factor $(1+1 / t-\eta\mu)^{t-k}$ for the empirical error $\e[L_S(\bw_k)]$. The detailed proof is given in Appendix~\ref{sec:proof-stab-sc}.

\begin{theorem}[Stability bounds]\label{thm:stab-sc}
Let $S,S^{(i)}$ be defined in Definition~\ref{def:stability}. Let $\{\bw_t\}_{t\ge 1}$ and $\{\bw_t^{(i)}\}_{t\ge 1}$ be the sequences produced by Algorithm~\ref{algo: saga} applied to $S$ and $S^{(i)}$, respectively.
Assume that for any $\bz\in\zcal$, the map $\bw\mapsto\ell(\bw;\bz)$ is nonnegative, $\mu$-strongly convex and $\alpha$-smooth.
If $\eta \leq \min\big\{\frac{1}{2 \mu n}, \frac{n-2}{6\alpha(n-1)}\big\}$, then for any $t\ge 1$, we have
\[
\e[\|\bw_{t+1}-\bw_{t+1}^{(i)}\|^2] \leq \sum_{k=1}^t(1+1 / t-\eta\mu)^{t-k} \frac{8 \alpha(6+t / n)}{n} \eta^2 \e[L_S(\bw_k)]+32\alpha\eta^2(1+1 / t-\eta\mu)^t L(\bw_1).
\]
\end{theorem}

\begin{remark}\label{rmk:saga-stab-sc}
    Similar to SVRG (see Remark~\ref{rmk:svrg-stab-sc}), SAGA also achieves optimal stability bounds of order $O(1/(\mu n))$ in the strongly convex case, as indicated by Theorem~\ref{thm:stab-sc}. To see this, we know from the analysis in the proof of \citep[Theorem 1]{defazio2014saga} that if we set
    \[
    \eta\leq\frac{1}{2\mu n+12\alpha}\leq \min\Big\{\frac{1}{2\mu n}, \frac{n-2}{6\alpha(n-1)}\Big\},
    \]
    then $\e[\|\bw_k-\bw_S\|^2]\lesssim(1-\eta\mu)^k$. Hence, from the $\alpha$-smoothness of $L_S$ and optimality condition $\nabla L_S(\bw_S)=0$, we further obtain
    \[
    \e[L_S(\bw_k)]-L(\bw^*)\leq \e[L_S(\bw_k)-L_S(\bw_S)]\leq\frac{\alpha}{2}\e[\|\bw_k-\bw_S\|^2]\lesssim (1-\eta\mu)^k.
    \]
    Plugging the above inequality into Theorem~\ref{thm:stab-sc} and choosing $t\asymp n$ such that $1 / t\leq \eta\mu/2$ gives
    \begin{align}
        &\e[\|\bw_{t+1}-\bw_{t+1}^{(i)}\|^2] \notag\\
        \lesssim & \sum_{k=1}^t(1+1 / t-\eta\mu)^{t-k} \frac{1+t/n}{n} \eta^2 \big[L(\bw^*)+(1-\eta\mu)^k\big]+\eta^2(1+1 / t-\eta\mu)^t L(\bw_1)\notag\\
        \leq & \frac{(1+t/n)\eta^2}{n}\Big(L(\bw^*)\sum_{k=1}^t(1+1 / t-\eta\mu)^{t-k}+t(1+1 / t-\eta\mu)^t\Big)+\eta^2(1+1 / t-\eta\mu)^t L(\bw_1)\notag\\
        \leq & \frac{(1+t/n)\eta^2}{n}\Big(\frac{L(\bw^*)}{\eta\mu-1/t}+t\Big)+\eta^2 L(\bw_1)
        \leq  \frac{(1+t/n)\eta^2}{n}\Big(\frac{2 L(\bw^*)}{\eta\mu}+t\Big)+\eta^2 L(\bw_1) \lesssim \frac{1}{(\mu n)^2}.\notag%
    \end{align}
\end{remark}

Combining the stability bound established in Theorem \ref{thm:stab-sc} and the existing convergence analysis of SAGA in the strongly convex case \citep{defazio2014saga}, we obtain the following EPR bound. The proofs of Theorem~\ref{thm:epr-sc} and Corollary~\ref{co:epr-sc} are given in Appendix~\ref{sec:proof-exc-sc}.
\begin{theorem}[Excess population risk\label{thm:epr-sc}]
    Assume that for any $\bz\in\zcal$, the map $\bw\mapsto\ell(\bw;\bz)$ is nonnegative, $\mu$-strongly convex, and $\alpha$-smooth. Suppose $\e_S[\|\bw_1-\bw_S\|^2]<\infty$ and $n^{-1}\lesssim\mu$.
    Let $\{\bw_t\}_{t\ge 1}$ be the sequence produced by Algorithm~\ref{algo: saga} applied to $S$ with $\eta\leq\frac{1}{2\mu n+12\alpha}$.
    Then for any $t> 1/(\eta\mu)\asymp n$ and $\gamma>0$, we have
    \begin{multline*}
        \e[L(\bw_t)]-L(\bw^*) \lesssim \frac{1}{\gamma}\Big(L(\bw^*)+(1-\eta\mu)^t/\mu\Big)\\+  \frac{(1+\gamma)(1+t / n)}{n} \eta^2\Big(\frac{L(\bw^*)}{\eta\mu-1/t}+t(1+1 / t-\eta\mu)^t/\mu\Big)
        + (1+\gamma) \eta^2 L(\bw_1) + (1-\eta\mu)^t/\mu.
    \end{multline*}
\end{theorem}

We then specify the parameters to derive EPR bounds of order $O\big(\log^{2/3} n/(\mu n)\big)$, which matches the minimax optimal rate $O\big(1/(\mu n)\big)$ up to a logarithmic factor.
\begin{corollary}\label{co:epr-sc}
Let the assumptions in Theorem \ref{thm:epr-sc} hold and suppose $n^{-1}(\log n)^{2/3}\lesssim \mu$. Then we can take $t\asymp n \log n$, $\eta = \frac{1}{2\mu n + 12\alpha}$, and $\gamma = \mu n/(\log n)^{2/3}$ to derive
$
\ebb[L(\bw_t)-L(\bw^*)]\lesssim \frac{\log^{2/3} n}{\mu n}.
$
\end{corollary}

\subsubsection{Proof of Stability Bounds\label{sec:proof-stab-sc}}
\begin{proof}[Proof of Theorem~\ref{thm:stab-sc}]
    We start by plugging Eq.~\eqref{eq:cvx-stab-dt} into Eq.~\eqref{eq:cvx-stab2} to derive
\begin{multline}
    \e[\|\bw_{t+1}-\bw_{t+1}^{(i)}\|^2]
    \leq  \e[\|G_t-G_t^{(i)}\|^2]+\\
    \frac{2 \eta^2}{n} \sum_{j=1}^n\e[\|\nabla \ell(\phi_{t, j}; \bz_j)-\nabla \ell(\phi_{t, j}^{(i)}; \bz_j^{(i)})\|^2]+\eta^2\e[\|\nabla \ell(\bw_t; \bz_{i_t})-\nabla \ell(\bw_t^{(i)}; \bz_{i_t}^{(i)})\|^2]\label{eq:sc-stab-0}.
\end{multline}
By similar arguments as in the derivation of Eq.~\eqref{eq:svrg-sc-G}, we obtain
\begin{multline}\label{eq:sc-G-G-3}
    \e[\|G_t-G_t^{(i)}\|^2] \leq (1+p / n-\eta\mu) \e[\|\bw_t-\bw_t^{(i)}\|^2]\\-\frac{1}{n}\Big(\frac{(n\!-\!2)\eta}{(n\!-\!1)\alpha}\!-\!\eta^2\Big) \sum_{j \neq i} \e[\|\nabla \ell(\bw_t; \bz_j)\!-\!\nabla \ell(\bw_t^{(i)}; \bz_j)\|^2] \!+\!\frac{1+1/p}{n} \eta^2 \e[\|\nabla \ell(\bw_t; \bz_i)\!-\!\nabla \ell(\bw_t^{(i)}; \bz_i^{\prime})\|^2].
\end{multline}
Plugging Eq.~\eqref{eq:sc-G-G-3} back to Eq.~\eqref{eq:sc-stab-0}, we obtain
\begin{align}
    &\e[\|\bw_{t+1}-\bw_{t+1}^{(i)}\|^2]\notag\\
    \leq &\big(1\!+\!\frac{p}{n}\!-\!\eta\mu\big) \e[\|\bw_t\!-\!\bw_t^{(i)}\|^2]-\frac{1}{n}\Big(\frac{(n\!-\!2)\eta}{(n\!-\!1)\alpha}\!-\!2\eta^2\Big) \sum_{j \neq i} \e[\|\nabla \ell(\bw_t; \bz_j)\!-\!\nabla \ell(\bw_t^{(i)}; \bz_j)\|^2]\notag \\
    &\!+\!\frac{2\!+\!1 / p}{n} \eta^2 \e[\|\nabla \ell(\bw_t; \bz_i)\!-\!\nabla \ell(\bw_t^{(i)}; \bz_i^{\prime})\|^2]\!+\!\frac{2 \eta^2}{n}\!\! \sum_{j=1}^n \!\e[\|\nabla \ell(\phi_{t, j}; \bz_j)\!-\!\nabla \ell(\phi_{t, j}^{(i)}; \bz_j^{(i)})\|^2]\label{eq:sc-stab1}.
\end{align}
To proceed, recall we have from the update rule of $\{\phi_{t+1,j}\}$ that
\begin{align}
    &\e\Big[\sum_{j=1}^n\|\nabla \ell(\phi_{t+1, j}; \bz_j)-\nabla \ell(\phi_{t+1,j}^{(i)}; \bz_j^{(i)})\|^2\Big] \notag\\
    \!=& \frac{1}{n} \e\Big[\sum_{j=1}^n\|\nabla \ell(\bw_t; \bz_j)\!-\!\nabla \ell(\bw_t^{(i)}; \bz_j^{(i)})\|^2\Big]\!+\!\frac{n\!-\!1}{n} \e\Big[\sum_{j=1}^n\|\nabla \ell(\phi_{t, j}; \bz_j)\!-\!\nabla \ell(\phi_{t, j}^{(i)}; \bz_j^{(i)})\|^2\Big]\label{eq:sc-stab-phi}.
\end{align}
Multiplying both sides of Eq.~\eqref{eq:sc-stab-phi} by $4\eta^2$ and adding it to Eq.~\eqref{eq:sc-stab1}, we derive
\begin{align}
&\e\Big[\|\bw_{t+1}-\bw_{t+1}^{(i)}\|^2+4 \eta^2  \sum_{j=1}^n\|\nabla \ell(\phi_{t+1, j}; \bz_j)-\nabla \ell(\phi_{t+1,j}^{(i)}; \bz_j^{(i)})\|^2\Big]\notag \\
\leq & \big(1\!+\!\frac{p}{n}\!-\!\eta\mu\big) \e[\|\bw_t\!-\!\bw_t^{(i)}\|^2]\!-\!\frac{1}{n}\Big(\frac{(n\!-\!2)\eta}{(n\!-\!1)\alpha}\!-\!6\eta^2\Big) \sum_{j \neq i} \e[\|\nabla \ell(\bw_t; \bz_j)\!-\!\nabla \ell(\bw_t^{(i)}; \bz_j)\|^2]\notag \\
& \!+\!\frac{6\!+\!1 / p}{n} \eta^2 \e[\|\nabla \ell(\bw_t; \bz_i)\!-\!\nabla \ell(\bw_t^{(i)}; \bz_i^{\prime})\|^2]\!+\!\Big(4\!-\!\frac{2}{n}\Big)\eta^2\! \sum_{j=1}^n\! \e[\|\nabla \ell(\phi_{t,j}; \bz_j)\!-\!\nabla \ell(\phi_{t,j}^{(i)}; \bz_j^{(i)})\|^2]\label{eq:sc-stab2}.
\end{align}
Define the Lyapunov function $H_t$ via
\[H_t :=\|\bw_t-\bw_t^{(i)}\|^2+ 4\eta^2 \sum_{j=1}^n \| \nabla \ell(\phi_{t, j}; \bz_j)-\nabla \ell(\phi_{t,j}^{(i)}; \bz_j^{(i)})\|^2.\]
From the condition $\eta\leq 1/(2\mu n)$ and $p>0$, we know that
\[
4-\frac{2}{n}\leq 4(1-\eta\mu)\leq 4(1+p/n-\eta\mu),
\]
which together with Eq.~\eqref{eq:sc-stab2} leads to
\begin{multline}\label{eq:sc-phi-1}
    \e[H_{t+1}]\leq (1+p/n-\eta\mu) \e[H_t]\\-\!\frac{1}{n}\Big(\frac{(n\!-\!2)\eta}{(n\!-\!1)\alpha}\!-\!6\eta^2\Big) \sum_{j \neq i} \e[\|\nabla \ell(\bw_t; \bz_j)\!-\!\nabla \ell(\bw_t^{(i)}; \bz_j)\|^2]
    \!+\!\frac{6\!+\!\frac{1}{p}}{n} \eta^2 \e[\|\nabla \ell(\bw_t; \bz_i)\!-\!\nabla \ell(\bw_t^{(i)}; \bz_i^{\prime})\|^2].
\end{multline}
Since $\eta\leq \frac{n-2}{6\alpha(n-1)}$, we know the coefficient $\frac{(n-2)\eta}{(n-1)\alpha}-6\eta^2\ge 0$. Moreover, recall $\e[\|\nabla \ell(\bw_t; \bz_{i}) -\nabla \ell(\bw_t^{(i)}; \bz_i^{\prime})\|^2]\leq 8\alpha\ebb[L_S(\bw_{t})]$ as shown in Eq.~\eqref{eq:cvx-stab-i}, we thus obtain from Eq.~\eqref{eq:sc-phi-1} that
\begin{equation}\label{eq:sc-phi-2}
    \e[H_{t+1}]\leq (1+p/n-\eta\mu) \e[H_t]+\frac{8\alpha(6+1 / p)}{n} \eta^2 \e[L_S(\bw_t)].
\end{equation}
Applying Eq.~\eqref{eq:sc-phi-2} recursively gives
\begin{equation}\label{eq:sc-phi-3}
    \e[H_{t+1}] \leq \sum_{k=1}^t(1+p / n-\eta\mu)^{t-k} \frac{8 \alpha(6+1 / p)}{n} \eta^2 \e[L_S(\bw_k)]+(1+p / n-\eta\mu)^t \e[H_1].
\end{equation}
From $\bw_1=\bw_1^{(i)}$ and $\phi_{1, j}=\bw_1,~\phi_{1, j}^{(i)}=\bw_1^{(i)}$ for all $j\in[n]$, we know that
\[
H_1=4 \eta^2 \sum_{j=1}^n \| \nabla \ell(\bw_1; \bz_j)-\nabla \ell(\bw_1; \bz_j^{(i)})\|^2=4 \eta^2 \| \nabla \ell(\bw_1; \bz_i)-\nabla \ell(\bw_1; \bz'_i)\|^2,
\]
which together with Eq.~\eqref{eq:cvx-stab-i} implies that
\begin{equation}\label{eq:sc-phi-4}
  \e[H_1] \leq 32\alpha\eta^2\e[L_S(\bw_1)]=32\alpha\eta^2L(\bw_1).
\end{equation}
Plugging Eq.~\eqref{eq:sc-phi-4} into Eq.~\eqref{eq:sc-phi-3} and letting $p=n/t$, we obtain
\begin{align}
    &\e[\|\bw_{t+1}-\bw_{t+1}^{(i)}\|^2] \leq \e[H_{t+1}]\notag\\
    \leq&\sum_{k=1}^t(1+1 / t-\eta\mu)^{t-k} \frac{8 \alpha(6+t / n)}{n} \eta^2 \e[L_S(\bw_k)]+32\alpha\eta^2(1+1 / t-\eta\mu)^t L(\bw_1)\label{eq:sc-stab-f},
\end{align}
which completes the proof.
\end{proof}

\subsubsection{Proof of EPR Bounds\label{sec:proof-exc-sc}}
\begin{proof}[Proof of Theorem~\ref{thm:epr-sc}]
Note for any $t>1/(\eta\mu)$, we have $1+1 / t-\eta\mu<1$. Hence it follows from $\eta\leq\frac{1}{2\mu n+12\alpha}\leq \min\{\frac{1}{2\mu n}, \frac{n-2}{6\alpha(n-1)}\}$ and Eq.~\eqref{eq:sc-stab-f} that
\begin{equation}\label{eq:epr-sc1}
    \e[\|\bw_t-\bw_t^{(i)}\|^2] \leq \sum_{k=1}^t(1+1 / t-\eta\mu)^{t-k} \frac{8 \alpha(6+t / n)}{n} \eta^2 \e[L_S(\bw_k)]+32\alpha\eta^2 L(\bw_1).
\end{equation}
Applying Lemma \ref{lem:general-stablity}, together with Eq.~\eqref{eq:epr-sc1}, we obtain
\begin{multline}
\e[L(\bw_t)-L_S(\bw_t)] \leq \frac{\alpha}{\gamma} \e[L_S(\bw_t)]+\frac{\alpha+\gamma}{2 n} \sum_{i=1}^n \e[\|\bw_t-\bw_t^{(i)}\|^2]
\leq  \frac{\alpha}{\gamma} \e[L_S(\bw_t)]\\+\frac{4(\alpha+\gamma) \alpha(6+t / n)}{n} \eta^2 \sum_{k=1}^t (1+1 / t-\eta\mu)^{t-k} \e[L_S(\bw_k)]+16(\alpha+\gamma) \alpha \eta^2 L(\bw_1)\label{eq:epr-sc2}.
\end{multline}
Next, we bound the terms involving $\e[L_S(\bw_k)]$. From the analysis in the proof of \citep[Theorem 1]{defazio2014saga}, we know that for any $k\ge 1$, it holds that
\[
\e[\|\bw_k-\bw_S\|^2]\leq (1-\eta\mu)^{k-1}\big(\e_S[\|\bw_1-\bw_S\|^2]+2n\eta\e[L_S(\bw_1)-L_S(\bw_S)]\big)
\]
Plugging in the given choice of $\eta$ in the above inequality, and noting that $\e_S[\|\bw_1-\bw_S\|^2]\leq \infty$ and $n^{-1}\lesssim\mu$, we know that
\begin{align*}
    \e[\|\bw_k-\bw_S\|^2]\leq & (1-\eta\mu)^{k-1}\Big(\e_S[\|\bw_1-\bw_S\|^2]+\frac{2n}{2\mu n+12\alpha}L(\bw_1)\Big)\notag\\
    \lesssim &(1-\eta\mu)^k\Big(1+\frac{1}{\mu+1/n}\Big)\lesssim (1-\eta\mu)^k/\mu.
\end{align*}
Hence, from Eq.~\eqref{eq:epr-cvx2-0}, the $\alpha$-smoothness of $L_S$, and $\nabla L_S(\bw_S)=0$, we obtain
\begin{align}
  &\e[L_S(\bw_k)]-L(\bw^*)\leq \e[L_S(\bw_k)-L_S(\bw_S)]\notag\\
  \leq &\langle \nabla L_S(\bw_S), \bw_k-\bw_S\rangle+\frac{\alpha}{2}\e[\|\bw_k-\bw_S\|^2]=\frac{\alpha}{2}\e[\|\bw_k-\bw_S\|^2]\lesssim (1-\eta\mu)^k/\mu \label{eq:epr-sc-opt},
\end{align}
which further leads to
\begin{align}
    &\sum_{k=1}^t (1+1 / t-\eta\mu)^{t-k} \e[L_S(\bw_k)]\lesssim \sum_{k=1}^t (1+1 / t-\eta\mu)^{t-k} \big(L(\bw^*)+(1-\eta\mu)^k/\mu\big)\notag\\
    \lesssim & L(\bw^*)\sum_{k=1}^t(1+1 / t-\eta\mu)^{t-k}+t(1+1 / t-\eta\mu)^t/\mu \leq \frac{L(\bw^*)}{\eta\mu-1/t}+t(1+1 / t-\eta\mu)^t/\mu.\label{eq:epr-sc4}
\end{align}
Applying Eq.~\eqref{eq:epr-sc-opt} with $k=t$ and Eq.~\eqref{eq:epr-sc4} on Eq.~\eqref{eq:epr-sc2} gives
\begin{multline}
  \e[L(\bw_t)-L_S(\bw_t)]
  \lesssim \frac{1}{\gamma}\Big(L(\bw^*)+(1-\eta\mu)^t/\mu\Big)+  \\
  \frac{(1+\gamma)(1+t / n)}{n} \eta^2\Big(\frac{L(\bw^*)}{\eta\mu-1/t}+t(1+1 / t-\eta\mu)^t/\mu\Big)+ (1+\gamma) \eta^2 L(\bw_1). \label{eq:epr-sc-gen}
\end{multline}
Combining the generalization error in Eq.~\eqref{eq:epr-sc-gen} and optimization error in Eq.~\eqref{eq:epr-sc-opt} with $k=t$ leads to
\begin{multline}\label{eq:epr-sc}
    \e[L(\bw_t)]-L(\bw^*) \lesssim \frac{1}{\gamma}\Big(L(\bw^*)+(1-\eta\mu)^t/\mu\Big)\\
    +\frac{(1+\gamma)(1+t / n)}{n} \eta^2\Big(\frac{L(\bw^*)}{\eta\mu-1/t}+t(1+1 / t-\eta\mu)^t/\mu\Big)+ (1+\gamma) \eta^2 L(\bw_1) + (1-\eta\mu)^t/\mu,
\end{multline}
which completes the proof.
\end{proof}
\begin{proof}[Proof of Corollary~\ref{co:epr-sc}]
With the given choices of parameters, we first show that
\begin{equation}\label{eq:epr-sc-1}
    (1-\eta\mu)^t \lesssim \frac{1}{n}.
\end{equation}
To see this, plugging in $\eta=\frac{1}{2\mu n+12\alpha}$, $t \asymp n \log n$ and taking logarithm on both sides of Eq.~\eqref{eq:epr-sc-1}, then it suffices to show that
\[
\log\Big(1-\frac{1}{2n+ 12\alpha/\mu}\Big)\lesssim -\frac{1}{n},
\]
which clearly holds due to the basic inequality $\log(1+x)\leq x$ with $x=-\frac{1}{2n+ 12\alpha/\mu}$ and the condition $n^{-1}\lesssim n^{-1} (\log n)^{2/3}\lesssim\mu$.

Moreover, since $t\asymp n\log n$, we can assume $t\ge 2/(\eta\mu)\asymp n$ without loss of generality. Hence we have
\[
(1+1 / t-\eta\mu)^t\leq (1-\eta\mu/2)^t\lesssim \frac{1}{n}\lesssim \frac{\mu}{(\log n)^{2/3}},
\]
where the second inequality holds by using a similar approach to the proof of Eq.~\eqref{eq:epr-sc-1}. Noting that $1\lesssim \gamma$, we further obtain
\begin{align}
    &\frac{1}{\gamma}\Big(L(\bw^*)+(1-\eta\mu)^t/\mu\Big)\lesssim \frac{(\log n)^{2/3}}{\mu n}\Big(L(\bw^*)+\frac{1}{\mu n}\Big)\lesssim \frac{(\log n)^{2/3}}{\mu n},\label{eq:epr-sc-2}\\
    &\frac{(1+\gamma)(1+t / n)}{n} \eta^2\Big(\frac{L(\bw^*)}{\eta\mu-1/t}+t(1+1 / t-\eta\mu)^t/\mu\Big)\notag\\
    \lesssim & \frac{\frac{\mu n}{(\log n)^{2/3}} \log n}{n}\frac{1}{(\mu n)^2} \Big(\frac{1}{1/n-1/(n \log n)}+n \log n \frac{\mu}{(\log n)^{2/3}}\frac{1}{\mu}\Big)\lesssim \frac{(\log n)^{2/3}}{\mu n},\label{eq:epr-sc-3}
\end{align}
and
\begin{align}
    (1+\gamma) \eta^2 L(\bw_1)\lesssim &\Big(1+\frac{\mu n}{(\log n)^{2/3}}\Big)\frac{1}{(\mu n)^2}\lesssim \frac{1}{\mu n (\log n)^{2/3}}.\label{eq:epr-sc-4}
\end{align}
Plugging Eq.~\eqref{eq:epr-sc-1}-Eq.~\eqref{eq:epr-sc-4} into Eq.~\eqref{eq:epr-sc} and noting that $(1-\eta\mu)^t/\mu \lesssim \frac{1}{\mu n}$ from \eqref{eq:epr-sc-1} gives
\[
\e[L(\bw_t)]-L(\bw^*) \lesssim \frac{(\log n)^{2/3}}{\mu n},
\]
which completes the proof.
\end{proof}

\bibliographystyle{abbrvnat}
\bibliography{learning}

\begin{thebibliography}{50}
\providecommand{\natexlab}[1]{#1}
\providecommand{\url}[1]{\texttt{#1}}
\expandafter\ifx\csname urlstyle\endcsname\relax
  \providecommand{\doi}[1]{doi: #1}\else
  \providecommand{\doi}{doi: \begingroup \urlstyle{rm}\Url}\fi

\bibitem[Agarwal et~al.(2009)Agarwal, Wainwright, Bartlett, and Ravikumar]{agarwal2009information}
A.~Agarwal, M.~J. Wainwright, P.~L. Bartlett, and P.~K. Ravikumar.
\newblock Information-theoretic lower bounds on the oracle complexity of convex optimization.
\newblock In \emph{Advances in Neural Information Processing Systems}, pages 1--9, 2009.

\bibitem[Allen-Zhu(2018)]{allen2018katyusha}
Z.~Allen-Zhu.
\newblock Katyusha: The first direct acceleration of stochastic gradient methods.
\newblock \emph{Journal of Machine Learning Research}, 18\penalty0 (221):\penalty0 1--51, 2018.

\bibitem[Bassily et~al.(2020)Bassily, Feldman, Guzm{\'a}n, and Talwar]{bassily2020stability}
R.~Bassily, V.~Feldman, C.~Guzm{\'a}n, and K.~Talwar.
\newblock Stability of stochastic gradient descent on nonsmooth convex losses.
\newblock \emph{Advances in Neural Information Processing Systems}, 33, 2020.

\bibitem[Blatt et~al.(2007)Blatt, Hero, and Gauchman]{blatt2007convergent}
D.~Blatt, A.~O. Hero, and H.~Gauchman.
\newblock A convergent incremental gradient method with a constant step size.
\newblock \emph{SIAM Journal on Optimization}, 18\penalty0 (1):\penalty0 29--51, 2007.

\bibitem[Bottou et~al.(2018)Bottou, Curtis, and Nocedal]{bottou2018optimization}
L.~Bottou, F.~E. Curtis, and J.~Nocedal.
\newblock Optimization methods for large-scale machine learning.
\newblock \emph{SIAM Review}, 60\penalty0 (2):\penalty0 223--311, 2018.

\bibitem[Bousquet and Bottou(2008)]{bousquet2008tradeoffs}
O.~Bousquet and L.~Bottou.
\newblock The tradeoffs of large scale learning.
\newblock In \emph{Advances in Neural Information Processing Systems}, pages 161--168, 2008.

\bibitem[Bousquet and Elisseeff(2002)]{bousquet2002stability}
O.~Bousquet and A.~Elisseeff.
\newblock Stability and generalization.
\newblock \emph{Journal of Machine Learning Research}, 2\penalty0 (Mar):\penalty0 499--526, 2002.

\bibitem[Chang(2008)]{chang2008libsvm}
C.~Chang.
\newblock {LIBSVM data: Classification, regression, and multi-label}.
\newblock \emph{http://www. csie. ntu. edu. tw/\~{} cjlin/libsvmtools/datasets/}, 2008.

\bibitem[Charles and Papailiopoulos(2018)]{charles2018stability}
Z.~Charles and D.~Papailiopoulos.
\newblock Stability and generalization of learning algorithms that converge to global optima.
\newblock In \emph{International Conference on Machine Learning}, pages 744--753, 2018.

\bibitem[Chen et~al.(2024)Chen, Fernando, Ying, and Chen]{chen2024three}
L.~Chen, H.~Fernando, Y.~Ying, and T.~Chen.
\newblock Three-way trade-off in multi-objective learning: Optimization, generalization and conflict-avoidance.
\newblock \emph{Journal of Machine Learning Research}, 25\penalty0 (193):\penalty0 1--53, 2024.

\bibitem[Chen et~al.(2023)Chen, Zheng, Long, and Su]{chen2023minimax}
S.~Chen, Q.~Zheng, Q.~Long, and W.~J. Su.
\newblock Minimax estimation for personalized federated learning: An alternative between {FedAvg} and local training?
\newblock \emph{Journal of Machine Learning Research}, 24\penalty0 (262):\penalty0 1--59, 2023.

\bibitem[Cotter et~al.(2011)Cotter, Shamir, Srebro, and Sridharan]{cotter2011better}
A.~Cotter, O.~Shamir, N.~Srebro, and K.~Sridharan.
\newblock Better mini-batch algorithms via accelerated gradient methods.
\newblock \emph{Advances in Neural Information Processing Systems}, 24:\penalty0 1647--1655, 2011.

\bibitem[Csiba et~al.(2015)Csiba, Qu, and Richt{\'a}rik]{csiba2015stochastic}
D.~Csiba, Z.~Qu, and P.~Richt{\'a}rik.
\newblock Stochastic dual coordinate ascent with adaptive probabilities.
\newblock In \emph{International Conference on Machine Learning}, pages 674--683. PMLR, 2015.

\bibitem[Defazio et~al.(2014)Defazio, Bach, and Lacoste-Julien]{defazio2014saga}
A.~Defazio, F.~Bach, and S.~Lacoste-Julien.
\newblock {SAGA}: A fast incremental gradient method with support for non-strongly convex composite objectives.
\newblock In \emph{Advances in Neural Information Processing Systems}, pages 1646--1654, 2014.

\bibitem[Dekel et~al.(2012)Dekel, Gilad-Bachrach, Shamir, and Xiao]{dekel2012optimal}
O.~Dekel, R.~Gilad-Bachrach, O.~Shamir, and L.~Xiao.
\newblock Optimal distributed online prediction using mini-batches.
\newblock \emph{Journal of Machine Learning Research}, 13\penalty0 (1), 2012.

\bibitem[Deng et~al.(2025)Deng, Shen, Li, Sun, Li, and Tao]{deng2025towards}
X.~Deng, L.~Shen, S.~Li, T.~Sun, D.~Li, and D.~Tao.
\newblock Towards understanding the generalizability of delayed stochastic gradient descent.
\newblock \emph{IEEE Transactions on Pattern Analysis and Machine Intelligence}, 2025.

\bibitem[Deora et~al.(2024)Deora, Ghaderi, Taheri, and Thrampoulidis]{deora2024optimization}
P.~Deora, R.~Ghaderi, H.~Taheri, and C.~Thrampoulidis.
\newblock On the optimization and generalization of multi-head attention.
\newblock \emph{Transactions on Machine Learning Research}, 2024.

\bibitem[Fang et~al.(2018)Fang, Li, Lin, and Zhang]{fang2018spider}
C.~Fang, C.~J. Li, Z.~Lin, and T.~Zhang.
\newblock Spider: Near-optimal non-convex optimization via stochastic path-integrated differential estimator.
\newblock In \emph{Advances in Neural Information Processing Systems}, pages 689--699, 2018.

\bibitem[Fang et~al.(2019)Fang, Lin, and Zhang]{fang2019sharp}
C.~Fang, Z.~Lin, and T.~Zhang.
\newblock Sharp analysis for nonconvex sgd escaping from saddle points.
\newblock In \emph{Conference on Learning Theory}, pages 1192--1234. PMLR, 2019.

\bibitem[Gower et~al.(2020)Gower, Schmidt, Bach, and Richt{\'a}rik]{gower2020variance}
R.~M. Gower, M.~Schmidt, F.~Bach, and P.~Richt{\'a}rik.
\newblock Variance-reduced methods for machine learning.
\newblock \emph{Proceedings of the IEEE}, 108\penalty0 (11):\penalty0 1968--1983, 2020.

\bibitem[Hardt et~al.(2016)Hardt, Recht, and Singer]{hardt2016train}
M.~Hardt, B.~Recht, and Y.~Singer.
\newblock Train faster, generalize better: Stability of stochastic gradient descent.
\newblock In \emph{International Conference on Machine Learning}, pages 1225--1234, 2016.

\bibitem[Johnson and Zhang(2013)]{johnson2013accelerating}
R.~Johnson and T.~Zhang.
\newblock Accelerating stochastic gradient descent using predictive variance reduction.
\newblock In \emph{Advances in Neural Information Processing Systems}, pages 315--323, 2013.

\bibitem[Kuzborskij and Lampert(2018)]{kuzborskij2018data}
I.~Kuzborskij and C.~Lampert.
\newblock Data-dependent stability of stochastic gradient descent.
\newblock In \emph{International Conference on Machine Learning}, pages 2820--2829, 2018.

\bibitem[Lei and Ying(2020)]{lei2020fine}
Y.~Lei and Y.~Ying.
\newblock Fine-grained analysis of stability and generalization for stochastic gradient descent.
\newblock In \emph{International Conference on Machine Learning}, pages 5809--5819, 2020.

\bibitem[Liu et~al.(2017)Liu, Lugosi, Neu, and Tao]{liu2017algorithmic}
T.~Liu, G.~Lugosi, G.~Neu, and D.~Tao.
\newblock Algorithmic stability and hypothesis complexity.
\newblock In \emph{International Conference on Machine Learning}, pages 2159--2167, 2017.

\bibitem[Liu et~al.(2024)Liu, Zhang, Gu, and Chen]{liu2024general}
X.~Liu, H.~Zhang, B.~Gu, and H.~Chen.
\newblock General stability analysis for zeroth-order optimization algorithms.
\newblock In \emph{The Twelfth International Conference on Learning Representations}, 2024.

\bibitem[Loizou and Richt{\'a}rik(2020)]{loizou2020momentum}
N.~Loizou and P.~Richt{\'a}rik.
\newblock Momentum and stochastic momentum for stochastic gradient, {Newton}, proximal point and subspace descent methods.
\newblock \emph{Computational Optimization and Applications}, 77\penalty0 (3):\penalty0 653--710, 2020.

\bibitem[Mairal(2015)]{mairal2015incremental}
J.~Mairal.
\newblock Incremental majorization-minimization optimization with application to large-scale machine learning.
\newblock \emph{SIAM Journal on Optimization}, 25\penalty0 (2):\penalty0 829--855, 2015.

\bibitem[Meng et~al.(2017)Meng, Wang, Chen, Wang, Ma, and Liu]{meng2017generalization}
Q.~Meng, Y.~Wang, W.~Chen, T.~Wang, Z.~Ma, and T.~Liu.
\newblock Generalization error bounds for optimization algorithms via stability.
\newblock In \emph{Proceedings of the AAAI Conference on Artificial Intelligence}, volume~31, 2017.

\bibitem[Mou et~al.(2018)Mou, Wang, Zhai, and Zheng]{mou2018generalization}
W.~Mou, L.~Wang, X.~Zhai, and K.~Zheng.
\newblock Generalization bounds of {SGLD} for non-convex learning: Two theoretical viewpoints.
\newblock In \emph{Conference on Learning Theory}, pages 605--638, 2018.

\bibitem[Nguyen et~al.(2017)Nguyen, Liu, Scheinberg, and Tak{\'a}{\v{c}}]{nguyen2017sarah}
L.~M. Nguyen, J.~Liu, K.~Scheinberg, and M.~Tak{\'a}{\v{c}}.
\newblock {SARAH}: A novel method for machine learning problems using stochastic recursive gradient.
\newblock In \emph{International Conference on Machine Learning}, pages 2613--2621. JMLR. org, 2017.

\bibitem[Nguyen et~al.(2021)Nguyen, Tran-Dinh, Phan, Nguyen, and Van~Dijk]{nguyen2021unified}
L.~M. Nguyen, Q.~Tran-Dinh, D.~T. Phan, P.~H. Nguyen, and M.~Van~Dijk.
\newblock A unified convergence analysis for shuffling-type gradient methods.
\newblock \emph{Journal of Machine Learning Research}, 22\penalty0 (207):\penalty0 1--44, 2021.

\bibitem[Nikolakakis et~al.(2022)Nikolakakis, Haddadpour, Kalogerias, and Karbasi]{nikolakakis2022black}
K.~Nikolakakis, F.~Haddadpour, D.~Kalogerias, and A.~Karbasi.
\newblock Black-box generalization: Stability of zeroth-order learning.
\newblock \emph{Advances in Neural Information Processing Systems}, 35:\penalty0 31525--31541, 2022.

\bibitem[Nikolakakis et~al.(2025)Nikolakakis, Karbasi, and Kalogerias]{nikolakakis2025select}
K.~E. Nikolakakis, A.~Karbasi, and D.~Kalogerias.
\newblock Select without fear: Almost all minibatch schedules generalize optimally.
\newblock \emph{SIAM Journal on Mathematics of Data Science}, 7\penalty0 (3):\penalty0 965--992, 2025.

\bibitem[Reddi et~al.(2016)Reddi, Hefny, Sra, Poczos, and Smola]{reddi2016stochastic}
S.~Reddi, A.~Hefny, S.~Sra, B.~Poczos, and A.~Smola.
\newblock Stochastic variance reduction for nonconvex optimization.
\newblock In \emph{International Conference on Machine Learning}, pages 314--323, 2016.

\bibitem[Richards and Kuzborskij(2021)]{richards2021stability}
D.~Richards and I.~Kuzborskij.
\newblock Stability \& generalisation of gradient descent for shallow neural networks without the neural tangent kernel.
\newblock \emph{Advances in Neural Information Processing Systems}, 34, 2021.

\bibitem[Robbins and Monro(1951)]{robbins1951stochastic}
H.~Robbins and S.~Monro.
\newblock A stochastic approximation method.
\newblock \emph{Annals of Mathematical Statistics}, pages 400--407, 1951.

\bibitem[Rosasco et~al.(2020)Rosasco, Villa, and V{\~u}]{rosasco2020convergence}
L.~Rosasco, S.~Villa, and B.~C. V{\~u}.
\newblock Convergence of stochastic proximal gradient algorithm.
\newblock \emph{Applied Mathematics \& Optimization}, 82\penalty0 (3):\penalty0 891--917, 2020.

\bibitem[Roux et~al.(2012)Roux, Schmidt, and Bach]{roux2012stochastic}
N.~Roux, M.~Schmidt, and F.~Bach.
\newblock A stochastic gradient method with an exponential convergence rate for finite training sets.
\newblock \emph{Advances in Neural Information Processing Systems}, 25, 2012.

\bibitem[Schliserman and Koren(2022)]{schliserman2022stability}
M.~Schliserman and T.~Koren.
\newblock Stability vs implicit bias of gradient methods on separable data and beyond.
\newblock In \emph{Conference on Learning Theory}, pages 3380--3394, 2022.

\bibitem[Shalev-Shwartz and Zhang(2013)]{shalev2013stochastic}
S.~Shalev-Shwartz and T.~Zhang.
\newblock Stochastic dual coordinate ascent methods for regularized loss minimization.
\newblock \emph{Journal of Machine Learning Research}, 14\penalty0 (2), 2013.

\bibitem[Shalev-Shwartz et~al.(2010)Shalev-Shwartz, Shamir, Srebro, and Sridharan]{shalev2010learnability}
S.~Shalev-Shwartz, O.~Shamir, N.~Srebro, and K.~Sridharan.
\newblock Learnability, stability and uniform convergence.
\newblock \emph{Journal of Machine Learning Research}, 11\penalty0 (Oct):\penalty0 2635--2670, 2010.

\bibitem[Srebro et~al.(2010)Srebro, Sridharan, and Tewari]{srebro2010smoothness}
N.~Srebro, K.~Sridharan, and A.~Tewari.
\newblock Smoothness, low noise and fast rates.
\newblock In \emph{Advances in Neural Information Processing Systems}, pages 2199--2207, 2010.

\bibitem[Sun et~al.(2023)Sun, Shen, and Tao]{sun2023understanding}
Y.~Sun, L.~Shen, and D.~Tao.
\newblock Understanding how consistency works in federated learning via stage-wise relaxed initialization.
\newblock \emph{Advances in Neural Information Processing Systems}, 36:\penalty0 80543--80574, 2023.

\bibitem[Sun et~al.(2024)Sun, Niu, and Wei]{sun2024understanding}
Z.~Sun, X.~Niu, and E.~Wei.
\newblock Understanding generalization of federated learning via stability: Heterogeneity matters.
\newblock In \emph{International Conference on Artificial Intelligence and Statistics}, pages 676--684. PMLR, 2024.

\bibitem[Taheri and Thrampoulidis(2024)]{taheri2024generalization}
H.~Taheri and C.~Thrampoulidis.
\newblock Generalization and stability of interpolating neural networks with minimal width.
\newblock \emph{Journal of Machine Learning Research}, 25\penalty0 (156):\penalty0 1--41, 2024.

\bibitem[Wang et~al.(2013)Wang, Chen, Smola, and Xing]{wang2013variance}
C.~Wang, X.~Chen, A.~J. Smola, and E.~P. Xing.
\newblock Variance reduction for stochastic gradient optimization.
\newblock \emph{Advances in Neural Information Processing Systems}, 26, 2013.

\bibitem[Xiao and Zhang(2014)]{xiao2014proximal}
L.~Xiao and T.~Zhang.
\newblock A proximal stochastic gradient method with progressive variance reduction.
\newblock \emph{SIAM Journal on Optimization}, 24\penalty0 (4):\penalty0 2057--2075, 2014.

\bibitem[Yuan and Li(2023)]{yuan2023sharper}
X.~Yuan and P.~Li.
\newblock Sharper analysis for minibatch stochastic proximal point methods: Stability, smoothness, and deviation.
\newblock \emph{Journal of Machine Learning Research}, 24\penalty0 (270):\penalty0 1--52, 2023.

\bibitem[Zeng and Lei(2026)]{zeng2026stochastic}
S.~Zeng and Y.~Lei.
\newblock Stochastic gradient methods: Bias, stability and generalization.
\newblock \emph{Journal of Machine Learning Research}, 27\penalty0 (6):\penalty0 1--55, 2026.

\end{thebibliography}
\end{document}